\documentclass[final,leqno,onefignum,onetabnum]{siamltexmm}

\usepackage{amsmath,graphicx,amssymb}
\usepackage{enumerate}
\usepackage{url}

\usepackage{caption}
\usepackage{algorithm}
\usepackage{algpseudocode}

\usepackage{tabularx}
\usepackage{multirow}

\title{Efficient Blind Compressed Sensing Using Sparsifying Transforms \\ with Convergence Guarantees and Application to MRI \thanks{This work was supported in part by the National Science Foundation (NSF) under grants CCF-1018660 and CCF-1320953.
A shorter version of this work appears elsewhere \cite{samptabcs}.}} 

\author{Saiprasad~Ravishankar and~Yoram~Bresler\thanks{S. Ravishankar and Y. Bresler are with the Department of Electrical and Computer Engineering and the  Coordinated Science Laboratory, University of Illinois, Urbana-Champaign, IL, 61801 USA
(\email{(ravisha3, ybresler)@illinois.edu}). Questions, comments, or corrections
to this document may be directed to that email address.}}

\begin{document} \sloppy
\maketitle

\newcommand{\slugmaster}{%
\slugger{siims}{2014}{xx}{x}{x--x}}

\begin{abstract}
Natural signals and images are well-known to be approximately sparse in transform domains such as Wavelets and DCT. This property has been heavily exploited in various applications in image processing and medical imaging. Compressed sensing exploits the sparsity of images or image patches in a transform domain or synthesis dictionary to reconstruct images from undersampled measurements. In this work, we focus on blind compressed sensing, where the underlying sparsifying transform is a priori unknown, and propose a framework to simultaneously reconstruct the underlying image as well as the sparsifying transform from highly undersampled measurements. The proposed block coordinate descent type algorithms involve highly efficient optimal updates.
Importantly, we prove that although the proposed blind compressed sensing formulations are highly nonconvex, our algorithms are \emph{globally convergent} (i.e., they converge from any initialization) to the set of critical points of the objectives defining the formulations. These critical points are guaranteed to be at least partial global and partial local minimizers. The exact point(s) of convergence may depend on initialization.
We illustrate the usefulness of the proposed framework for magnetic resonance image reconstruction from highly undersampled k-space measurements. As compared to previous methods involving the synthesis dictionary model, our approach is much faster, while also providing promising reconstruction quality. 
\end{abstract}



\pagestyle{myheadings}
\thispagestyle{plain}
\markboth{S.~Ravishankar and~Y.~Bresler}{Transform-blind Compressed Sensing with Convergence Guarantees}

\section{Introduction}  \label{sec1}

Sparsity-based techniques have become extremely popular in various applications in image processing and imaging in recent years. These techniques typically exploit the sparsity of images or image patches in a transform domain or dictionary to compress \cite{jpg2}, denoise, or restore/reconstruct images. In this work, we investigate the subject of blind compressed sensing, which aims to reconstruct images in the scenario when a good sparse model for the image is unknown a priori. In particular, we will work with the classical sparsifying transform model \cite{sabres} that has been shown to be useful in various imaging applications. In the following, we briefly review the topics of sparse modeling, compressed sensing, and blind compressed sensing. We then list the contributions of this work.

\subsection{Sparse Models}

Sparsity-based techniques in imaging and image processing typically rely on the existence of a good sparse model for the data. Various models are known to provide sparse representations of data such as the \emph{synthesis dictionary model} \cite{elmiru, ambruck}, and the \emph{sparsifying transform model} \cite{sabres, tfcode}. 

The synthesis dictionary model suggests that a signal $ y  \in \mathbb{C}^{n} $ can be sparsely represented as a linear combination of a small number of atoms or columns from a synthesis dictionary $ D  \in \mathbb{C}^{n \times K}$, i.e., $y = Dx$ with $x \in \mathbb{C}^{K} $ being sparse, i.e., $\left \| x \right \|_{0}\ll K$. The $l_{0}$ quasi norm measures the sparsity of $x$ by counting the number of non-zeros in $x$. Real-world signals usually satisfy $y = Dx + e$, where $e$ is the approximation error in the signal domain \cite{ambruck}. The well-known \emph{synthesis sparse coding} problem finds $x$ that minimizes $\left \| y-Dx \right \|_{2}^{2}$ subject to  $\left \| x \right \|_{0}\leq s$, where $s$ is a given sparsity level. 
A disadvantage of the synthesis model is that sparse coding is typically NP-hard (Non-deterministic Polynomial-time hard) \cite{npb, npa}, and the various approximate sparse coding algorithms \cite{pati, mp2, wei, chen2, RaoFocus, befro} tend to be computationally expensive for large scale problems. 

The alternative sparsifying transform model suggests that a signal $y $ is approximately sparsifiable using a transform $ W  \in \mathbb{C}^{m \times n} $. The assumption here is that $Wy = x + \eta$, where $x \in \mathbb{C}^{m}$ is sparse in some sense, and $\eta$ is a small residual error in the \emph{transform domain} rather than in the signal domain.
Natural signals and images are known to be approximately sparsifiable by analytical transforms such as the discrete cosine transform (DCT), Wavelets \cite{wav}, finite differences, and Contourlets \cite{do}.
The advantage of the transform model over the synthesis dictionary model \footnote{The sparsifying transform model enjoys similar advantages over other models such as the noisy signal analysis dictionary model \cite{sabres}, which is not discussed here for reasons of space.} is that sparse coding can be performed exactly and cheaply by zeroing out all but a certain number of nonzero transform coefficients of largest magnitude \cite{sabres}. 

\subsection{Compressed Sensing}

In the context of imaging, the recent theory of Compressed Sensing (CS) \cite{tao1, don, cand} (see also \cite{feng96a, BreFen-C96c, Fen-PT97, VenBre-C98b, BreGasVen-C99, GasBre-C00a, YeBreMou-J02, Bre-C2008a} for the earliest versions of CS for Fourier-sparse signals and for Fourier imaging) enables accurate recovery of images from significantly fewer measurements than the number of unknowns. In order to do so, it requires that the underlying image be sufficiently sparse in some transform domain or dictionary, and that the measurement acquisition procedure be incoherent, in an appropriate sense, with the transform. However, the image reconstruction procedure for compressed sensing is typically computationally expensive and non-linear. 

The image reconstruction problem in compressed sensing is typically formulated as follows
\begin{equation}\label{eq2}
\min_{x}\: \left \| \Psi x \right \|_{0}  \: \:\;  s.t.\:\;   Ax=y
\end{equation}
Here, $x \in \mathbb{C}^{p}$ is a vectorized representation of the image (obtained by stacking the image columns on top of each other) to be reconstructed, and $y \in \mathbb{C}^{m}$ denotes the measurements. The operator $A \in \mathbb{C}^{m \times p}$, with $m \ll p$ is known as the sensing matrix, or measurement matrix. The matrix $\Psi \in \mathbb{C}^{t \times p}$ is a sparsifying transform (typically chosen as orthonormal). The aim of Problem \eqref{eq2} is to find the image satisfying the measurement equation $Ax=y$, that is the sparsest possible in the $\Psi$-transform domain. Since, in CS, the measurement equation $Ax=y$ represents an underdetermined system of equations, an additional model (such as the sparsity model above) is needed to estimate the true underlying image.

When $\Psi$ is orthonormal, Problem \eqref{eq2} can be rewritten as 
\begin{equation}\label{eq3}
\min_{z}\: \left \| z \right \|_{0}  \: \:\;  s.t.\:\;   A \Psi^{H}z=y
\end{equation}
where we used the substitution $\Psi x = z$, and $(\cdot)^{H}$ denotes the matrix Hermitian (conjugate transpose) operation. Similar to the synthesis sparse coding problem, Problem \eqref{eq3} too is NP-hard. Often the $\textit{l}_{0}$ quasi norm in \eqref{eq2} is replaced with its convex relaxation, the $\textit{l}_{1}$ norm \cite{Dono}, and the following convex problem is solved to reconstruct the image, when the CS measurements are noisy \cite{lustig, Dono2}. 
\begin{equation}\label{eq1}
\min_{x}\:\left \| Ax-y \right \|_{2}^{2}+\lambda \left \| \Psi x \right \|_{1}
\end{equation}
In Problem \eqref{eq1}, the $\ell_2$ penalty for the measurement fidelity term can also be replaced with alternative penalties such as a weighted $\ell_2$ penalty, depending on the physics of the measurement process and the statistics of the measurement noise.


Recently, CS theory has been applied to imaging techniques such as magnetic resonance imaging (MRI) \cite{lustig, lustig2, Char, josh, Yoo, kal, Qu11}, computed tomography (CT) \cite{chen, choi11, CT11},
and Positron emission tomography (PET) imaging \cite{vali1, malc2}, demonstrating high quality reconstructions from a reduced set of measurements. Such compressive measurements are highly advantageous in these applications. For example, they help reduce the radiation dosage in CT, and reduce scan times and improve clinical throughput in MRI. Well-known inverse problems in image processing such as inpainting (where an image is reconstructed from a subset of measured pixels) can also be viewed as compressed sensing problems.

\subsection{Blind Compressed Sensing}

While conventional compressed sensing techniques utilize fixed analytical sparsifying transforms such as wavelets \cite{wav}, finite differences, and contourlets \cite{do}, to reconstruct images, in this work, we instead focus on the idea of blind compressed sensing (BCS) \cite{bresai, symul, Glei12, lingal1, wangying}, where the underlying sparse model is assumed unknown a priori. The goal of blind compressed sensing is to simultaneously reconstruct the underlying image(s) as well as the dictionary or transform from highly undersampled measurements. Thus, BCS enables the sparse model to be adaptive to the specific data under consideration. Recent research has shown that such data-driven adaptation of dictionaries or transforms is advantageous in many applications \cite{elad2, elad3, elad5, elad6, gyu, bresai, doubsp2l, yblsgg, sbclsTS2, saiwen}. While the adaptation of synthesis dictionaries \cite{ols, eng, elad, Yagh, skret, Mai} has been extensively studied, recent work has shown advantages in terms of computation and application-specific performance, for the adaptation of transform models \cite{sabres, sbclsTS2, saiwen}. 


In a prior work on BCS \cite{bresai}, we successfully demonstrated the usefulness of dictionary-based blind compressed sensing for MRI, even in the case when the undersampled measurements corresponding to only a single image are provided. In the latter case, the overlapping patches of the underlying image are assumed to be sparse in a dictionary, and the (unknown) patch-based dictionary, that is typically much smaller in size than the image, is learnt directly from the compressive measurements. 

BCS techniques have been demonstrated to provide much better image reconstruction quality compared to compressed sensing methods that utilize a fixed sparsifying transform or dictionary \cite{bresai, lingal1, wangying}. This is not surprising since BCS methods allow for data-specific adaptation, and data-specific dictionaries typically sparsify the underlying images much better than analytical ones.

\newpage

\subsection{Contributions}

\subsubsection{Highlights}

The BCS framework  assumes a particular class of sparse models for the underlying image(s) or image patches. 
While prior work on BCS primarily focused on the synthesis dictionary model, in this work, we instead focus on the sparsifying transform model. We propose novel problem formulations for BCS involving well-conditioned or orthonormal adaptive \emph{square} sparsifying transforms. Our framework simultaneously adapts the sparsifying transform and reconstructs the underlying image(s) from highly undersampled measurements. We propose efficient block coordinate descent-type algorithms for transform-based BCS. Importantly, we establish that our iterative algorithms are \emph{globally convergent} (i.e., they converge from any initialization) to the set of critical points of the proposed highly non-convex BCS cost functions. These critical points are guaranteed to be \emph{at least} partial global and partial local minimizers. The exact point(s) of convergence may depend on initialization.
Such convergence guarantees have not been established for prior \emph{blind} compressed sensing methods.

Note that although we focus on compressed sensing in the discussions and experiments of this work, the formulations and algorithms proposed by us can also handle the case when the measurement or sensing matrix $A$ is square (e.g., in signal denoising), or even tall (e.g., deconvolution). 

\subsubsection{Magnetic Resonance Imaging Application}

MRI is a non-invasive and non-ionizing imaging technique that offers a variety of contrast mechanisms, and enables excellent visualization of anatomical structures and physiological functions.  
However, the data in MRI, which are samples in k-space of the spatial Fourier transform of the object, are acquired sequentially in time. 
Hence, a major drawback of MRI, that affects both clinical throughput and image quality especially in dynamic imaging applications, is that it is a relatively slow imaging technique. 
Although there have been advances in scanner hardware \cite{pMRI-Survey} and pulse sequences, the rate at which MR data are acquired is limited by MR physics and  physiological constraints on RF energy deposition. Compressed sensing MRI (either blind, or with known sparse model) has become quite popular in recent years, and it alleviates some of the aforementioned problems by enabling accurate reconstruction of MR images from highly undersampled measurements. 

In this work, we illustrate the usefulness of the proposed transform-based BCS schemes for magnetic resonance image reconstruction from highly undersampled k-space data.
We show that our adaptive transform-based BCS provides better image reconstruction quality compared to prior methods that involve fixed image-based, or patch-based sparsifying transforms. Importantly, transform-based BCS is shown to be more than 10x faster than synthesis dictionary-based BCS \cite{bresai} for reconstructing 2D MRI data. The speedup is expected to be much higher when considering 3D or 4D MR data. These advantages make the proposed scheme more amenable for adoption for clinical use in MRI.

\subsection{Organization}

The rest of this paper is organized as follows. Section \ref{sec2} describes our transform learning-based blind compressed sensing formulations and their properties. 
In Section \ref{sec3}, we derive efficient block coordinate descent algorithms for solving the BCS Problems, and discuss the algorithms' computational costs. In Section \ref{sec4}, we present novel convergence guarantees for our algorithms. The proof of convergence is provided in the Appendix. Section \ref{sec5} presents experimental results demonstrating the convergence behavior, performance, and computational efficiency of the proposed scheme for the MRI application. In Section \ref{sec6}, we conclude with proposals for future work.

\section{Problem Formulations}  \label{sec2}


\subsection{Synthesis dictionary-based Blind Compressed Sensing}

The compressed sensing-based image reconstruction  Problem \eqref{eq1} can be viewed as a particular instance of the following constrained regularized inverse problem, with $\zeta (x) = \lambda \left \| \Psi x \right \|_{1}$, $\nu=1$, and $\mathcal{S} = \mathbb{C}^{p}$
\begin{equation}\label{reginveq1}
\min_{x \in \mathcal{S}}\: \nu \left \| Ax-y \right \|_{2}^{2}+ \zeta (x)
\end{equation}
However, CS image reconstructions employing fixed, non-adaptive sparsifying transforms typically suffer from many artifacts at high undersampling factors \cite{bresai}. Blind compressed sensing allows for the sparse model to be directly adapted to the object(s) being imaged. For example, the overlapping patches of the underlying image may be assumed to be sparse in  a certain adaptive model. In prior work \cite{bresai}, the following patch-based dictionary learning regularizer was used within Problem \eqref{reginveq1} along with $\mathcal{S} = \mathbb{C}^{p}$ 
\begin{align}
\zeta(x) = & \min_{D, B} \:  \sum_{j=1}^{N} \left \| P_{j}x-D b_{j} \right \|_{2}^{2}   \;\, s.t. \;\, \left \| d_{k} \right \|_{2}=1\:\forall \;k, \;\;  \left \| b_{j} \right \|_{0}\leq s \; \: \forall \,\,  j.
\end{align}
The resulting synthesis dictionary based BCS formulation is as follows
\begin{align}
\nonumber (\mathrm{P0})\:\:\: & \min_{x,D, B} \:   \nu \left \| Ax-y \right \|_{2}^{2} + \sum_{j=1}^{N} \left \| P_{j}x-D b_{j} \right \|_{2}^{2}   \;\, s.t. \;\; \left \| d_{k} \right \|_{2}=1\:\forall \;k, \;\;  \left \| b_{j} \right \|_{0}\leq s \; \: \forall \,\,  j.
\end{align}
Here, $\nu >0$ is a weight for the measurement fidelity term ($\left \| Ax-y \right \|_{2}^{2}$), and $ P_{j} \in \mathbb{C}^{n \times p} $ represents the operator that extracts a $ \sqrt{n} \times \sqrt{n} $ 2D patch as a vector $P_{j}x  \in \mathbb{C}^{n}$ from the image $x$. A total of $N$ overlapping 2D patches are used. The synthesis model allows each patch $P_{j}x$ to be approximated by a linear combination $D b_{j}$ of a small number of columns from a dictionary $ D  \in \mathbb{C}^{n \times K} $, where $b_{j} \in \mathbb{C}^{K} $ is sparse. The columns of the learnt dictionary (represented by $ d_{k}, 1 \leq k \leq K $) in (P0) are additionally constrained to be of unit norm in order to avoid the scaling ambiguity \cite{kar}.
The dictionary, and the image patch, are assumed to be much smaller than the image ($n, K \ll p$) in (P0). Problem (P0) thus enforces all the $N$ (a typically large number) overlapping image patches to be sparse in some dictionary $D$, which can be considered as a strong yet flexible prior on the underlying image.

We use $B \in \mathbb{C}^{n \times N}$ to denote the matrix that has the sparse codes of the patches $b_{j}$ as its columns. Each sparse code is permitted a maximum sparsity level of $s \ll n$ in (P0). Although a single sparsity level is used for all patches in (P0) for simplicity, in practice, different sparsity levels may be allowed for different patches (for example, by setting an appropriate error threshold in the sparse coding step of optimization algorithms \cite{bresai}).
For the case of MRI, the sensing matrix $A$ in (P0) is $ F_{u}  \in \mathbb{C}^{m \times p} $, the undersampled Fourier encoding matrix \cite{bresai}. 
The weight $ \nu $ in (P0) is set depending on the measurement noise standard deviation $ \sigma $ as $ \nu=\frac{\theta }{\sigma } $, where $ \theta $ is a positive constant \cite{bresai}. In practice, if the noise level is unknown, it may be estimated.



Problem (P0) is to learn a patch-based synthesis sparsifying dictionary ($n, K$ $\ll p$), and reconstruct the image simultaneously from highly undersampled measurements. As discussed before, we have previously shown significantly superior image reconstructions for MRI using (P0), as compared to non-adaptive compressed sensing schemes that solve Problem \eqref{eq1}. However, the BCS Problem (P0) is both non-convex and NP-hard. Approximate iterative algorithms for (P0) (e.g., the DLMRI algorithm \cite{bresai}) typically solve the synthesis sparse coding problem repeatedly, which makes them computationally expensive. Moreover, no convergence guarantees exist for the algorithms that solve (P0).


\subsection{Sparsifying Transform-based Blind Compressed Sensing} \label{dfgy6}

\subsubsection{Problem Formulations with Sparsity Constraints}
In order to overcome some of the aforementioned drawbacks of synthesis dictionary-based BCS, we propose using the sparsifying transform model in this work. Sparsifying transform learning has been shown to be effective and efficient in applications, while also enjoying good convergence guarantees \cite{sbclsTS2}.
Therefore, we use the following transform learning regularizer \cite{sabres}
\begin{align} \label{tLreg1}
\zeta(x) = & \min_{W, B}\:  \sum_{j=1}^{N} \left \| W P_{j}x- b_{j} \right \|_{2}^{2} + \lambda \, Q(W) \;\; s.t.\;\;  \left \| B \right \|_{0}\leq s
\end{align}
along with the constraint set $\mathcal{S} = \left \{ x \in \mathbb{C}^{p} : \left \| x \right \|_{2}\leq C \right \}$ within Problem \eqref{reginveq1} to arrive at the following adaptive sparsifying transform-based BCS problem formulation
\begin{align}
\nonumber (\mathrm{P1})\:\:\: & \min_{x,W, B}\:  \nu \left \| Ax-y \right \|_{2}^{2}  +  \sum_{j=1}^{N} \left \| W P_{j}x- b_{j} \right \|_{2}^{2} + \lambda \, Q(W) \;\, \, s.t.\;\;  \left \| B \right \|_{0}\leq s, \;\; \left \| x \right \|_{2}\leq C.
\end{align}
Here, $W \in \mathbb{C}^{n \times n}$ denotes a square sparsifying transform for the patches of the underlying image.
The penalty $\left \| W P_{j}x- b_{j} \right \|_{2}^{2}$ denotes the sparsification error (transform domain residual) \cite{sabres} for the $\mathrm{j^{th}}$ patch, with $b_{j}$ denoting the transform sparse code.
The sparsity term $\left \| B \right \|_{0}=\sum_{j=1}^{N}\left \| b_{j} \right \|_{0}$ counts the number of non-zeros in the sparse code matrix $B$.
Notice that the sparsity constraint in (P1) is enforced on all the overlapping patches, taken together. This is a way of enabling variable sparsity levels for each specific patch. 
The constraint $ \left \| x \right \|_{2}\leq C$ with $C>0$ in (P1), is to enforce any prior knowledge on the signal energy (or, range). For example, if the pixels of the underlying image take intensity values in the range $0-255$, then $C = 255 \sqrt{p}$ is an appropriate bound. 
The function $Q(W) : \mathbb{C}^{n \times n} \mapsto \mathbb{R}$ in Problem (P1) denotes a regularizer for the transform, and the weight $\lambda>0$. Notice that without an additional regularizer, $W=0$ is a trivial sparsifier for any patch, and therefore, $W=0$, $b_{j} =0 $ $\forall j$, $x= A^{\dagger} y$ (assuming this $x$ satisfies $\left \| x \right \|_{2}\leq C$) with $(\cdot)^{\dagger}$ denoting the pseudo-inverse, would trivially minimize the objective (without the regularizer $Q(W)$) in Problem (P1). 

Similar to prior work on transform learning \cite{sabres, sabres3}, we set $Q(W) \triangleq - \log \,\left | \mathrm{det \,} W \right |  + 0.5 \left \| W \right \|_{F}^{2} $ as the regularizer in the objective to prevent trivial solutions. The $- \log \,\left | \mathrm{det \,} W \right | $ penalty eliminates degenerate solutions such as those with repeated rows.  The $\left \| W \right \|_{F}^{2}$ penalty helps remove a `scale ambiguity' in the solution \cite{sabres}, which occurs when the optimal solution satisfies an exactly sparse representation, i.e., the optimal $(x, W, B)$ in (P1) is such that $W P_{j}x = b_{j}$ $\forall$ $j$, and  $\left \| B \right \|_{0}\leq s$. In this case, if the $\left \| W \right \|_{F}^{2}$ penalty is absent in (P1), the optimal $(W, B)$ can be scaled by $\beta \in \mathbb{C}$, with $\left | \beta \right | \to \infty$, which causes the objective to decrease unbounded. 


The $- \, \log \,\left | \mathrm{det \,} W \right | $ and $0.5 \left \| W \right \|_{F}^{2}$ penalties together also additionally help control the condition number $\kappa(W)$ and scaling of the learnt transform. 
If we were to minimize only the $Q(W)$ regularizer in Problem (P1) with respect to $W$, then the minimum is achieved with a $W$ that is unit conditioned, and with spectral norm (scaling) of  $1$ \cite{sabres}, i.e., a \emph{unitary} or \emph{orthonormal} transform $W$. Thus, similar to Corollary 2 in \cite{sabres}, it is easy to show that as $\lambda  \to \infty $ in Problem (P1), the optimal sparsifying transform(s) tends to a unitary one. 
In practice, transforms learnt via (P1) are typically close to unitary even for finite $\lambda$.
Adaptive well-conditioned transforms (small $\kappa(W) > 1$) have been previously shown to perform better than adaptive (strictly) orthonormal ones in some scenarios in image representation, or image denoising \cite{sabres, sabres3}. 


In this work, we set $\lambda = \lambda_{0} N$ in (P1), where $\lambda_0>0$ is a constant. This setting allows $\lambda$ to scale with the size of the data (i.e., total number of patches). In practice, the weight $\lambda_0$ needs to be set according to the expected range (in intensity values)  of the underlying image, as well as depending on the desired condition number of the learnt transform. The weight $\nu$ in (P1) is set similarly as in (P0).

When a unitary sparsifying transform is preferred, the $Q(W)$ regularizer in (P1) (and in \eqref{tLreg1}) could instead be replaced by the constraint $W^{H} W = I$, where $I$ denotes the identity matrix, yielding the following formulation 
\begin{align}
\nonumber (\mathrm{P2})\:\:\: & \min_{x, W, B}\:  \nu \left \| Ax-y \right \|_{2}^{2}  + \sum_{j=1}^{N} \left \| W P_{j}x- b_{j} \right \|_{2}^{2}   \; \,\, s.t.\;\; W^{H}W=I, \;  \left \| B \right \|_{0}\leq s,  \;\; \left \| x \right \|_{2}\leq C.
\end{align}
The unitary sparsifying transform case is special, in that Problem (P2) is also a unitary synthesis dictionary-based blind compressed sensing problem, with $W^{H}$ denoting the synthesis dictionary. This follows from the identity $\|W P_{j}x- b_{j} \|_{2} = \| P_{j}x- W^{H} b_{j}  \|_{2}$, for unitary $W$.




\subsubsection{Properties of Transform BCS Formulations  - Identifiability and Uniqueness} \label{identyuniq}
The following simple proposition considers an ``error-free" scenario and establishes the global identifiability of the underlying image and sparse model in BCS via solving the proposed Problems (P1) or (P2). 

\vspace{0.04in}
\begin{proposition} \label{optimalmodelsbcs}
Let $x \in \mathbb{C}^{p}$ with $\left \| x \right \|_{2}\leq C$, and let $y = Ax$ with $A \in \mathbb{C}^{m \times p}$. Suppose that $W \in \mathbb{C}^{n \times n}$ is a unitary transform that sparsifies the collection of patches of $x$ as $\sum_{j=1}^{N}\left \| W P_{j}x \right \|_{0}\leq s$. Further, let $B$ denote the matrix that has $ W P_{j}x$ as its columns. Then, $(x, W, B)$ is a global minimizer of both Problems (P1) and (P2), i.e., it is identifiable by solving these problems.
\end{proposition}

\begin{proof}:
For the given $(x, W, B)$, the terms $\sum_{j=1}^{N} \left \| W P_{j}x- b_{j} \right \|_{2}^{2}$ and $\left \| Ax-y \right \|_{2}^{2}$ in (P1) and (P2) each attain their minimum possible value (lower bound) of zero. Since $W$ is unitary, the penalty $Q(W)$ in (P1) is also minimized by the given $W$. Notice that the constraints in both (P1) and (P2) are satisfied for the given $(x, W, B)$. Therefore, this triplet is feasible for both problems and achieves the minimum possible value of the objective in both cases. Thus, it is a global minimizer of both (P1) and (P2). 
\end{proof}

Thus, when ``error-free" measurements are provided, and the patches of the underlying image are \emph{exactly} sparse (as defined by the constraint in (P1)) in some unitary transform, Proposition \ref{optimalmodelsbcs} guarantees that the image as well as the model are jointly identifiable by solving (i.e., finding global minimizers in) (P1) (or, (P2)). 


An interesting topic, which we do not fully pursue here pertains to the condition(s) under which the underlying image in Proposition \ref{optimalmodelsbcs} is the unique minimizer of the proposed BCS problems.
The proposed problems do admit an equivalence class of solutions/minimizers with respect to the transform $W$ and the set of sparse codes $ \left \{ b_{j} \right \}_{j=1}^{N}$ \footnote{In the remainder of this work, when certain indexed variables are enclosed within braces, it means that we are considering the set of variables over the range of all the indices.}.
Given a particular minimizer $(x, W, B)$ of (P1) or (P2), we have that $(x, \Theta W, \Theta B)$ is another equivalent minimizer for all \emph{sparsity-preserving unitary matrices} $\Theta$, i.e., $\Theta$ such that $\Theta^{H} \Theta = I$ and $ \sum_{j}\left \| \Theta b_{j} \right \|_{0}\leq s$. For example, $\Theta$ can be a row permutation matrix, or a diagonal $\pm 1$ sign matrix.
Importantly however, the minimizer with respect to $x$ in (P1) or (P2) is invariant to the modification of $(W, B)$ by sparsity-preserving unitary matrices $\Theta$, i.e., the optimal $x$ in (P1) or (P2) remains the same for all such choices of $(\Theta W, \Theta B)$.

We note that by imposing additional structure on $W$ in our transform-based BCS formulations, one can derive conditions for the uniqueness of the minimizers. Assume a global minimizer $(x, W, B)$ exists in (P2) satisfying the conditions (error-free scenario) in Proposition \ref{optimalmodelsbcs}.
Then, for example, when the unitary transform $W$ in (P2) is further constrained to be doubly sparse, i.e., $W = S \Phi$, with $S$ a sparse matrix and $\Phi$ a known matrix (e.g., DCT, or Wavelet $\Phi$), then because $W^{H} = \Phi^{H} S^{H}$ is an equivalent (doubly sparse) synthesis dictionary (corresponding to the transform $W$), the uniqueness conditions (involving the spark condition) proposed in prior work on synthesis dictionary-based BCS \cite{Glei12} (Section V-A of \cite{Glei12}) can be extended to the transform-based setting here. A detailed analysis and description of such uniqueness results will be presented elsewhere.


\subsubsection{Problem Formulations with Sparsity Penalties}

While Problem (P1) involves a sparsity constraint, an alternative version of Problem (P1) is obtained by replacing the $\ell_{0}$ sparsity constraint with an $\ell_{0}$ penalty in the objective (and in \eqref{tLreg1}), in which case we have the following optimization problem
\begin{align}
\nonumber (\mathrm{P3})\:\:\: & \min_{x,W, B}\:  \sum_{j=1}^{N} \left \| W P_{j}x- b_{j} \right \|_{2}^{2} + \nu \left \| Ax-y \right \|_{2}^{2} + \lambda \, Q(W) + \, \eta^{2}  \left \| B \right \|_{0} \;\, s.t.  \;\; \left \| x \right \|_{2}\leq C.
\end{align}
where $\eta^{2}$, with $\eta > 0$, denotes the weight for the sparsity penalty. 

A version of Problem (P3) (without the $\ell_{2}$ constraint) has been used very recently in adaptive tomographic reconstruction \cite{luke1, luke2}. 
However, it is interesting to note that in the absence of the $\left \| x \right \|_{2}\leq C$ condition, the objective in (P3) is actually non-coercive.
To see this, consider $W = I$ and $x_{\beta} = x_{0} + \beta z$, where $x_{0}$ is a particular solution to $y=Ax$, $\beta \in \mathbb{R}$, and $z \in \mathcal{N}(A)$ with $\mathcal{N}(A)$ denoting the null space of $A$. For this setting, as $\beta \to \infty$ with the $\mathrm{j^{th}}$ sparse code in (P3) set to $W P_{j} x_{\beta}$, it is obvious that the objective in (P3) remains finite, thereby making it non-coercive. The energy constraint on $x$ in (P3) restricts the set of feasible images to a compact set, and alleviates potential problems (such as unbounded iterates within a minimization algorithm)
due to the non-coercive objective.


While a single weight $\eta^{2}$ is used for the sparsity penalty $\left \| B \right \|_{0} = \sum _{j=1}^{N} \left \| b_{j} \right \|_{0}$ in (P3), one could also use different weights $\eta_{j}^{2}$ for the sparsity penalties $ \left \| b_{j} \right \|_{0} $  corresponding to different patches, if such weights are known, or estimated.



Just as Problem (P3) is an alternative to Problem (P1), we can also obtain a corresponding alternative version (denoted as (P4)) of Problem (P2) by replacing the sparsity constraint with a penalty. Although, in the rest of this paper, we consider Problems (P1)-(P3), the proposed algorithms and convergence results in this work easily extend to the case of (P4).


\subsection{Extensions}
While the proposed sparsifying transform-based BCS problem formulations are for the (extreme) scenario when the CS measurements corresponding to a single image are provided, these formulations can be easily extended to other scenarios too. For example, when multiple images (or frames, or slices) have to be jointly reconstructed using a single adaptive (spatial) sparsifying transform, then the objectives in Problems (P1)-(P3) for this case are the summation of the corresponding objective functions for each image. 
In applications such as dynamic MRI (or for example, compressive video), the proposed formulations can be extended by considering adaptive spatiotemporal sparsifying transforms of 3D patches (cf. \cite{wangying} that extends Problem (P0) in such a way to compressed sensing dynamic MRI). Similar extensions are also possible for higher-dimensional applications such as 4D imaging.

\section{Algorithm and Properties}  \label{sec3}

\subsection{Algorithm}


Here, we propose block coordinate descent-type algorithms to solve the proposed transform-based BCS problem formulations (P1)-(P3). Our algorithms alternate between solving for the sparse codes $\left \{ b_{j} \right \}$ (\emph{sparse coding step}), transform $W$ (\emph{transform update step}), and image $x$ (\emph{image update step}), with the other variables kept fixed. One could also alternate a few times between the sparse coding and transform update steps, before performing one image update step. 
In the following, we describe the three main steps in detail. We show that each of the steps has a simple solution that can be computed cheaply in practical applications such as MRI.


\subsubsection{Sparse Coding Step}
The sparse coding step of our algorithms for Problems (P1) and (P2) involves the following optimization problem
\begin{equation} \label{bcs1}
\min_{B}\:  \sum_{j=1}^{N} \left \| W P_{j}x- b_{j} \right \|_{2}^{2} \;\, \, s.t.\;\;  \left \| B \right \|_{0}\leq s.
\end{equation}
Now, let $Z \in \mathbb{C}^{n \times N}$ be the matrix with the transformed (vectorized) patches $W P_{j}x$ as its columns. Then, using this notation, Problem \eqref{bcs1} can be rewritten as follows, where $\left \| \cdot \right \|_{F}$ denotes the standard Frobenius norm.
\begin{equation} \label{bcs2}
\min_{B}\: \left \| Z - B \right \|_{F}^{2} \;\, \, s.t.\;\;  \left \| B \right \|_{0} \leq s.
\end{equation}
The above problem is to project $Z$ onto the non-convex set $\left \{ B \in \mathbb{C}^{n \times N} : \left \| B \right \|_{0} \leq s \right \}$ of matrices that have sparsity $\leq s$, which we call the $s$-$\ell_0$ ball. The optimal projection $\hat{B}$ is easily computed by zeroing out all but the $s$ coefficients of largest magnitude in $Z$. We denote this operation by $\hat{B} = H_{s}(Z)$, where $H_{s}(\cdot)$ is the corresponding projection operator. In case, there is more than one choice for the $s$ elements of largest magnitude in $Z$, then $H_{s}(Z)$ is chosen as the projection for which the indices of these $s$ elements are the lowest possible in lexicographical order.

In the case of Problem (P3), the sparse coding step involves the following unconstrained (and non-convex) optimization problem
\begin{equation} \label{bcs4}
 \min_{B}\: \left \| Z - B \right \|_{F}^{2} + \eta^{2}  \left \| B \right \|_{0}
\end{equation}
The optimal solution $\hat{B}$ in this case is obtained as $\hat{B} =  \hat{H}_{\eta}^{1} (Z)$, with the hard-thresholding operator $\hat{H}_{\eta}^{1} (\cdot)$ defined as follows, where the subscript $ij$ indexes matrix entries ($i$ for row and $j$ for column).
\begin{equation} \label{bcs5}
 \left ( \hat{H}_{\eta}^{1} (Z) \right )_{ij}=\left\{\begin{matrix}
 0&, \;\;\left | Z_{ij} \right | < \eta \\
Z_{ij}  & ,\;\;\left | Z_{ij} \right | \geq \eta 
\end{matrix}\right.
\end{equation}
The optimal solution to Problem \eqref{bcs4} is not unique when the condition $\left | Z_{ij} \right | = \eta$ is satisfied for some $i$, $j$ (cf. Page 3 of \cite{sbclsTS2} for a similar scenario and an explanation).
The definition in \eqref{bcs5} chooses \emph{one} of the multiple optimal solutions in this case.


\subsubsection{Transform Update Step}

Here, we solve for $W$ in the proposed formulations, with the other variables kept fixed. In the case of Problems (P1) and (P3), this involves the following optimization problem
\begin{equation} \label{bcs6}
 \min_{W}\:  \sum_{j=1}^{N} \left \| W P_{j}x- b_{j} \right \|_{2}^{2} + 0.5 \lambda \left \| W \right \|_{F}^{2}  -  \lambda \log \,\left | \mathrm{det \,} W \right |  
\end{equation}
Now, let $X \in \mathbb{C}^{n \times N}$ be the matrix with the vectorized patches $P_{j}x$ as its columns, and recall that $B$ is the matrix of codes $b_{j}$. Then, Problem \eqref{bcs6} becomes
\begin{equation} \label{bcs7}
 \min_{W}\:  \left \| W X - B \right \|_{F}^{2} + 0.5 \lambda \left \| W \right \|_{F}^{2}  - \lambda \log \,\left | \mathrm{det \,} W \right |  
\end{equation}
An analytical solution for this problem has been recently derived \cite{sabres3, sbclsTS2}, and is stated in the following proposition. It is expressed in terms of an appropriate singular value decomposition (SVD). We let $M^{\frac{1}{2}}$ denote the positive definite square root of a positive definite matrix $M$.
\vspace{0.03in}
\begin{proposition}\label{propel1bcs} 
Given $X \in \mathbb{C}^{n \times N}$, $B \in \mathbb{C}^{n \times N}$,  and $\lambda>0$, factorize $ XX^{H} + 0.5 \lambda I$ as $LL^{H}$, with $L \in \mathbb{C}^{n \times n}$. Further, let $L^{-1}XB^{H}$ have a full SVD of $V \Sigma R^{H}$. Then, a global minimizer for the transform update step \eqref{bcs7} is 
\begin{equation} \label{tru1bcs}
\hat{W}=0.5  R \left(\Sigma+ \left ( \Sigma^{2}+2\lambda I \right )^{\frac{1}{2}}\right)V^{H}L^{-1}
\end{equation}
The solution is unique if and only if $XB^{H}$ is non-singular. Furthermore, the solution is invariant to the choice of factor $L$. 
\end{proposition}

\begin{proof}: See the proof of Proposition 1 of \cite{sbclsTS2}, particularly the discussion following that proof.
\end{proof} 


The factor $L$ in Proposition \ref{propel1bcs} can for example be the
factor $L$ in the Cholesky factorization $ XX^{H} + 0.5 \lambda I = LL^{H}$, or the Eigenvalue Decomposition (EVD) square root of $XX^{H} + 0.5 \lambda I$. The closed-form solution \eqref{tru1bcs} is nevertheless invariant to the particular choice of $L$ \cite{sbclsTS2}. Although in practice both the SVD and the square root of non-negative scalars are computed using iterative methods, we will assume in the convergence analysis in this work, that the solution \eqref{tru1bcs} (as well as later ones that involve such computations) is computed exactly. 
In practice, standard numerical methods are guaranteed to quickly provide machine precision accuracy for the SVD or other (aforementioned) computations.



In the case of Problem (P2), the transform update step involves the following problem
\begin{equation} \label{bcs8}
 \min_{W}\:   \left \| W X- B \right \|_{F}^{2}  \; \,\, s.t.\;\; W^{H}W=I.
\end{equation}
The solution to the above problem can be expressed as follows (see \cite{sabres3}, or Proposition 2 of \cite{sbclsTS2}).
\vspace{0.03in}
\begin{proposition}\label{propel2bcs} 
Given $X \in \mathbb{C}^{n \times N}$ and $B \in \mathbb{C}^{n \times N}$, let $XB^{H}$ have a full SVD of $U \Sigma V^{H}$. Then, a global minimizer in \eqref{bcs8} is 
\begin{equation} \label{tru2bcs}
\hat{W}= VU^{H}
\end{equation}
The solution is unique if and only if $XB^{H}$ is non-singular.
\end{proposition}

\subsubsection{Image Update Step: General Case} \label{imupstepg}


In this step, we solve Problems (P1)-(P3) for the image $x$, with the other variables fixed. This involves the following optimization problem
\begin{equation} \label{bcs9}
 \min_{x}\:  \sum_{j=1}^{N} \left \| W P_{j}x- b_{j} \right \|_{2}^{2} + \nu \left \| Ax-y \right \|_{2}^{2} \;\, s.t.  \;\; \left \| x \right \|_{2}\leq C.
\end{equation}
Problem \eqref{bcs9} is a least squares problem with an $\ell_{2}$ (alternatively, squared $\ell_{2}$) constraint \cite{golkah1}.
It can be solved exactly by using the Lagrange multiplier method \cite{golkah1}. 
The corresponding Lagrangian formulation is
\begin{equation} \label{bcs9b}
 \min_{x}\:  \sum_{j=1}^{N} \left \| W P_{j}x- b_{j} \right \|_{2}^{2} + \nu \left \| Ax-y \right \|_{2}^{2} +  \mu \left ( \left \| x \right \|_{2}^{2} - C \right )
\end{equation}
where $\mu \geq 0$ is the Lagrange multiplier. The solution to \eqref{bcs9b} satisfies the following Normal Equation
\begin{equation} \label{bcs10}
\left (  G \; +\;\nu\: A^{H}A + \mu I \right )x= \sum_{j=1}^{N} P_{j}^{T}W^{H} b_{j}\;+\:\nu \: A^{H}y
\end{equation}
where 
\begin{equation} \label{Gdef11}
G \triangleq \sum_{j=1}^{N} P_{j}^{T}W^{H}W P_{j}
\end{equation}
and $(\cdot)^{T}$ (matrix transpose) is used instead of $(\cdot)^{H}$ above for real matrices. The solution to \eqref{bcs10} is unique for any $\mu \geq 0$ because matrix $G$ is positive-definite. To see why, consider any $z \in \mathbb{C}^{p}$. Then, we have
$z^{H} G z   = \sum_{j=1}^{N}\left \|W P_{j}z  \right \|_{2}^{2}$,
which is strictly positive unless $W P_{j}z = 0$ $\forall$ $j$. Since the $W$ in our algorithm is ensured to be invertible, we have that $W P_{j}z = 0$ $\forall$ $j$ if and only if $P_{j}z = 0$ $\forall$ $j$, which implies (assuming that the set of patches in our formulations covers all pixels in the image) that $z=0$. This ensures $G \succ 0$. 
The \emph{unique} solution to \eqref{bcs10} can be found by direct methods (for small-sized problems), or by conjugate gradients (CG).




To solve the original Problem \eqref{bcs9}, the Lagrange multiplier $\mu$ in \eqref{bcs10} must also be chosen optimally. This is done by first computing the EVD of the $p \times p$ matrix $G$ $+ \nu A^{H}A$ as $U \Sigma U^{H}$. Since $G \succ 0$, we have that $\Sigma\succ 0$. Then, defining $z \triangleq U^{H}\begin{pmatrix}
\sum_{j=1}^{N} P_{j}^{T}W^{H} b_{j}\;+\:\nu \: A^{H}y
\end{pmatrix}$, we have that \eqref{bcs10} implies $U^{H}x = \left ( \Sigma + \mu I \right )^{-1}z$. Therefore,
\begin{equation} \label{bcs11b}
\left \| x \right \|_{2}^{2} = \begin{Vmatrix}
U^{H}x
\end{Vmatrix}_{2}^{2} = \sum_{i=1}^{p} \frac{\left | z_{i} \right |^{2}}{\left ( \Sigma_{ii} + \mu \right )^{2}} \; \triangleq \tilde{f}(\mu)
\end{equation}
where $z_{i}$ above denotes the $\mathrm{i^{th}}$ entry of vector $z$. If $\tilde{f}(0) \leq C^{2}$, then $\hat{\mu}=0$ is the optimal multiplier. Otherwise, the optimal $\hat{\mu}>0$. In the latter case, since the function $\tilde{f}(\mu)$ in \eqref{bcs11b} is monotone decreasing for $\mu>0$, and $\tilde{f}(0) > C^{2}$, there is a unique multiplier $\hat{\mu}>0$ such that $\tilde{f}(\hat{\mu}) - C^{2}= 0$.
The optimal $\hat{\mu}$ is found by using the classical Newton's method (or, alternatively \cite{eldena}), which in our case is guaranteed to converge to the optimal solution at a quadratic rate. Once the optimal $\hat{\mu}$ is found (to within machine precision), the \emph{unique} solution to the image update Problem \eqref{bcs9} is $U \left ( \Sigma + \hat{\mu} I \right )^{-1}z$, coinciding with the solution to \eqref{bcs10} with $\mu = \hat{\mu}$.

In practice, when a large value (or, loose estimate) of $C$ is used (for example, in our experiments later in Section \ref{sec5}), the optimal solution to \eqref{bcs9} is typically obtained with the setting $\hat{\mu} = 0$ for the multiplier in \eqref{bcs9b}. In this case, the unique minimizer of the objective in \eqref{bcs9} (e.g., obtained with CG) directly satisfies the constraint. Therefore, the additional computations (e.g., EVD) to find the optimal $\hat{\mu}$ can be avoided in this case. 
Other alternative ways to find the solution to \eqref{bcs9} when the optimal $\hat{\mu} \neq 0$, without the EVD computation, include using the projected gradient method, or solving \eqref{bcs9b} repeatedly (by CG) for various $\mu$ (tuned in steps) until the $\left \| x \right \|_{2}= C$ condition is satisfied.




When employing CG (or, the projected gradient method), the structure of the various matrices in \eqref{bcs10} can be exploited to enable efficient computations.
First, we show that under certain assumptions, the matrix $G$ in \eqref{bcs10} is a 2D circulant matrix, i.e., a \emph{Block Circulant matrix with Circulant Blocks} (abbreviated as BCCB matrix), enabling efficient computation (via FFTs) of the product $Gx$ (used in CG).
Second, in certain applications, the matrix $A^{H} A$ in \eqref{bcs10} may have a structure (e.g., sparsity, Toeplitz structure, etc.) that enables efficient computation of $A^{H} A x$ (used in CG). In such cases, the CG iterations can be performed efficiently.

We now show the matrix $G$ in \eqref{bcs10} is a BCCB matrix. To do this, we make some assumptions on how the overlapping 2D patches are selected from the image(s) in our formulations.
First, we assume that  periodically positioned, overlapping 2D image patches are used.
Furthermore, the patches that overlap the image boundaries are assumed to `wrap around' on the opposite side of the image \cite{bresai}. Now, defining the patch \emph{overlap stride} $r$ to be the distance in pixels between corresponding pixel locations in adjacent image patches, it is clear that the setting $r=1$ results in a maximal set of overlapping patches. 
When $r=1$ (and assuming patch `wrap around'), the following proposition establishes that the matrix $G$ is a Block Circulant matrix with Circulant Blocks. Let $ F  \in \mathbb{C}^{p \times p} $ denote the full (2D) Fourier encoding matrix assumed normalized such that $ F^{H}F = I$. 


\begin{proposition}\label{propel5} 
Let $r=1$, and assume that all `wrap around' image patches are included. Then, the matrix $G = \sum_{j=1}^{N} P_{j}^{T}W^{H}W P_{j}$ in \eqref{bcs10} is a BCCB matrix with eigenvalue decomposition $F^{H} \Gamma F$, with $\Gamma \succ 0$.
\end{proposition}

\begin{proof}:
First, note that $ W^{H} W = \sum_{i=1}^{n} e_{i}e_{i}^{T}W^{H} W$, where  $\left \{ e_{i} \right \}_{i=1}^{n}$ are the columns of the $n \times n$ identity matrix. Denote the $\mathrm{i^{th}}$ row of $W^{H}W$ by $h_i$.  Then, the matrix $e_{i}e_{i}^{T}W^{H} W$ is all zero except for its $\mathrm{i^{th}}$ row, which is equal to $h_i$.
Let $G_i \triangleq \sum_{j=1}^{N} P_{j}^{T} \begin{pmatrix}
e_{i}e_{i}^{T}W^{H} W
\end{pmatrix} P_{j}$. Then, $G = \sum_{i=1}^{n} G_{i}$.

Consider a vectorized image $z \in \mathbb{C}^{p}$.
Because the set of entries of $G_{i} z$ is simply the set of inner products between $h_{i}$ and all the (overlapping and wrap around) patches of the image corresponding to $z$, it follows that
applying $G_i$ on 2D circularly shifted versions of the image corresponding to $z$ produces correspondingly shifted versions of the output $G_{i} z$. 
Hence, $G_i$ is an operator corresponding to 2D circular convolution.

Now, it follows that each $G_{i}$ is a BCCB matrix. Since $G = \sum_{i=1}^{n} G_{i}$ is a sum of BCCB matrices, it is therefore a BCCB matrix as well. Now, from standard results regarding BCCB matrices, we know that $G$ has an EVD of $F^{H} \Gamma F$, where $\Gamma$ is a diagonal matrix of eigenvalues. It was previously established that $G$ is positive-definite. Therefore, $\Gamma \succ 0$ holds.
\end{proof} 






Proposition \ref{propel5} states that matrix $G$ in \eqref{bcs10} is a BCCB matrix, that is diagonalizable by the Fourier basis. Therefore, $G$ can be applied (via FFTs) on a vector (e.g., in CG) at $O(p \log p)$ cost. This is much lower than the $O(n^{2} p)$ cost ($n^{2} >> \log p$ typically) of applying $G = \sum_{j=1}^{N} P_{j}^{T}W^{H}W P_{j}$ directly using patch-based operations.
We now show how the diagonal matrix $\Gamma$ in Proposition \ref{propel5} can be computed efficiently. 
First, in Proposition \ref{propel5}, in the case that $W$ is a unitary matrix (as in (P2)), the matrix $G = \sum_{j=1}^{N} P_{j}^{T}P_{j}$ is simply the identity scaled by a factor $n$ \cite{bresai}. In this case, $\Gamma = n I$. Second, when $W$ is non-unitary, let us assume that the Fourier matrix $F$ is arranged so that its first column is the constant column (with entries $=1/\sqrt{p}$). Then, the diagonal of $\Gamma$ in this case is obtained efficiently (via FFTs) as $\sqrt{p} F a_{1}$, where $a_{1}$ is the first column of $G$. Now, the first column of $G$ can itself be easily computed by applying this operator (using simple patch-based operations) on $z \in \mathbb{C}^{p}$ that has a one in its first entry (corresponding to an image with a 1 in the top left corner) and zeros elsewhere. Note that since $z$ is extremely sparse, $a_1$ is computed at a negligible cost. The aforementioned computation for finding the diagonal of $\Gamma$ is equivalent, of course, to finding the spectrum of the operator corresponding to $G$, as the DFT of its impulse response. 

\subsubsection{Image Update Step: Case of MRI} \label{imupstepm}

In certain scenarios, the optimal $\hat{x}$ in \eqref{bcs9} can be found very efficiently. Here, we consider the case of Fourier imaging modalities, or more specifically, MRI, where $A= F_{u}$, the undersampled Fourier encoding matrix. In order to obtain an efficient solution for the $\hat{x}$ in \eqref{bcs9}, we assume that the k-space measurements in MRI are obtained by subsampling on a uniform Cartesian grid. 
We then show that the optimal multiplier $\hat{\mu}$ and the corresponding optimal $\hat{x}$ can be computed without any EVD computations, or CG, for MRI.







Assume $r=1$, and that all `wrap around' image patches are included in the formulation.
Then, empowered with the diagonalization result of Proposition \ref{propel5}, we can simplify equation \eqref{bcs10} for MRI, by rewriting it as follows
\begin{equation} \label{bcs12}
\left (F G F^{H}\; +\;\nu\: F F_{u}^{H}F_{u}F^{H} + \mu I \right )Fx = F \sum_{j=1}^{N} P_{j}^{T}W^{H} b_{j} \;+\:\nu \:FF_{u}^{H}y
\end{equation}
All $p$-dimensional vectors (vectorized images) in \eqref{bcs12} are in Fourier or k-space.
Vector $ FF_{u}^{H}y \in \mathbb{C}^{p}$ represents the zero-filled (or, zero padded) k-space measurements. 
The matrix $ FF_{u}^{H}F_{u}F^{H} $ is a diagonal matrix consisting of ones and zeros, with the ones at those diagonal entries that correspond to sampled locations in k-space.
By Proposition \ref{propel5}, for $r=1$, the matrix $F G F^{H} = \Gamma$ is diagonal. Therefore, the matrix pre-multiplying $ Fx $ in \eqref{bcs12} is diagonal and invertible. Denoting the diagonal of $\Gamma$ by $\gamma \in \mathbb{R}^{p}$ (all positive vector), $ S_{0} \triangleq FF_{u}^{H}y $,
and $S \triangleq F \sum_{j=1}^{N} P_{j}^{T}W^{H} b_{j}$, we have that the solution to \eqref{bcs12} for fixed $\mu$ is
\begin{equation}\label{bcs13}
Fx_{\mu} \:(k_{x},k_{y})=\left\{\begin{matrix}
\frac{S(k_{x},k_{y})}{\gamma(k_{x},k_{y})+ \mu} &,\:(k_{x},k_{y})\notin \Omega  \\
\frac{S(k_{x},k_{y})+\nu\, S_{0}(k_{x},k_{y})}{\gamma(k_{x},k_{y})+\nu + \mu}   & ,\:(k_{x},k_{y})\in \Omega
\end{matrix}\right.
\end{equation}
where $(k_{x},k_{y})$ indexes k-space locations, and $ \Omega $ represents the subset of k-space that has been sampled.
Equation \eqref{bcs13} provides a closed-form solution to the Lagrangian Problem \eqref{bcs9b} for CS MRI, with $Fx_{\mu} (k_{x},k_{y}) $ representing the optimal updated value (for a particular $\mu$) in k-space at location $ (k_{x},k_{y})$.


The function $\tilde{f}(\mu)$ in \eqref{bcs11b} now has a simple form (no EVD needed) as follows
\begin{equation} \label{bcs13b}
\tilde{f}(\mu)  = \begin{Vmatrix}
Fx
\end{Vmatrix}_{2}^{2} = \sum_{(k_{x},k_{y})\notin \Omega} \frac{\left | S(k_{x},k_{y}) \right |^{2}}{\left ( \gamma(k_{x},k_{y})+ \mu \right )^{2}} + \sum_{(k_{x},k_{y})\in \Omega} \frac{\left | S(k_{x},k_{y})+\nu\, S_{0}(k_{x},k_{y}) \right |^{2}}{\left ( \gamma(k_{x},k_{y})+\nu + \mu \right )^{2}}
\end{equation}
We check if $\tilde{f}(0) \leq C^{2}$ first, before applying Newton's method to solve $\tilde{f}(\hat{\mu}) = C^{2}$.
The optimal $ \hat{x} $ in \eqref{bcs9} is obtained via a 2D inverse FFT of the updated $ Fx_{\hat{\mu}} $ in \eqref{bcs13}.



The pseudocodes of the overall Algorithms A1, A2, and A3 corresponding to the BCS Problems (P1), (P2), and (P3) respectively, are shown below.
An appropriate choice for the initial $\left ( W^{0}, B^{0}, x^{0} \right )$  in Algorithms A1-A3 would depend on the specific application.
For example, $W^0$ could be the $n \times n$ 2D DCT matrix, $x^0 = A^{\dagger} y$ (normalized so that it satisfies $\left \| x^0 \right \|_{2} \leq C$), and $B^0$ can be the minimizer of Problems (P1)-(P3) for these fixed $W^0$ and $x^0$.

\subsection{Computational Cost}


Algorithms A1, A2, and A3 involve the steps of sparse coding, transform update, and image update. We now briefly discuss the computational costs of each of these steps.

First, in each outer iteration of our Algorithms A1 and A3, the computation of the matrix $ XX^{H} + 0.5 \lambda I$, where $X$ has the image patches as its columns, can be done in $O(n^{2}N)$ operations, where typically $n \ll N$. The computation of the inverse square root $L^{-1}$ requires only $O(n^{3})$ operations.

\begin{algorithm}[!t]
\caption*{Transform-Based BCS Algorithms \protect\footnotemark A1, A2, and A3}   
\textbf{Inputs:} \: $ y $ - measurements obtained with sensing matrix $A$, $s$ - sparsity, $\lambda$ - weight, $\nu$ - weight, $C$ - energy bound, $\hat{M}$ - number of inner iterations, $J$ - number of outer iterations.\\
 \textbf{Outputs:} \: $ x $  - reconstructed image, $W$ - adapted sparsifying transform, $B$ - matrix with sparse codes of all overlapping patches as columns. \\
\textbf{Initial Estimates:} \: $\left ( W^{0}, B^{0}, x^{0} \right )$. \\
\textbf{For \;t = 1: $\mathbf{J}$ Repeat}
\begin{enumerate} [1)] \setcounter{enumi}{0}
\item Form the matrix $X$ with $P_{j}x^{t-1}$ as its columns. Compute $L^{-1} = \left ( XX^{H} + 0.5 \lambda I \right )^{-1/2}$ for Algorithms A1 and A3. Set $\tilde{B}^{0} = B^{t-1}$.
\item \textbf{For \;l = 1: $\mathbf{\hat{M}}$ Repeat}
\begin{enumerate}
\item \textbf{Transform Update Step:}
\begin{enumerate}
\item Set $V \Sigma R^{H}$ as the full SVD of $L^{-1}X \begin{pmatrix} \tilde{B}^{l-1} \end{pmatrix}^{H}$ for Algorithms A1 and A3, or the full SVD of $X \begin{pmatrix} \tilde{B}^{l-1} \end{pmatrix}^{H}$ for Algorithm A2.
\item $\tilde{W}^{l} =0.5  R \left(\Sigma+ \left ( \Sigma^{2}+2\lambda I \right )^{\frac{1}{2}}\right)V^{H}L^{-1}$ for Algorithms A1 and A3, or $\tilde{W}^{l} =  R V^{H}$ for Algorithm A2.
\end{enumerate}
\item \textbf{Sparse Coding Step:} $\tilde{B}^{l} = H_{s}\begin{pmatrix}
\tilde{W}^{l} X
\end{pmatrix}$ for Algorithms A1 and A2, or  $\tilde{B}^{l} =\hat{H}_{\eta}^{1} \begin{pmatrix}
\tilde{W}^{l} X
\end{pmatrix}$ for Algorithm A3.
\end{enumerate}
\textbf{End} 
\item Set $W^{t} = \tilde{W}^{\hat{M}}$ and $B^{t} = \tilde{B}^{\hat{M}}$. Set $b_{j}^{t}$ as the $\mathrm{j^{th}}$ column of $B^{t}$ $\forall$ $j$.
\item \textbf{Image Update Step:}
\begin{enumerate} 
\item For generic CS scheme, solve \eqref{bcs10} for $x^{t}$ with $\mu=0$, by linear CG. If $\begin{Vmatrix}
x^{t}
\end{Vmatrix}_{2}>C$,
\begin{enumerate}
\item Compute $U \Sigma U^{H}$ as EVD of $\sum_{j=1}^{N} P_{j}^{T} \begin{pmatrix}
W^{t}
\end{pmatrix}^{H} W^{t} P_{j} \; +\;\nu\: A^{H}A$. 
\item Compute $z = U^{H}\begin{pmatrix} \sum_{j=1}^{N} P_{j}^{T}\begin{pmatrix} W^{t} \end{pmatrix}^{H} b_{j}^{t}\;+\:\nu \: A^{H}y \end{pmatrix}$.
\item Use Newton's method to find $\hat{\mu}$ such that $\tilde{f}(\hat{\mu}) = C^{2}$ in \eqref{bcs11b}.
\item $x^{t} = U \left ( \Sigma + \hat{\mu} I \right )^{-1} z$.
\end{enumerate}
\item For MRI, do the following
\begin{enumerate} 
\item Compute the image $c = \sum_{j=1}^{N} P_{j}^{T}\begin{pmatrix}
W^{t}
\end{pmatrix}^{H} b_{j}^{t}$. $S \leftarrow FFT(c)$.
\item Compute $a_{1}$ as the first column of $\sum_{j=1}^{N} P_{j}^{T} \begin{pmatrix}
W^{t}
\end{pmatrix}^{H} W^{t} P_{j}$.
\item Set $\gamma \leftarrow   \sqrt{p} \times FFT(a_{1})$.
\item Compute $\tilde{f}(0)$ as per \eqref{bcs13b}.
\item If $\tilde{f}(0) \leq C^{2}$, set $\hat{\mu}=0$. Else, use Newton's method to solve $\tilde{f}(\hat{\mu}) = C^{2}$ for $\hat{\mu}$.
\item Update $S$ to be the right hand side of \eqref{bcs13} with $\mu=\hat{\mu}$. $x^{t} = IFFT(S)$.
\end{enumerate}
\end{enumerate}
\end{enumerate}
\textbf{End}
\end{algorithm}
\footnotetext{The superscripts $t$ and $l$ denote the main iterates, and the iterates in the inner alternations between transform update and sparse coding, respectively. The $FFT$ and $IFFT$ denote the fast implementations of the normalized 2D DFT and 2D IDFT. For MRI, $r=1$, and the encoding matrix $F$ is assumed normalized, and arranged so that its first column is the constant DC column. In Step 4a, although we list the image update method involving EVD, an alternative scheme is one mentioned in the text in Section \ref{imupstepg}, involving only CG.}




The cost of the sparse coding step in our algorithms is dominated by the computation of the matrix $Z=WX$ in \eqref{bcs2} (for Algorithms A1, A2) or \eqref{bcs4} (for Algorithm A3), and therefore scales as $O(n^{2}N)$. Notably, the projection onto the $s$-$\ell_0$ ball in \eqref{bcs2} costs only $ O(nN \log N) $ operations, when employing sorting \cite{sabres}, with $\log N \ll n$ typically. Alternatively, in the case of \eqref{bcs4}, the hard thresholding operation costs only $O(nN)$ comparisons.


The cost of the transform update step of our algorithms is dominated by the computation of the matrix $XB^{H}$. Since $B$ is sparse, $XB^{H}$ is computed with $\alpha n^{2} N$ multiply-add operations, where $\alpha < 1$ is the fraction of non-zeros in $B$. 



The cost of the image update step in our algorithms depends on the specific application (i.e., the specific structure of  $A^{H} A$, etc.).
Here, we discuss the cost of the image update step discussed in Section \ref{imupstepm}, for the specific case of MRI. We assume $r=1$, and that the patches `wrap around', which implies that $N=p$ (i.e., number of patches equals number of image pixels). 
The computational cost here is dominated by the computation of the term $\sum_{j=1}^{N} P_{j}^{T}W^{H} b_{j}$ in the normal equation \eqref{bcs10}, which takes $O(n^{2} N)$ operations.
On the other hand, the various FFT and IFFT operations cost only $O(N \log N)$ operations, where $\log N \ll n^{2}$ typically. The Newton's method to compute the optimal multiplier $\hat{\mu}$ is only used when $\mu=0$ is non-optimal. In the latter case, Newton's method takes up $O(N\tilde{J})$ operations, with $\tilde{J}$ being the number of iterations (typically small, and independent of $n$) of Newton's method. 


Based on the preceding arguments, it is easy to observe that the total cost per (outer) iteration of the Algorithms A1-A3 scales (for MRI) as $O(n^{2} N \hat{M})$. Now, the recent synthesis dictionary-based BCS method called DLMRI \cite{bresai} learns a dictionary $D \in \mathbb{C}^{n \times K}$ from CS MRI measurements by solving Problem (P0). For this scheme, the computational cost per outer iteration scales as $ O(N Kn s \hat{J}) $ \cite{sabres}, where $\hat{J}$ is the number of (inner) iterations of dictionary learning (using the K-SVD algorithm \cite{elad}), and the other notations are the same as in (P0). Assuming that $K \propto n$, and that the synthesis sparsity $s \propto n$, we have that the cost per iteration of DLMRI scales as $O(n^{3} N \hat{J})$. Thus, the per-iteration computational cost of the proposed BCS schemes is much lower (lower in order by factor $n$ assuming $\hat{M} \sim \hat{J}$) than that for synthesis dictionary-based BCS. This gap in computations is amplified for higher-dimensional imaging applications such as 3D or 4D imaging, where the size of the 3D or 4D patches is typically much bigger than in the case of 2D imaging.

As illustrated in our experiments in Section \ref{sec5}, the proposed BCS algorithms converge in few iterations in practice. Therefore, the per-iteration computational advantages over synthesis dictionary-based BCS also typically translate to a net computational advantage in practice. 


\section{Convergence Results}  \label{sec4}

Here, we present convergence guarantees for Algorithms A1, A2, and A3, that solve Problems (P1), (P2), and (P3), respectively. These problems are highly non-convex. Notably they involve either an $\ell_0$ penalty or constraint for sparsity, a non-convex transform regularizer or constraint, and the term $\left \| W P_{j}x-b_{j} \right \|_{2}^{2}$ that is a non-convex function involving the product of unknown matrices and vectors. The proposed algorithms for Problems (P1)-(P3) are block coordinate descent-type algorithms. 
We previously established in Proposition \ref{optimalmodelsbcs}, the (noiseless) identifiability of the underlying image by solving the proposed problems (specifically, (P1) or (P2)). We are now interested in understanding whether the proposed algorithms converge to a minimizer of the corresponding problems, or whether they possibly get stuck in non-stationary points. Due to the high degree of non-convexity involved here, standard results on convergence of block coordinate descent methods (e.g., \cite{tseng6}) do not apply here.

Very recent works on the convergence of block coordinate descent-type algorithms (e.g., \cite{xu222}, or the Block Coordinate Variable Metric Forward-Backward algorithm \cite{emlie12}) prove convergence of the iterate sequence (for specific algorithm) to a critical point of the objective. However, these works make numerous assumptions, some of which can be easily shown to be violated for the proposed formulations (for example, the term $\sum_{j=1}^{N} \left \| W P_{j}x- b_{j} \right \|_{2}^{2}$ in the objectives of our formulations, although differentiable, violates the L-Lipschitzian gradient property described in Assumption 2.1 of \cite{emlie12}).




In fact, in certain simple scenarios, one can easily derive non-convergent iterate sequences for the Algorithms A1-A3. Non-convergence mainly arises for the transform or sparse code sequences (rather than the image sequence) due to the fact that the optimal solutions in the sparse coding or transform update steps may be non-unique. For example, in the trivial case when $y = 0$, the image $x = 0$ is the unique minimizer in the proposed BCS problems. Hence, if $x^{0}=0$, then the iterates in Algorithms A1-A3 easily satisfy $x^{t}=0$ $\forall$ $t$. Since the patches of $x^{t}$ are all zero, then any unitary matrix would be an (non-unique) optimal sparsifier in the transform update step of the algorithms, no matter how many inner alternations are performed between sparse coding and transform update within each outer iteration of the algorithms (the sparse codes are always $0$ in this case). Hence, the $\left\{ W^{t}\right\}$ sequence can be any sequence of unitary transforms in this case (a limit need not exist).

In this work, we provide convergence results for the proposed BCS approaches, where the only assumption is that the various steps in our algorithms are solved exactly. (Recall that machine precision is typically guaranteed in practice.)
Unlike prior works on dictionary-based blind compressed sensing \cite{bresai, wangying}, wherein the update steps such as the $\ell_{0}$ ``norm" based synthesis sparse coding are only solved approximately, our analysis here for the transform based BCS approaches makes use of the explicit solutions for the update steps in our algorithms, to prove convergence.



\subsection{Preliminaries} \label{prelim}

We first list some definitions that will be used in our analysis.
\begin{definition} \label{def2}
For a function $\phi: \mathbb{R}^{q} \mapsto (-\infty, + \infty]$, its domain is defined as $\mathrm{dom} \phi = \left \{ z \in \mathbb{R}^{q} : \phi(z) < + \infty \right \}$. Function $\phi$ is proper if $\mathrm{dom} \phi$ is nonempty.
\end{definition}

Next, we define the notion of Fr\'{e}chet sub-differential for a function as follows \cite{vari1, vari2}. The norm and inner product notations used below correspond to the euclidean $\ell_{2}$ settings.
\begin{definition} \label{def1}
Let $\phi: \mathbb{R}^{q} \mapsto (-\infty, + \infty]$ be a proper function and let $z \in \mathrm{dom} \phi$. The Fr\'{e}chet sub-differential of the function $\phi$ at $z$ is the following set:
\begin{equation}
\hat{\partial }\phi(z) \triangleq \begin{Bmatrix}
h \in \mathbb{R}^{q} : \underset{b \to z, b \neq z}{\lim \inf} \frac{1}{\left \| b-z \right \|}\left ( \phi(b) - \phi(z) - \left \langle b-z, h \right \rangle \right ) \geq 0
\end{Bmatrix}
\end{equation}
If $z \notin \mathrm{dom} \phi$, then $\hat{\partial }\phi(z) \triangleq \emptyset$.
The sub-differential of $\phi$ at $z$ is defined as
\begin{equation}
\partial \phi(z) \triangleq \begin{Bmatrix}
\tilde{h}  \in \mathbb{R}^{q} : \exists z_{k} \to z, \phi(z_{k}) \to \phi(z), h_{k} \in \hat{\partial }\phi(z_{k}) \to \tilde{h} 
\end{Bmatrix}.
\end{equation}
\end{definition}


A necessary condition for $z \in \mathbb{R}^{q}$ to be a minimizer of the function $\phi: \mathbb{R}^{q} \mapsto (-\infty, + \infty]$ is that $z$ is a \emph{critical point} of $\phi$, i.e., $0 \in \partial \phi(z)$. If $\phi$ is a convex function, this condition is also sufficient.
Critical points can be thought of as ``generalized stationary points" \cite{vari1}.

We say that a sequence $ \left \{ a^{t} \right \}$ with $a^{t} \in \mathbb{C}^{q}$ has an accumulation point $a$, if there is a subsequence that converges to $a$.

\subsection{Notations}


Problems (P1) and (P2) have the constraint $\left \| B \right \|_{0}\leq s$, which can instead be added as a penalty in the respective objectives by using a barrier function $\psi (B)$ (which takes the value $+ \infty$ when the sparsity constraint is violated, and is zero otherwise).
Problem (P2) also has the constraint $W^{H}W=I$, which can be equivalently added as a penalty in the objective of (P2) by using the barrier function $\varphi(W)$, which takes the value $+ \infty$ when the unitary constraint is violated, and is zero otherwise. Finally, the constraint $\left \| x \right \|_{2}\leq C$ in our formulations is replaced by the barrier function penalty $\chi(x)$. With these modifications, all the proposed problem formulations can be written in an unconstrained form. The objectives of (P1), (P2), and (P3), are then respectively denoted as 
\begin{align} 
g(W, B, x) & = \nu \left \| Ax-y \right \|_{2}^{2} + \sum_{j=1}^{N} \left \| W P_{j}x- b_{j} \right \|_{2}^{2} +  \lambda \, Q(W) + \psi (B) + \chi(x) \label{cnbcs1}\\
u(W, B, x) & =  \nu \left \| Ax-y \right \|_{2}^{2} + \sum_{j=1}^{N} \left \| W P_{j}x- b_{j} \right \|_{2}^{2} + \varphi(W)  + \psi (B) +  \chi(x) \label{cnbcs2}\\
v(W, B, x) & =  \nu \left \| Ax-y \right \|_{2}^{2} + \sum_{j=1}^{N} \left \| W P_{j}x- b_{j} \right \|_{2}^{2} +  \lambda \, Q(W) + \, \eta^{2}   \left \| B \right \|_{0} + \chi(x)  \label{cnbcs3}
\end{align}


It is easy to see that (cf. \cite{sbclsTS2} for a similar statement and justification) the unconstrained minimization problem involving the objective $g(W, B, x)$  is exactly equivalent to the corresponding constrained formulation (P1), in the sense that the minimum objective values as well as the set of minimizers of the two formulations are identical.
The same result also holds with respect to (P2) and $u(W, B, x)$, or (P3) and $v(W, B, x)$.


Since the functions $g$, $u$, and $v$ accept complex-valued (input) arguments, we will compute all derivatives or sub-differentials (Definition \ref{def1}) of these functions with respect to the (real-valued) real and imaginary parts of the variables ($W$, $B$, $x$). Note that the functions $g$, $u$, and $v$ are proper (we set the negative log-determinant penalty to be $+ \infty$ wherever $\det W = 0$) and lower semi-continuous.
For the Algorithms A1-A3, we denote the iterates (outputs) in each outer iteration $t$ by the set 
$\left ( W^{t}, B^{t}, x^{t} \right )$.

For a matrix $H$, we let $\rho_{j}(H)$ denote the magnitude of the $j^{\mathrm{th}}$ largest element (magnitude-wise) of the matrix $H$. For some matrix $E$, $\left \| E \right \|_{\infty} \triangleq \max_{i,j} \left | E_{ij} \right |$. Finally, $Re(A)$ denotes the real part of some scalar or matrix $A$.


\subsection{Main Results}

The following theorem provides the convergence result for Algorithm A1 that solves Problem (P1). We assume that the initial estimates $(W^{0}, B^{0}, x^{0})$ satisfy all problem constraints.

\begin{theorem}\label{theorem1bc}
Let $\left \{ W^{t}, B^{t}, x^{t} \right \}$ denote the iterate sequence generated by Algorithm A1 with measurements $y \in \mathbb{C}^{m}$ and initial $(W^{0}, B^{0}, x^{0})$. Then, the objective sequence  $\left \{ g^{t} \right \}$ with $g^{t} \triangleq g\left ( W^{t}, B^{t}, x^{t}  \right )$ is monotone decreasing, and converges to a finite value, say $g^{*} = g^{*}(W^{0}, B^{0}, x^{0})$. Moreover, the iterate sequence  $\left \{ W^{t}, B^{t}, x^{t} \right \}$ is bounded, and all its accumulation points are equivalent in the sense that they achieve the exact same value $g^{*}$ of the objective. The sequence $\left \{ a^{t} \right \}$ with $a^{t} \triangleq \left \| x^{t} - x^{t-1} \right \|_{2}$, converges to zero.
Finally, every accumulation point $(W, B, x)$ of $\left \{ W^{t}, B^{t}, x^{t} \right \}$ is a critical point of the objective $g$ satisfying the following partial global optimality conditions
\begin{align} 
x \in & \underset{\tilde{x}}{\arg\min} \; \,  g\left (W, B, \tilde{x} \right ) \label{cnbcs4}\\
W \in & \underset{\tilde{W}}{\arg\min} \; \,  g\left (\tilde{W}, B, x \right ) \label{cnbcs5}\\
B \in & \underset{\tilde{B}}{\arg\min} \; \,  g\left (W, \tilde{B}, x \right ) \label{cnbcs6}
\end{align}
Each accumulation point $(W, B, x)$ also satisfies the following partial local optimality conditions
\begin{align} 
g(W + dW, B + \Delta B, x) \geq &  g(W, B, x) = g^{*} \label{cnbcs4b}\\
g(W, B + \Delta B, x + \tilde{\Delta } x) \geq & g(W, B, x) = g^{*} \label{cnbcs5b}
\end{align}
The conditions each hold for all $\tilde{\Delta } x \in \mathbb{C}^{p}$, and all sufficiently small $dW \in \mathbb{C}^{n \times n}$ satisfying $\left \| dW \right \|_{F} \leq \epsilon'$ for some  $\epsilon' >0$ that depends on the specific $W$, and all $\Delta B \in \mathbb{C}^{n \times N}$ in the union of the following regions R1 and R2, where $X \in \mathbb{C}^{n \times N}$ is the matrix with $P_{j}x$, $1\leq j \leq N$, as its columns.
\begin{itemize}
\item[R1.] The half-space $Re\begin{pmatrix} tr\left \{ (WX-B)\Delta B^{H} \right \} \end{pmatrix} \leq 0$.  
\item[R2.] The local region defined by $\left \| \Delta B \right \|_{\infty} < \rho_{s}(WX)$.
\end{itemize}
Furthermore, if $ \left \|  WX \right \|_{0} \leq s$, then $\Delta B$ can be arbitrary. 
\end{theorem}


Theorem \ref{theorem1bc} establishes that for each initial point $(W^{0}, B^{0}, x^{0})$, the iterate sequence in Algorithm A1 converges to an equivalence class of accumulation points. Specifically, every accumulation point corresponds to the same value $g^{*} = g^{*}(W^{0}, B^{0}, x^{0})$ of the objective. The exact value of $g^{*}$ could vary with initialization. Importantly, the equivalent accumulation points are all critical points as well as \emph{at least} partial minimizers of the objective $g\left (W, B, x \right )$, in the following sense. Every accumulation point $(W, B, x)$ is a partial global minimizer of $g(W, B, x)$ with respect to each of $W$, $B$, and $x$, as well as a partial local minimizer of $g(W, B, x)$ with respect to $(W, B)$, and $(B, x)$, respectively. Therefore, we have the following corollary to Theorem \ref{theorem1bc}.

\begin{corollary}\label{corollary1a}
For each $(W^{0}, B^{0}, x^{0})$, the iterate sequence in Algorithm A1 converges to an equivalence class of critical points, that are also partial minimizers satisfying \eqref{cnbcs4}, \eqref{cnbcs5}, \eqref{cnbcs6}, \eqref{cnbcs4b}, and \eqref{cnbcs5b}.
\end{corollary}


Conditions \eqref{cnbcs4b} and \eqref{cnbcs5b} in Theorem \ref{theorem1bc} hold true not only for local (or small) perturbations of the sparse code matrix (accumulation point) $B$, but also for arbitrarily large perturbations of the sparse codes in a half space, as defined by region R1.
Furthermore, the partial optimality condition \eqref{cnbcs5b} also holds for arbitrary perturbations $\tilde{\Delta } x $ of $x$.

Theorem \ref{theorem1bc} also says that $\left \| x^{t} - x^{t-1} \right \|_{2} \to 0$. This is a necessary but not sufficient condition for convergence of the entire sequence $\left \{ x^{t} \right \} $.


The following corollary to Theorem \ref{theorem1bc} also holds, where `globally convergent' refers to convergence from any initialization. 

\begin{corollary}\label{corollary1b}
Algorithm A1 is globally convergent to a subset of the set of critical points of the non-convex objective $g\left (W, B, x \right )$. The subset includes all critical points $(W, B, x)$, that are at least partial global minimizers of $g(W, B, x)$ with respect to each of $W$, $B$, and $x$, as well as partial local minimizers of $g(W, B, x)$ with respect to each of the pairs $(W, B)$, and $(B, x)$. 
\end{corollary}

Theorem \ref{theorem1bc} holds for Algorithm A1 irrespective of the number of inner alternations $\hat{M}$, between transform update and sparse coding, within each outer algorithm iteration. In practice, we have observed that a larger value of $\hat{M}$ (particularly in initial algorithm iterations) may enable Algorithm A1 to be insensitive (for example, in terms of the quality of the image reconstructed) to the initial (even, badly chosen) values of $W^{0}$ and $B^{0}$. 

The convergence results for Algorithms A2 or A3 are quite similar to that for Algorithm A1. The following two Theorems briefly state the results for Algorithms A3 and A2, respectively.

\begin{theorem}\label{theorem2bc}
Theorem \ref{theorem1bc} applies to Algorithm A3 and the corresponding objective $v(W, B, x)$ as well, except that the set of perturbations $\Delta B \in \mathbb{C}^{n \times N}$ in Theorem \ref{theorem1bc} is restricted to $\left \| \Delta B \right \|_{\infty} < \eta/2$ for Algorithm A3.
\end{theorem}

\begin{theorem}\label{theorem3bc}
Theorem \ref{theorem1bc}, except for the condition \eqref{cnbcs4b}, applies to Algorithms A2 and the corresponding objective $u(W, B, x)$ as well.
\end{theorem}

Note that owing to Theorems \ref{theorem2bc} and \ref{theorem3bc}, results similar to Corollaries \ref{corollary1a} and \ref{corollary1b} also apply for Algorithms A3 and A2, respectively.
The proofs of the stated convergence theorems are provided in Appendix \ref{apbcs1}.

In general, the subset of the set of critical points to which each algorithm converges may be larger than the set of global minimizers (Section \ref{identyuniq}) in the respective problems. The question of the conditions under which the proposed algorithms converge to the (perhaps smaller) set of global minimizers in the proposed problems is open, and its investigation is left for future work.



\section{Numerical Experiments}  \label{sec5}

\subsection{Framework} \label{sec5framwork}


Here, we study the usefulness of the proposed sparsifying transform-based blind compressed sensing framework for the CS MRI application \footnote{We have also proposed another sparsifying transform-based BCS MRI method recently \cite{syber}. However, the latter approach involves many more parameters (e.g., error thresholds to determine patch-wise sparsity levels), which may be hard to tune in practice. In contrast, the methods proposed in this work involve only a few  parameters that are relatively easy to set.}.
The MR data used in these experiments are $512 \times 512$ complex-valued images shown (only the magnitude is displayed) in Fig. \ref{im1bcs}(a) and Fig. \ref{im2bcs}(a).
The image in  Fig. \ref{im1bcs}(a) was kindly provided by Prof. Michael Lustig, UC Berkeley, and is one image slice (with rich features) from a multislice data acquisition. The image in  Fig. \ref{im2bcs}(a) is publicly available \footnote{It can be downloaded from \url{http://web.stanford.edu/class/ee369c/data/brain.mat}. A phase-shifted version of the image was used in the experiments in \cite{samptabcs}.}.
We simulate various undersampling patterns in k-space
\footnote{We simulate the k-space of an image $x$ using the command fftshift(fft2(ifftshift(x))) in Matlab.
Fig. \ref{im2bcs}(b) shows the k-space (only magnitude is displayed) of the reference in Fig. \ref{im2bcs}(a).}
including variable density 2D random sampling \footnote{This sampling scheme is feasible when data corresponding to multiple image slices are jointly acquired, and the frequency encode direction is chosen perpendicular to the image plane.} \cite{josh, bresai}, and Cartesian sampling with (variable density) random phase encodes (1D random). We employ Problem (P1) and the corresponding Algorithm A1 to reconstruct images from undersampled measurements in the experiments here
\footnote{ Problem (P3) has been recently shown to be useful for adaptive tomographic reconstruction \cite{luke1, luke2}. The corresponding Algorithm A3 has the advantage that the sparse coding step involves the cheap hard thresholding operation, rather than the more expensive projection onto the $s$-$\ell_0$ ball (used in Algorithms A1 and A2). We have observed that Algorithm A3 also works well for MRI.
Problems (P1) and (P2) differ in that (P2) enforces unit conditioning of the learnt transform, whereas (P1) also allows for more general condition numbers. Algorithm A2 (for (P2)) involves slightly cheaper computations than Algorithm A1 (for (P1)).
In our experiments for MRI, we observed that well-conditioned transforms learnt via (P1) performed (in terms of image reconstruction quality) slightly better than strictly unitary learnt transforms. Therefore, we show results for (P1) in this work. We did not observe any dramatic difference in performance between the proposed methods in our experiments here.
A detailed investigation of scenarios and applications where one of (P1), (P2), or (P3), performs the best (in terms of reconstruction quality compared to the others) is left for future work on specific applications. In this work, we have emphasized more the properties of these formulations, and the novel convergence guarantees of the corresponding algorithms.}.
Our reconstruction method is referred to as Transform Learning MRI (TLMRI).




First, we illustrate the empirical convergence behavior of TLMRI.
We also compare the reconstructions provided by the TLMRI method to those provided by the following schemes: 1) the Sparse MRI method of Lustig et al \cite{lustig}, that utlilizes Wavelets and Total Variation as \emph{fixed} sparsifying transforms; 2) the DLMRI method \cite{bresai} that learns adaptive overcomplete sparsifying dictionaries; 3) the PANO method \cite{Qu2014843} that exploits the non-local similarities between image patches (similar to \cite{dbov}), and employs a 3D transform to sparsify groups of similar patches; and 4) the PBDWS method \cite{Qu12}. The PBDWS method is a very recent \emph{partially} adaptive sparsifying transform based compressed sensing reconstruction method that uses redundant Wavelets and trained patch-based geometric directions. It has been shown to perform better than the earlier PBDW  method \cite{Qu11}.

We simulated the performance of the Sparse MRI, PBDWS, and PANO methods using the software implementations available from the respective authors' websites \cite{lus33, Quweb, PANOweb}.
We used the built-in parameter settings in those implementations, which performed well in our experiments. Specifically, for the PBDWS method, the shift invariant discrete Wavelet transform (SIDWT) based reconstructed image is used as the \emph{guide} (initial) image \cite{Qu12, Quweb}.
We employed the zero-filling reconstruction (produced within the PANO demo code \cite{PANOweb}) as the initial guide image for the PANO method \cite{Qu2014843, PANOweb}.

The implementation of the DLMRI algorithm that solves Problem (P0) is also available online \cite{dlmri1}. For DLMRI, image patches of size $6 \times 6$ ($n=36$) are used, as suggested in \cite{bresai} \footnote{The reconstruction quality improves slightly with a larger patch size, but with a substantial increase in runtime.}, and a four fold overcomplete synthesis dictionary ($K=144$) is learnt using 25 iterations of the algorithm. A patch overlap stride of $r=1$ is used, and $14400$ (found empirically \footnote{Using a larger training size ($> 14400$) during the dictionary learning step of the algorithm provides negligible improvement in final image reconstruction quality, while leading to increased runtimes.}) randomly selected patches are used during the dictionary learning step (executed for 20 iterations) of the DLMRI algorithm.
Mean-subtraction is not performed for the patches prior to the dictionary learning step of DLMRI.
A maximum sparsity level (of $s= 7$ per patch) is employed together with an error threshold (for sparse coding) during the dictionary learning step.
The $\ell_{2}$ error threshold per patch varies linearly from $0.48$ to $0.15$ over the DLMRI iterations.
These parameter settings (all other parameters are set as per the indications in the DLMRI-Lab toolbox \cite{dlmri1}) were observed to work well for the DLMRI algorithm.





The parameters for TLMRI (with Algorithms A1) are set to $n= 36$, $r=1$ (with patch wrap around), $\nu = 3.81$, $\hat{M}=1$, $\lambda_{0}=0.2$, and $C=10^{5}$.
The sparsity level $s = 0.055 \times n N$ (this corresponds to an average sparsity level per patch of $0.055 \times n$, or $5.5\%$ sparsity) \footnote{The sparsity level $s$ is a regularization parameter in our framework that provides a trade-off between how much aliasing is removed over the algorithm iterations, and how much image information is kept or restored (i.e., not eliminated by the sparsity condition).
We determined the sparsity level empirically in the experiments in this work.}
, where $N= 512^{2}$,  is used in our experiments except in Section \ref{convlearn}, where $s = 0.045 \times n N$ is used.
The initial transform estimate $W^{0}$ is the (simple) patch-based 2D DCT \cite{sabres}, and the initial image $x^{0}$ is set to be the standard zero-filling Fourier reconstruction
\footnote{While we used the naive zero-filling Fourier reconstruction in our experiments here for simplicity, one could also use other better initializations for $x$ such as the SIDWT based reconstructed image \cite{Qu12}, or the reconstructions produced by recent methods (e.g., PBDWS, etc.). We have observed empirically that better initializations may lead to faster convergence of TLMRI, and TLMRI typically only improves the image quality compared to the initializations (assuming properly chosen sparsity levels).}.
The initial sparse code settings are the solution to \eqref{bcs1}, for the given $(W^{0}, x^{0})$.
Our TLMRI implementation was coded in Matlab version R2013a. Note that this implementation has not been optimized for efficiency.
A link to the Matlab implementation is provided at \url{http://www.ifp.illinois.edu/~yoram/}.
All simulations in this work were executed in Matlab. All computations were performed with an Intel Core i5 CPU at 2.5GHz and 4GB memory, employing a 64-bit Windows 7 operating system.

Similar to prior work \cite{bresai}, we quantify the quality of MR image reconstruction using the peak-signal-to-noise ratio (PSNR), and high frequency error norm (HFEN) metrics. 
The PSNR (expressed in decibels (dB)) is computed as the ratio of the peak intensity value of some reference image to the root mean square reconstruction error (computed between image magnitudes) relative to the reference.
In MRI, the reference image is typically the image reconstructed from fully sampled k-space data. The HFEN metric quantifies the quality of reconstruction of edges or finer features. A rotationally symmetric Laplacian of Gaussian (LoG) filter is used, whose kernel is of size $15 \times 15  $ pixels, and with a standard deviation of 1.5 pixels \cite{bresai}. HFEN is computed as the $\ell_{2}$ norm of the difference between the LoG filtered reconstructed and reference magnitude images.


\subsection{Convergence and Learning Behavior} \label{convlearn}

\begin{figure}[!t]
\begin{center}
\begin{tabular}{cccc}
\includegraphics[height=1.28in]{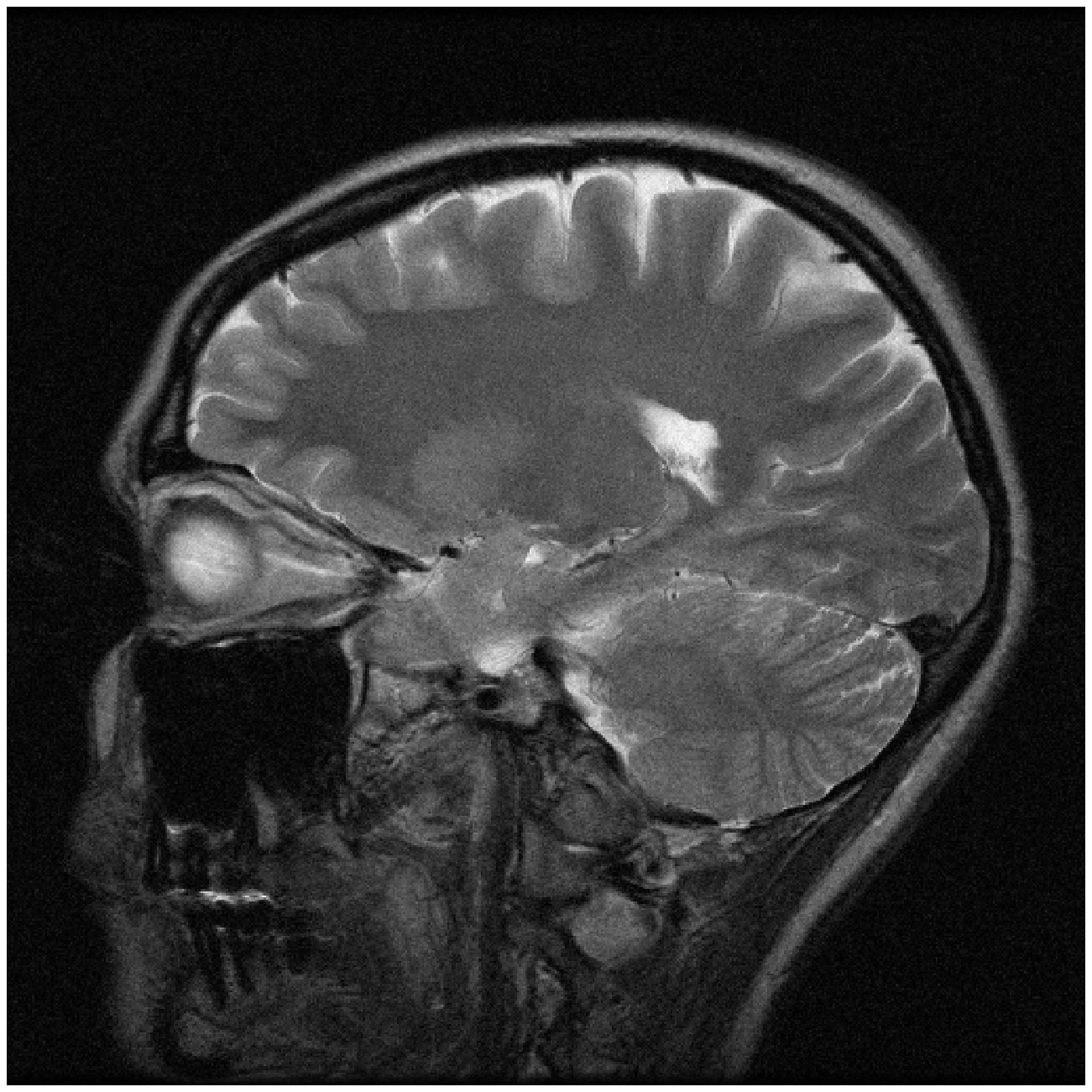}&
\includegraphics[height=1.28in]{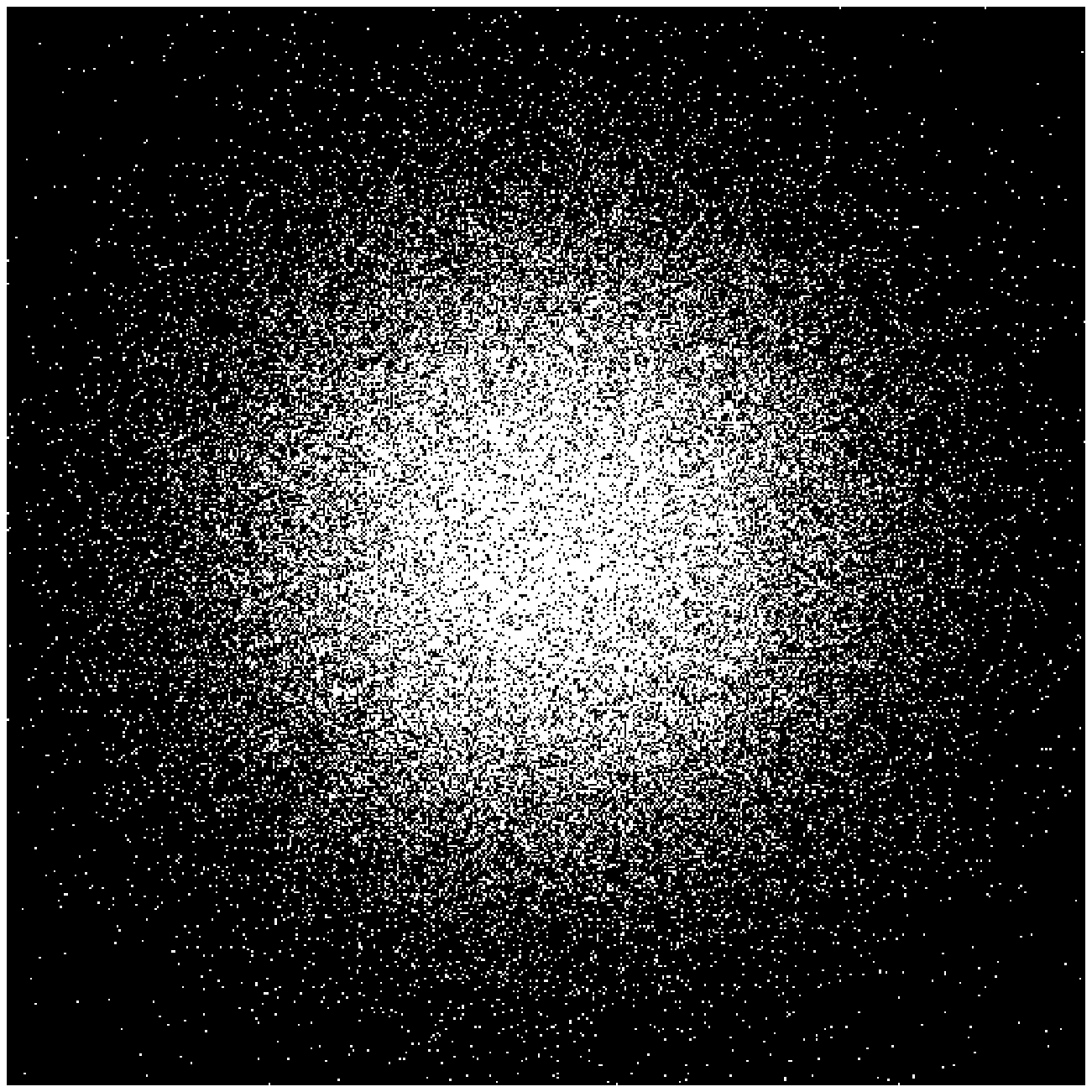}&
\includegraphics[height=1.28in]{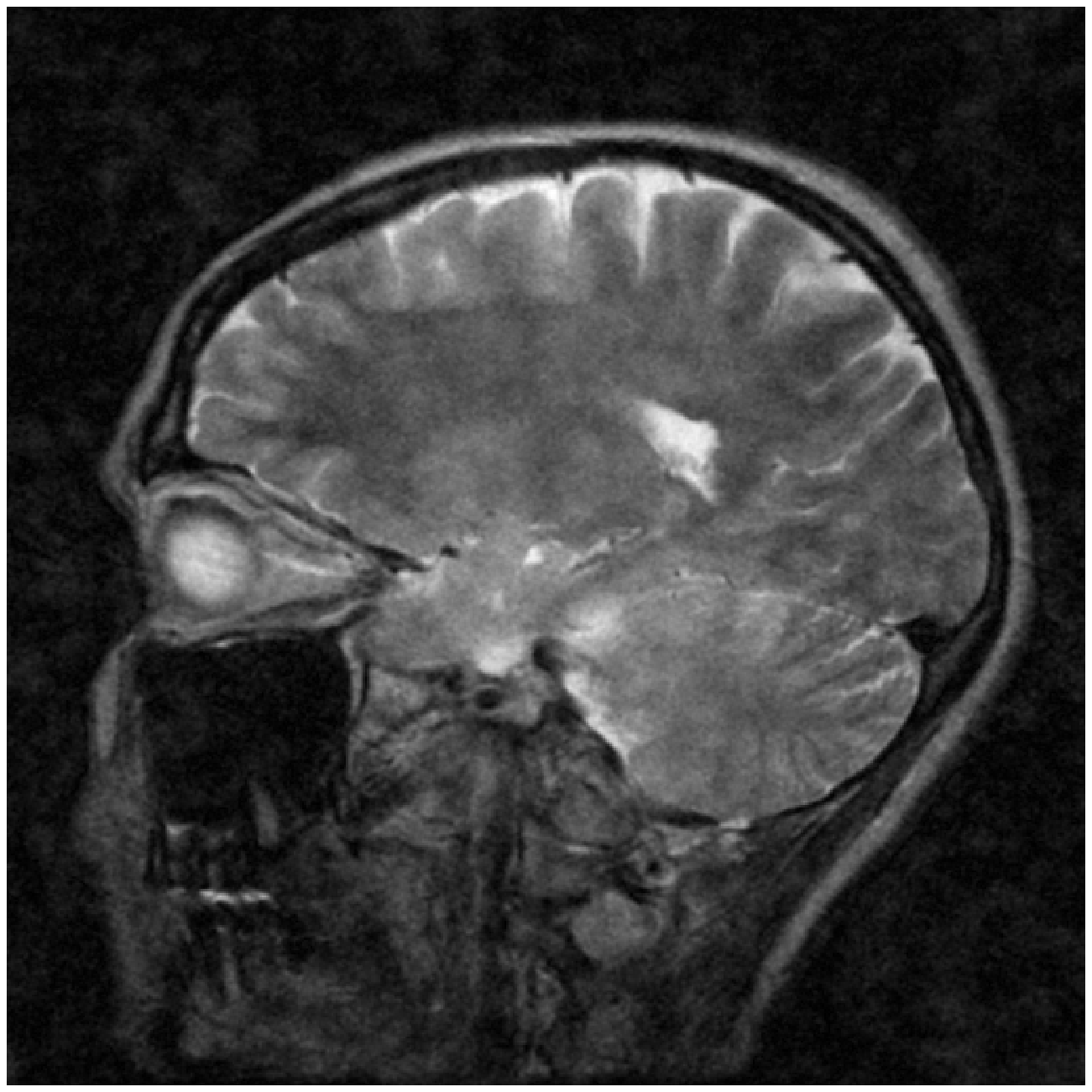}&
\includegraphics[height=1.28in]{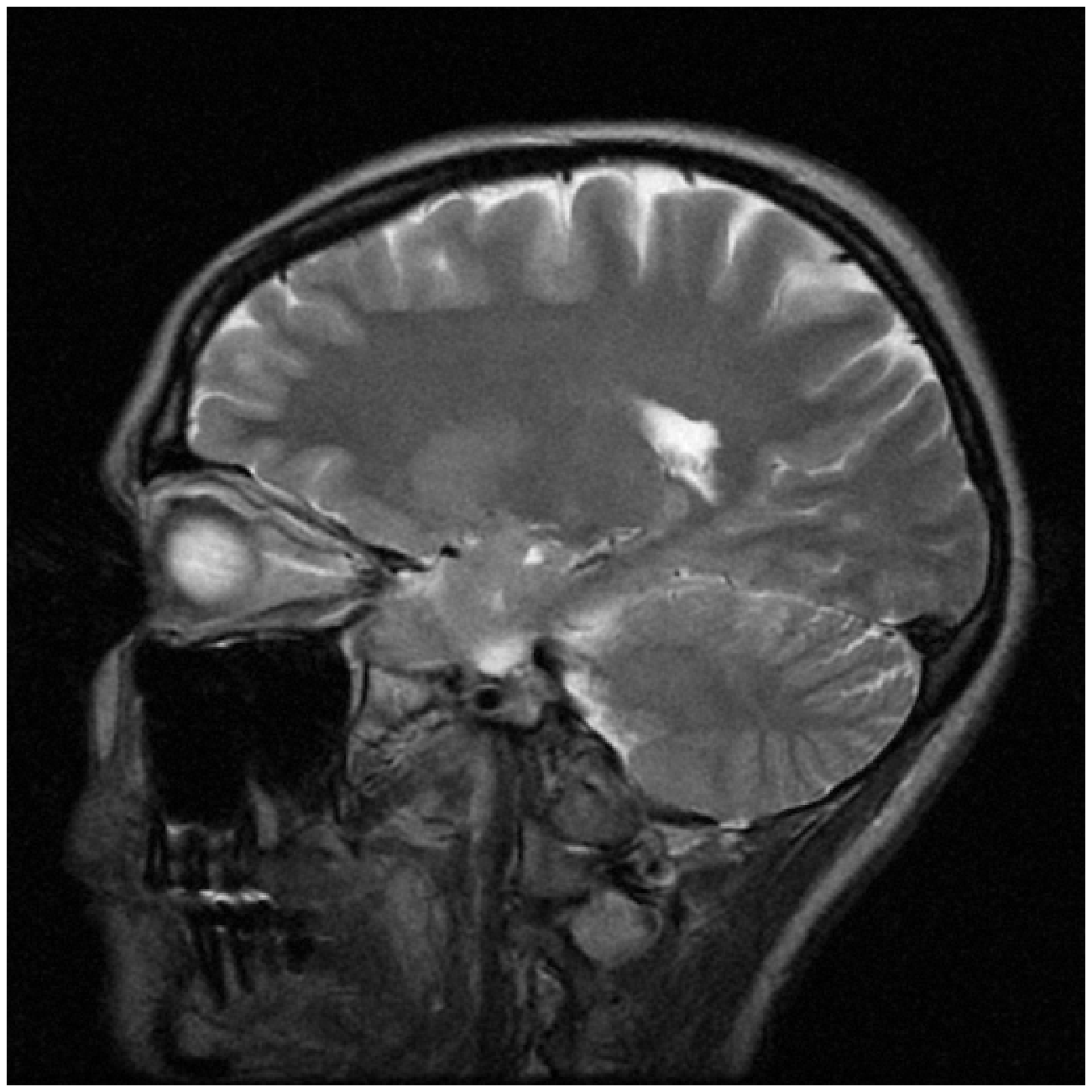}\\
(a) & (b) & (c) & (d)\\
\end{tabular}
\begin{tabular}{ccc}
\includegraphics[height=1.2in]{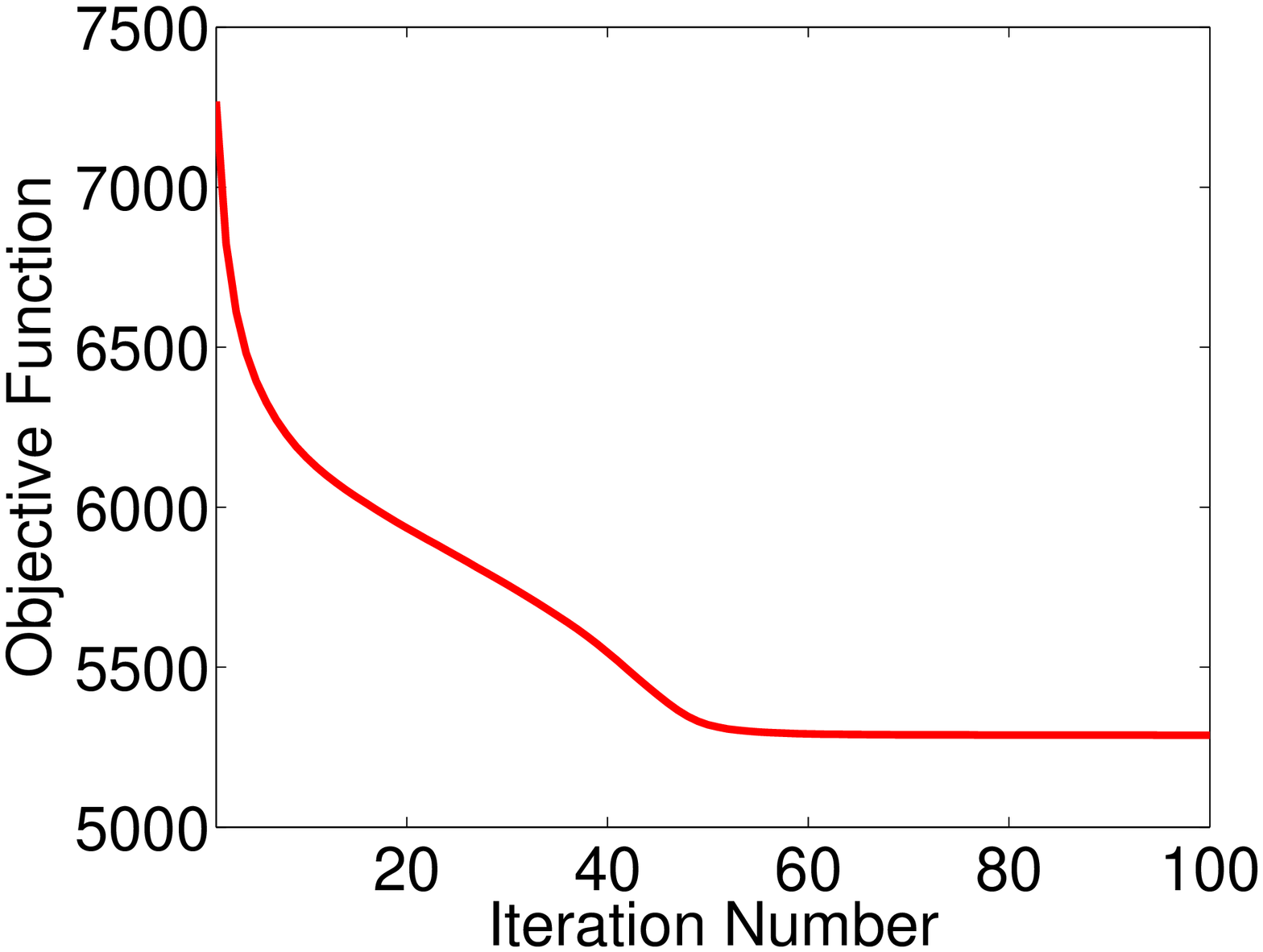}&
\includegraphics[height=1.2in]{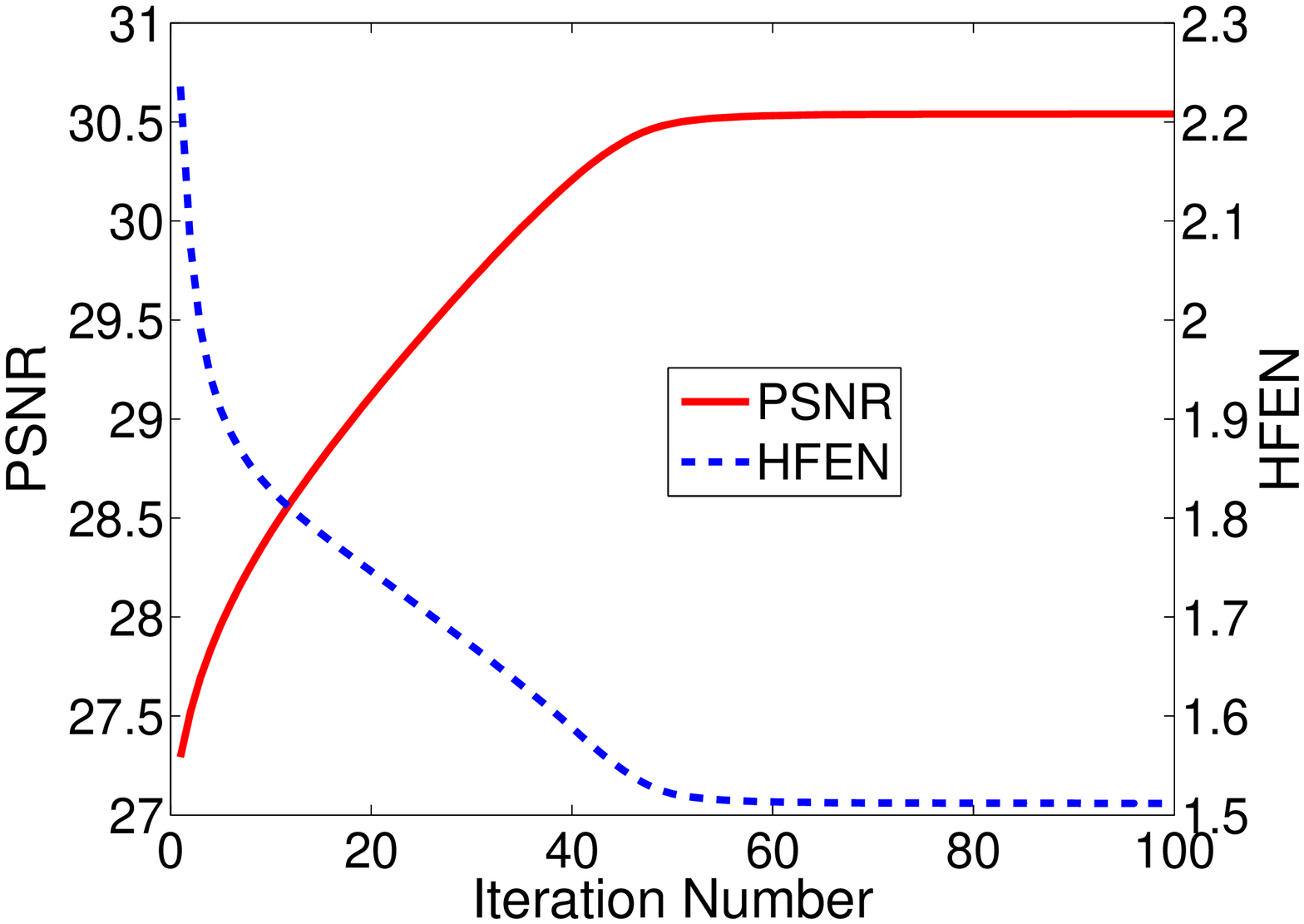}&
\includegraphics[height=1.2in]{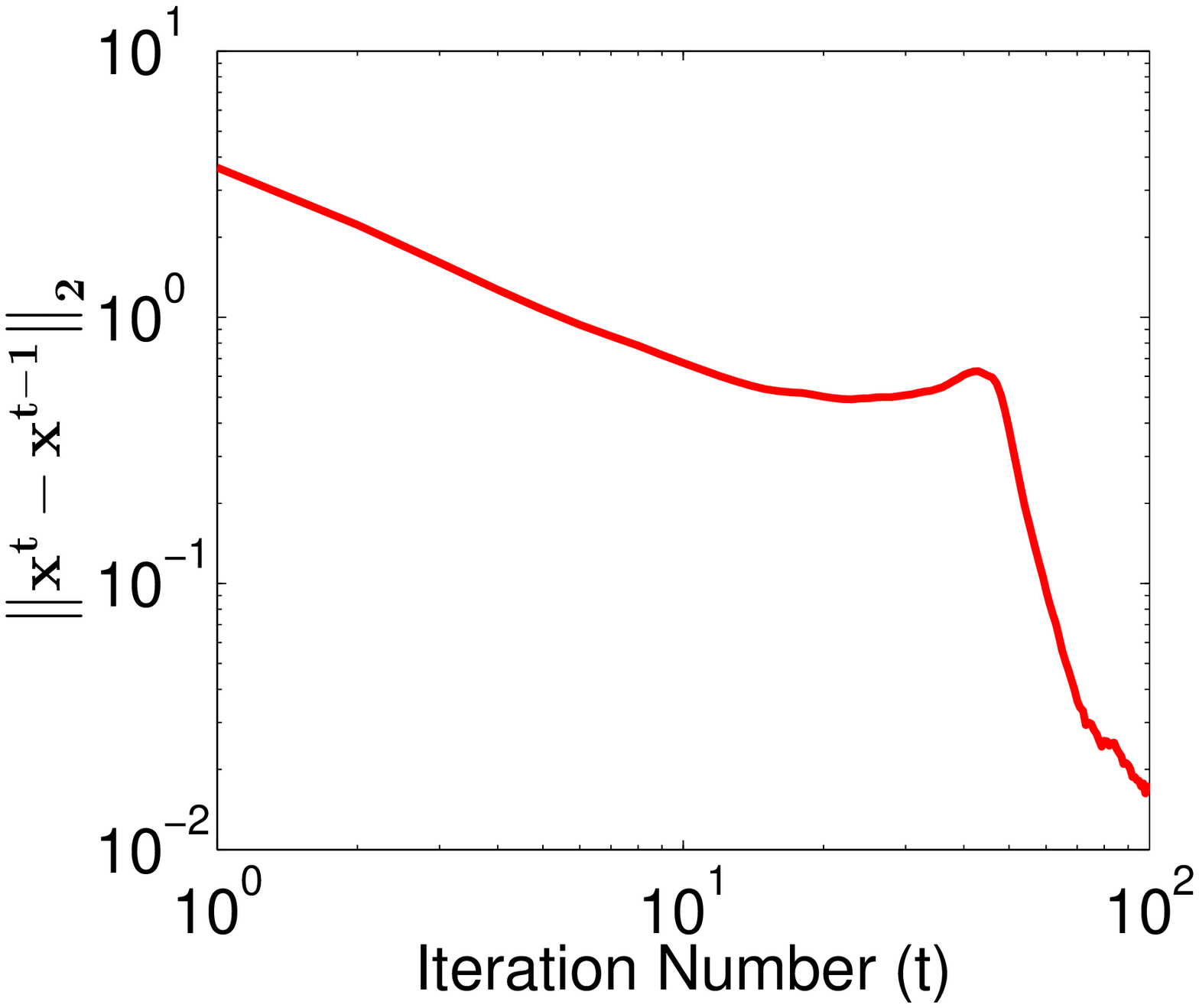}\\
(e) & (f) & (g)\\
\end{tabular}
\begin{tabular}{c}
\includegraphics[height=1.6in]{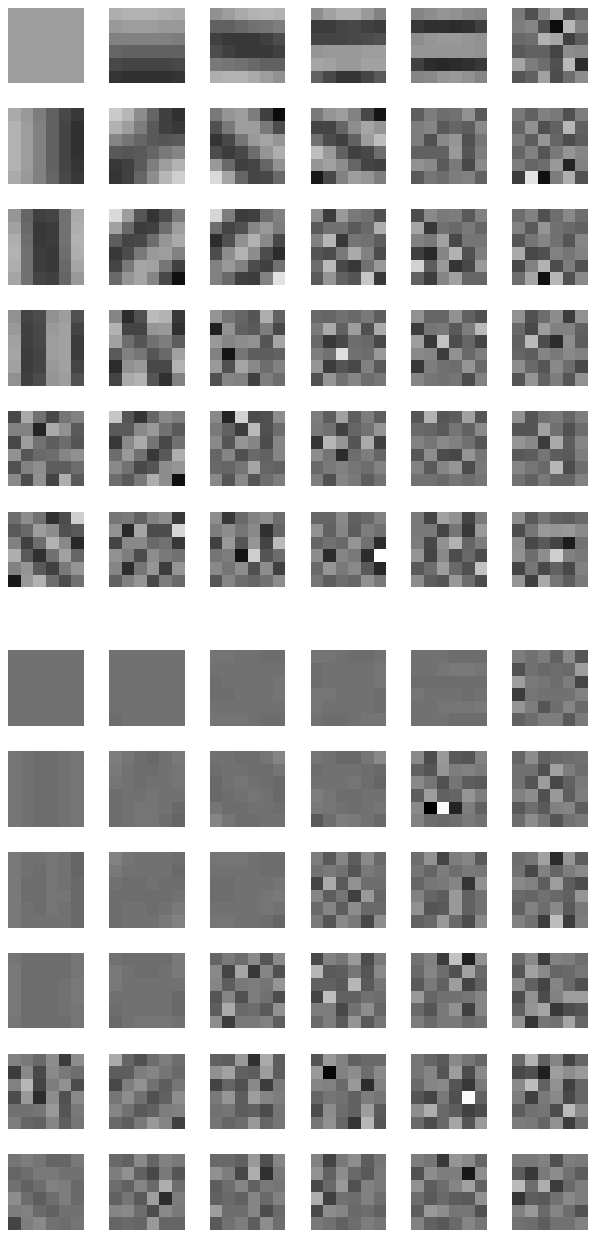}\\
(h)\\
\end{tabular}
\caption{Convergence of TLMRI with 5x undersampling: (a) Reference image; (b) sampling mask in k-space; (c) initial zero-filling reconstruction ($26.93$ dB); (d) TLMRI reconstruction ($30.54$ dB); (e) objective function (since the regularizer $Q(W) \geq n/2$ \cite{sabres}, we have subtracted out the constant offset of $n \lambda/2$ from the objective values here); (f) PSNR and HFEN; (g) changes between successive iterates ($\left \| x^{t}-x^{t-1} \right \|_{2}$), and (h) real (top) and imaginary (bottom) parts of the learnt $W$, with the matrix rows shown as patches.}
\label{im1bcs}
\end{center}
\end{figure}


In this experiment, we consider the reference image 
in Fig. \ref{im1bcs}(a).
We perform four fold undersampling of the k-space space of the (peak normalized \footnote{In practice, the data or k-space measurements can always be normalized prior to processing them for image reconstruction. Otherwise, the parameter settings for algorithms may need to be modified to account for data scaling.}) reference.
The (variable density \cite{bresai}) sampling mask is shown in Fig. \ref{im1bcs}(b). When the TLMRI algorithm is executed using the undersampled data, the objective function converges monotonically and quickly over the iterations as shown in Fig. \ref{im1bcs}(e). The changes between successive iterates $\left \| x^{t}-x^{t-1} \right \|_{2}$ (Fig. \ref{im1bcs}(g)) converge towards $0$. Such convergence was established by Theorem \ref{theorem1bc}, and is indicative (a necessary but not suffficient condition) of convergence of the entire sequence $\left \{ x^{t} \right \}$. As far as the performance metrics are concerned, the PSNR metric (Fig. \ref{im1bcs}(f)) increases over the iterations, and the HFEN metric decreases, indicating improving reconstruction quality over the algorithm iterations. These metrics also converge quickly.

The initial zero-filling reconstruction (Fig. \ref{im1bcs}(c)) shows aliasing artifacts that are typical in the undersampled measurement scenario, and has a PSNR of only $26.93$ dB. On the other hand, the final TLMRI reconstruction (Fig. \ref{im1bcs}(d)) is much enhanced (by $3.6$ dB), with a PSNR of $30.54$ dB. Since Algorithm A1 is guaranteed to converge to the set of critical points of Problem (P1), the result in Fig. \ref{im1bcs}(d) suggests that, in practice, the set of critical points may in fact include images that are close to the true image. Note that our identifiability result (Proposition \ref{optimalmodelsbcs}) in Section \ref{dfgy6} ensured global optimality of the underlying image only in a noiseless (or error-free) scenario. The learnt transform $W$ ($\kappa(W)= 1.01$) for this example is shown in Fig. \ref{im1bcs}(h). This is a complex valued transform. Both the real and imaginary parts of $W$ display texture or frequency like structures, that sparsify the patches of the MR image. Our algorithm is thus able to learn this structure and reconstruct the image using only the undersampled measurements. 




\subsection{Comparison to Other Methods}  \label{comparisons}

In the following experiments, we execute the TLMRI algorithm for 40 iterations (all other parameters are set to the values mentioned in Section \ref{sec5framwork}). We also use a lower sparsity level ($<0.055 \times n N$) during the initial several algorithm iterations, which leads to faster convergence. We consider the complex-valued reference images in Fig. \ref{im1bcs}(a) and Fig. \ref{im2bcs}(a) that are labeled as Image 1 and Image 2 respectively, and simulate variable density 2D random or Cartesian undersampling \cite{bresai} of the k-spaces of these images.
Table \ref{tab1bcs} lists the reconstruction PSNRs corresponding to the zero-filling, Sparse MRI, PBDWS, PANO, DLMRI, and TLMRI\footnote{We observed that in Table \ref{tab1bcs}, if we sampled vertical (instead of horizontal) lines for the Cartesian sampling patterns (by transposing the Cartesian sampling masks used in Table \ref{tab1bcs}), the reconstructed images for TLMRI had PSNRs (32.52 dB and 31.05 dB respectively, for the 4x and 7x undersampling cases) similar to the horizontal lines sampling case. However, the learnt sparsifying transforms for the two cases had many dissimilar rows. This is not surprising since prior work on transform learning \cite{sbclsTS2} has empirically shown that even the patches of a single image may admit multiple equally good sparsifying transforms, that may not be related by only row permutations or sign changes.} reconstructions for various cases.

\begin{table}[t]
\centering
\fontsize{7}{10pt}\selectfont
\begin{tabular}{|c|c|c|c|c|c|c|c|c|}
\hline
Image & Sampling Scheme & Undersampling & Zero-filling & Sparse MRI & PBDWS & PANO & DLMRI & TLMRI \\ 
\hline
2  & 2D Random     &     4x            &        25.3                  &    26.13   &    31.69      &  32.80     &   33.01         & \textbf{33.12} \\
 \hline
1 &  2D Random      &    5x             &        26.9                  &   27.84    &    30.27      &  30.37     &   30.49         & \textbf{30.56} \\
 \hline
2 &  2D Random       &   7x             &       25.3                   &   26.38    &    31.10       &  30.92     &   31.70         & \textbf{31.94} \\
 \hline
2 &   Cartesian     &       4x              &      28.9                    &   29.73    &    31.67      &   32.24      &  32.67         & \textbf{32.78} \\
 \hline
2 &    Cartesian    &       7x               &      27.9                    &   28.58    &    31.11      &   31.08       &  30.91         & \textbf{31.24} \\
 \hline
\end{tabular}
\caption{PSNRs corresponding to the Zero-filling, Sparse MRI \cite{lustig}, PBDWS \cite{Qu12}, PANO \cite{Qu2014843}, DLMRI \cite{bresai}, and TLMRI reconstructions, for various sampling schemes and undersampling factors. The best PSNRs are marked in bold.}
\label{tab1bcs}
\end{table}

The TLMRI algorithm is seen to provide the best PSNRs (analogous results were observed to typically hold with respect to the HFEN metric not shown in the table) for the various scenarios in Table \ref{tab1bcs}. Significant improvements (up to 7 dB) are observed over the Sparse MRI method, that uses fixed sparsifying transforms.
Moreover, TLMRI provides up to 1.4 dB improvement in PSNR over the recent (partially adaptive) PBDWS method, and up to  1 dB improvement over the recent non-local patch similarity-based PANO method.
Finally, the TLMRI reconstruction quality is somewhat (up to  0.33 dB) better than DLMRI. This is despite the latter using a 4 fold overcomplete (i.e., larger or richer) dictionary. 

Fig. \ref{im2bcs} shows the TLMRI reconstruction (Fig. \ref{im2bcs}(d)) of Image 2 for the case of 2D random sampling (sampling mask shown in Fig. \ref{im2bcs}(c)) and seven fold undersampling. The reconstruction errors (i.e., the magnitude of the difference between the magnitudes of the reconstructed and reference images) for several schemes are shown in Figs. \ref{im2bcs} (e)-(h). The error map for TLMRI clearly shows the smallest image distortions.
Fig. \ref{im3bcs} shows the TLMRI (Fig. \ref{im3bcs}(c)) and DLMRI (Fig. \ref{im3bcs}(e)) reconstructions and reconstruction error maps (Figs. \ref{im3bcs} (d), (f)) for Image 2 with Cartesian sampling (sampling mask shown in Fig. \ref{im3bcs}(b)) and seven fold undersampling. TLMRI provides a better reconstruction of image edges and better aliasing removal than DLMRI in this case.


\begin{figure}[!t]
\begin{center}
\begin{tabular}{cccc}
\includegraphics[height=1.2in]{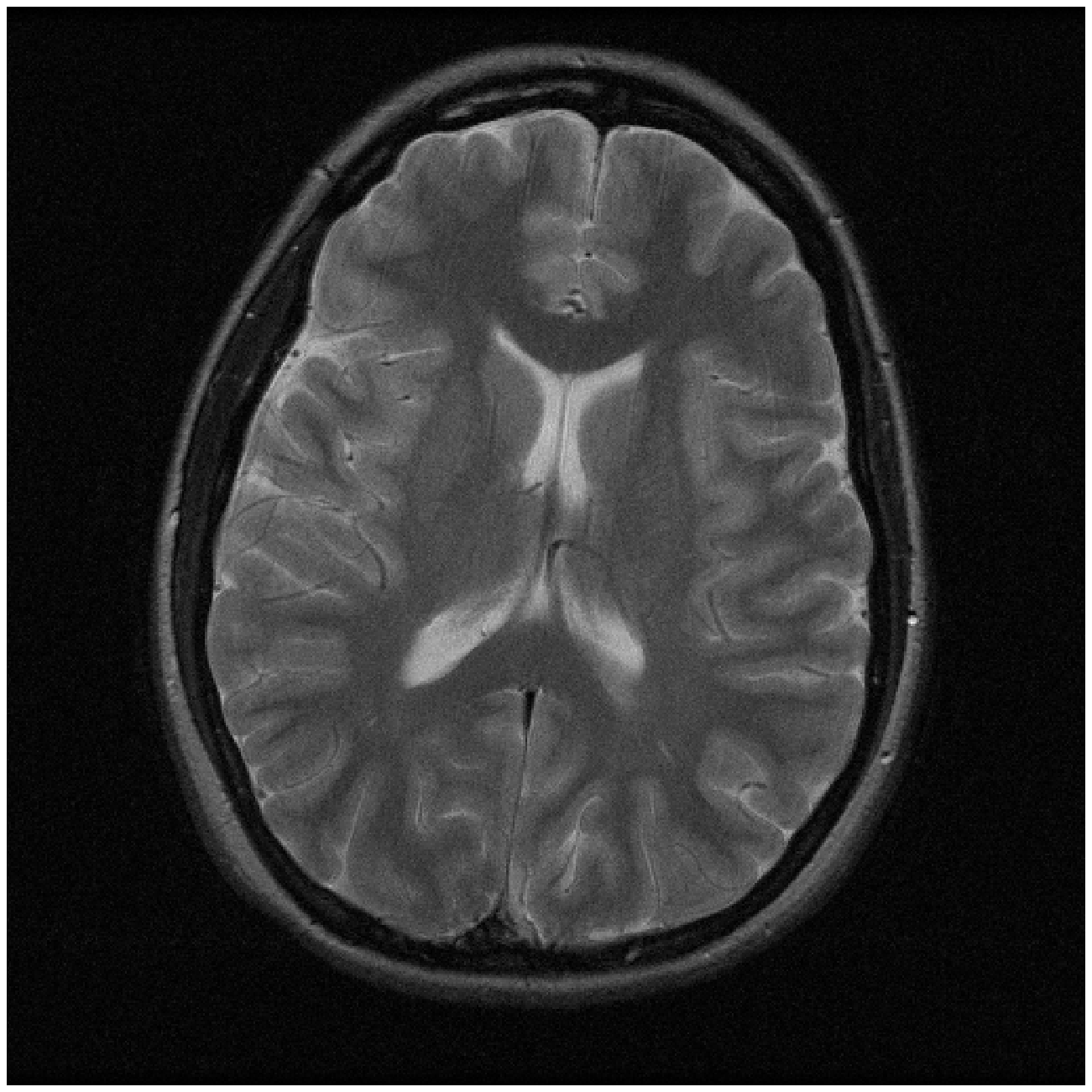}&
\includegraphics[height=1.2in]{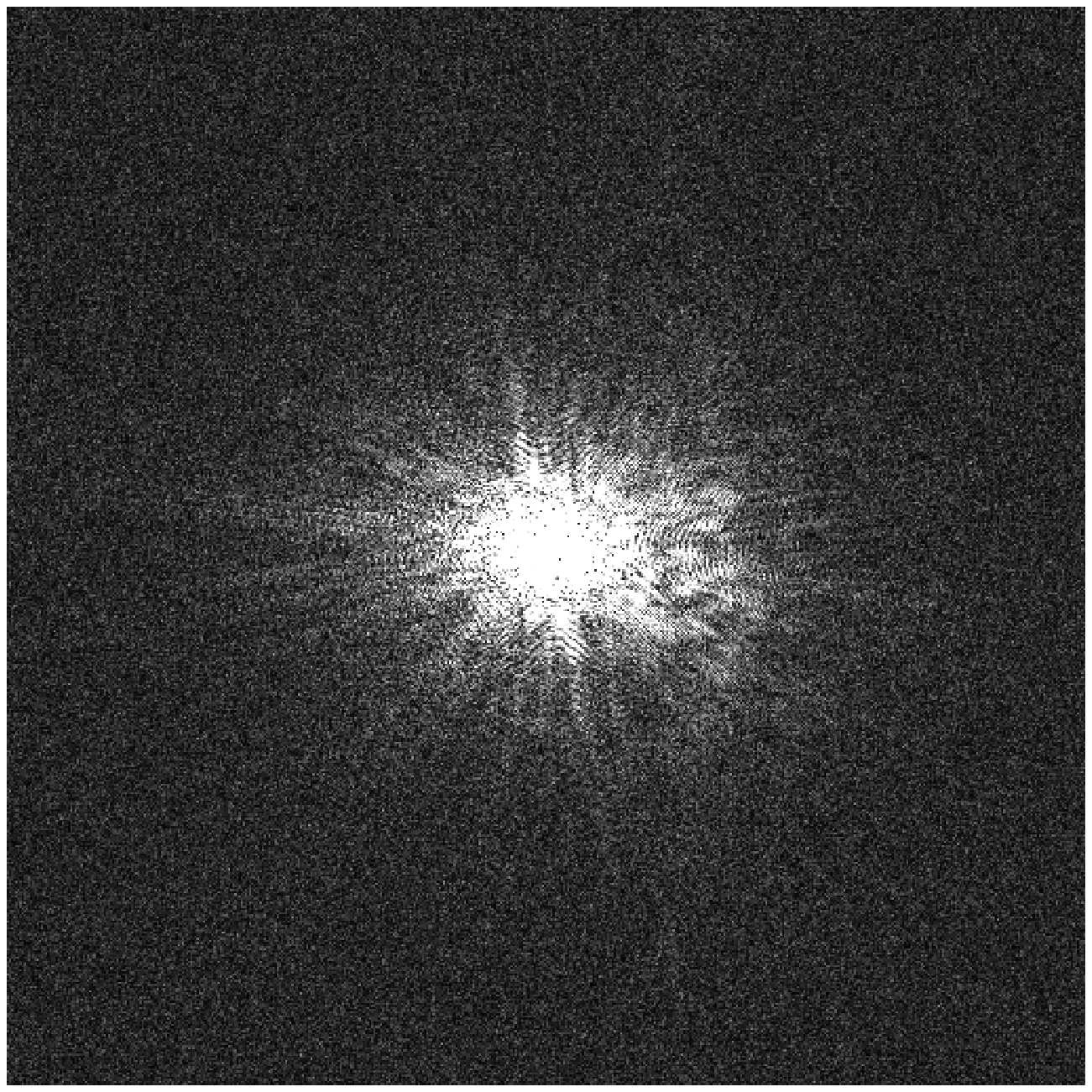}&
\includegraphics[height=1.2in]{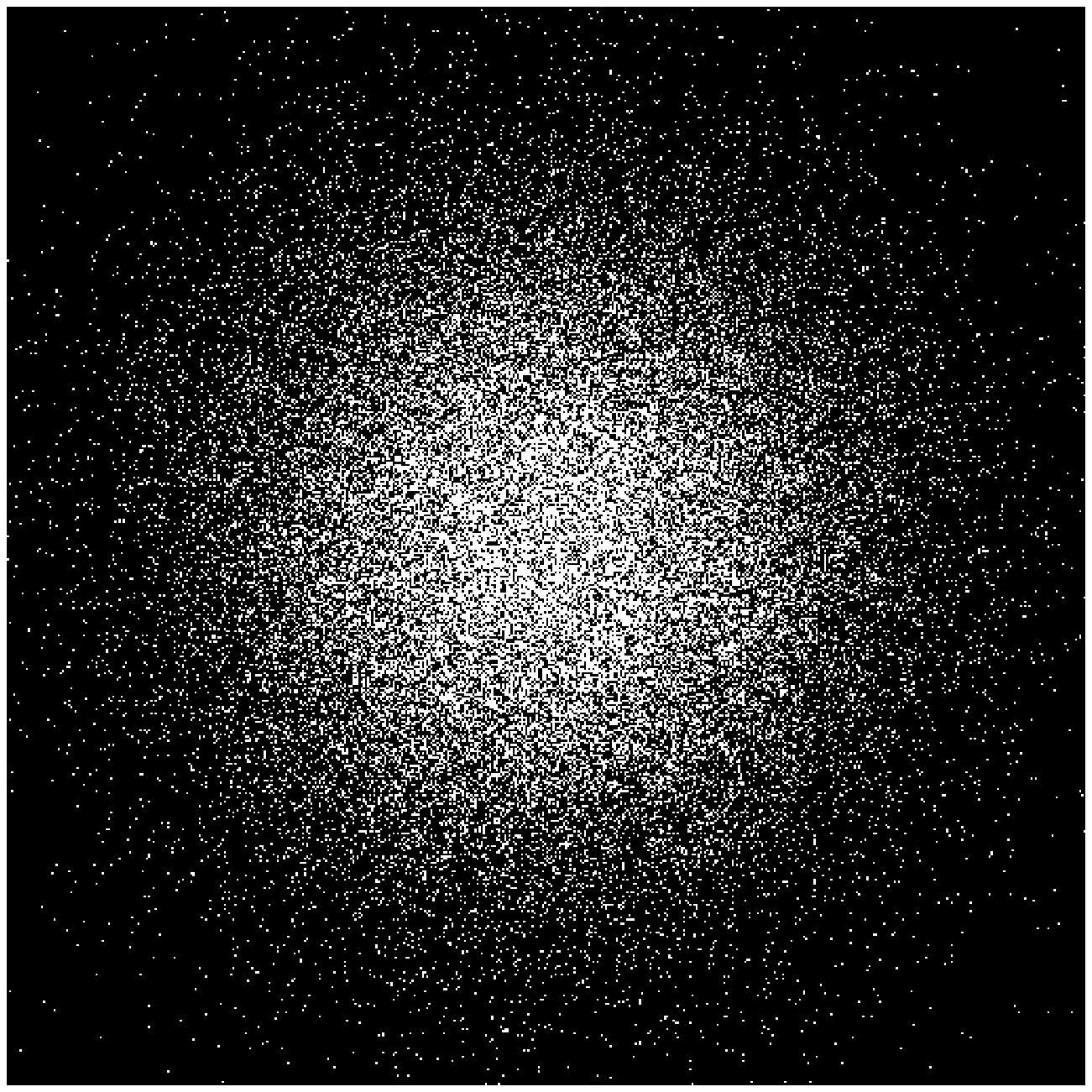}&
\includegraphics[height=1.2in]{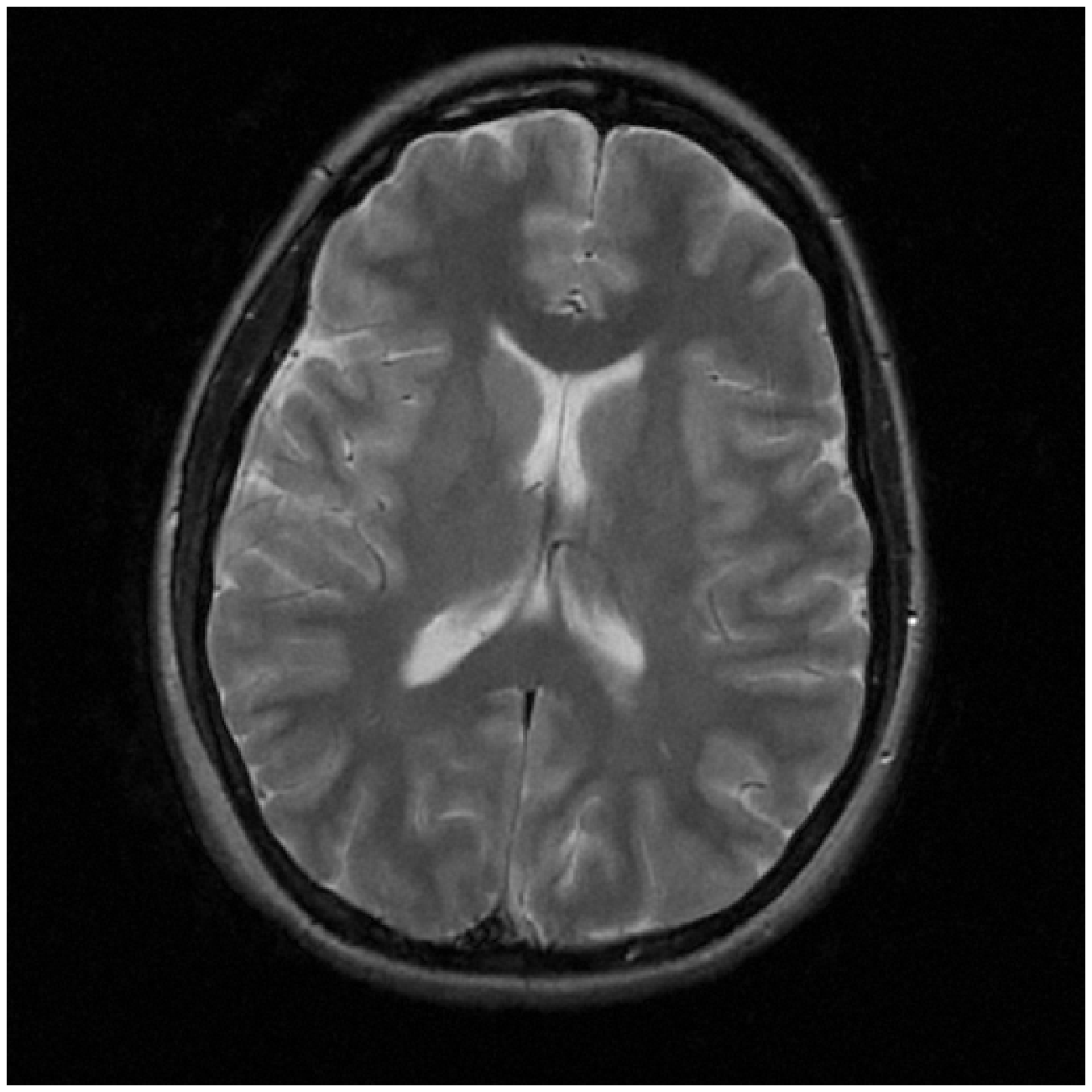} \\
(a) & (b) & (c) & (d)\\
\includegraphics[height=1.2in]{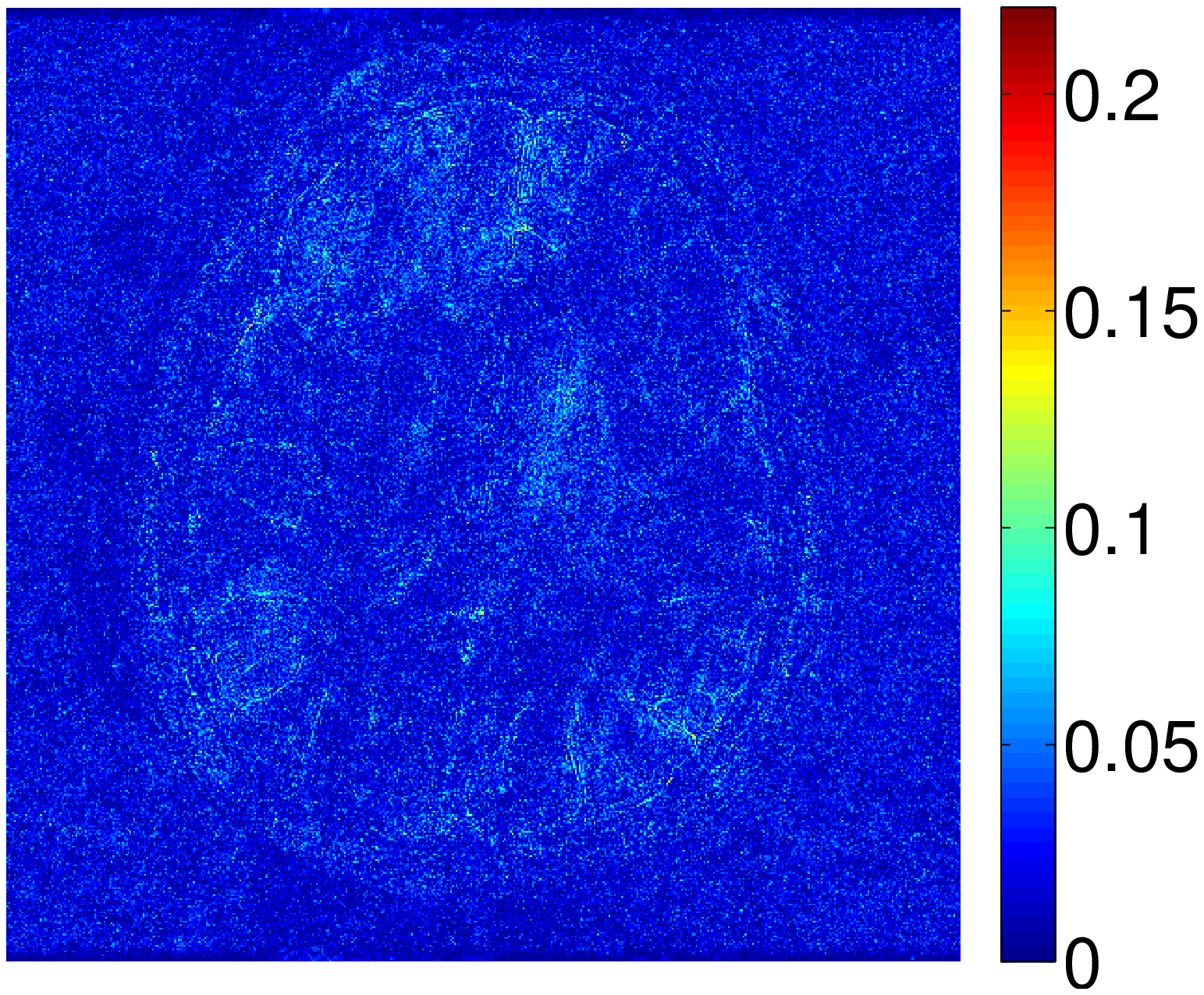}&
\includegraphics[height=1.2in]{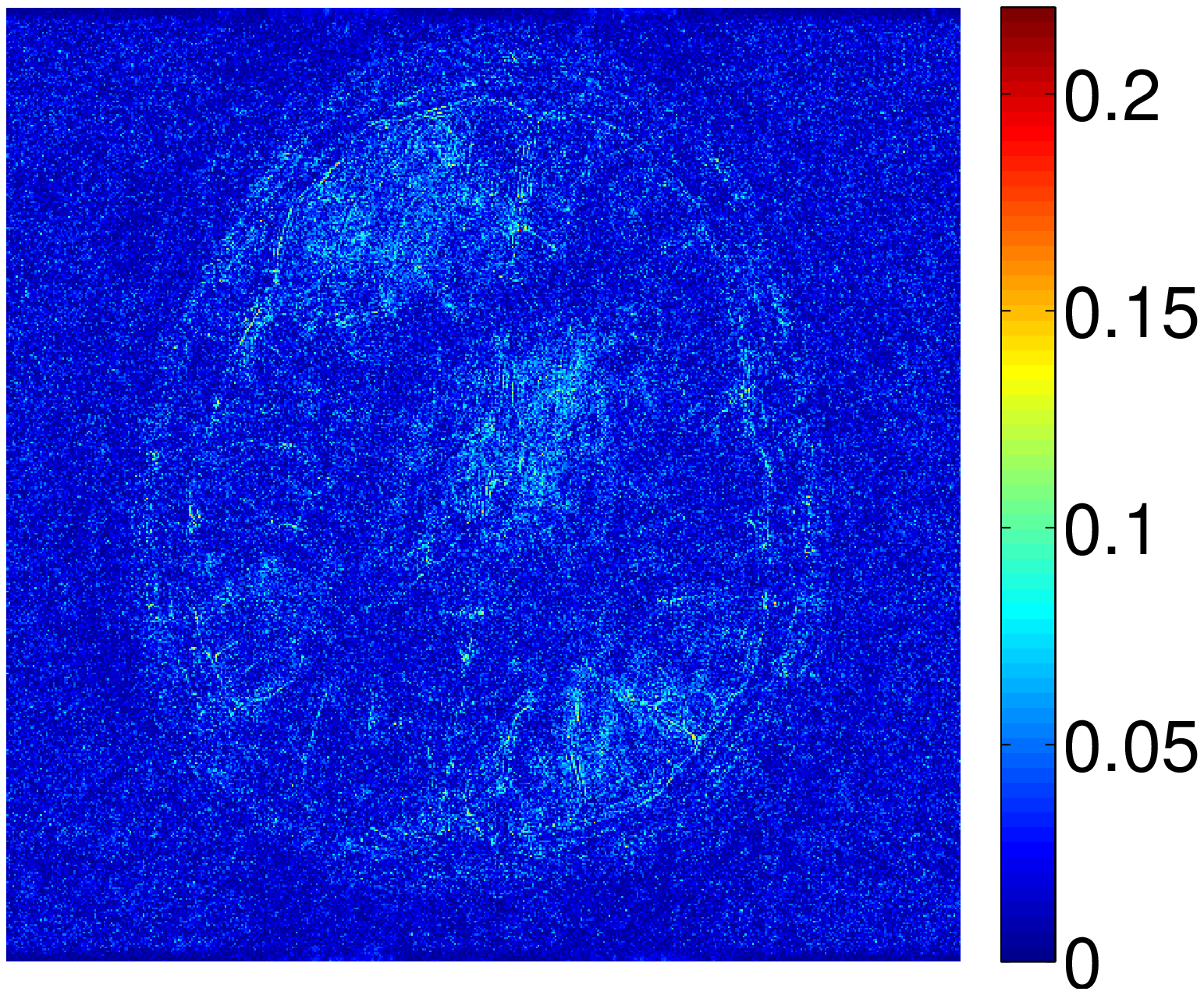}&
\includegraphics[height=1.2in]{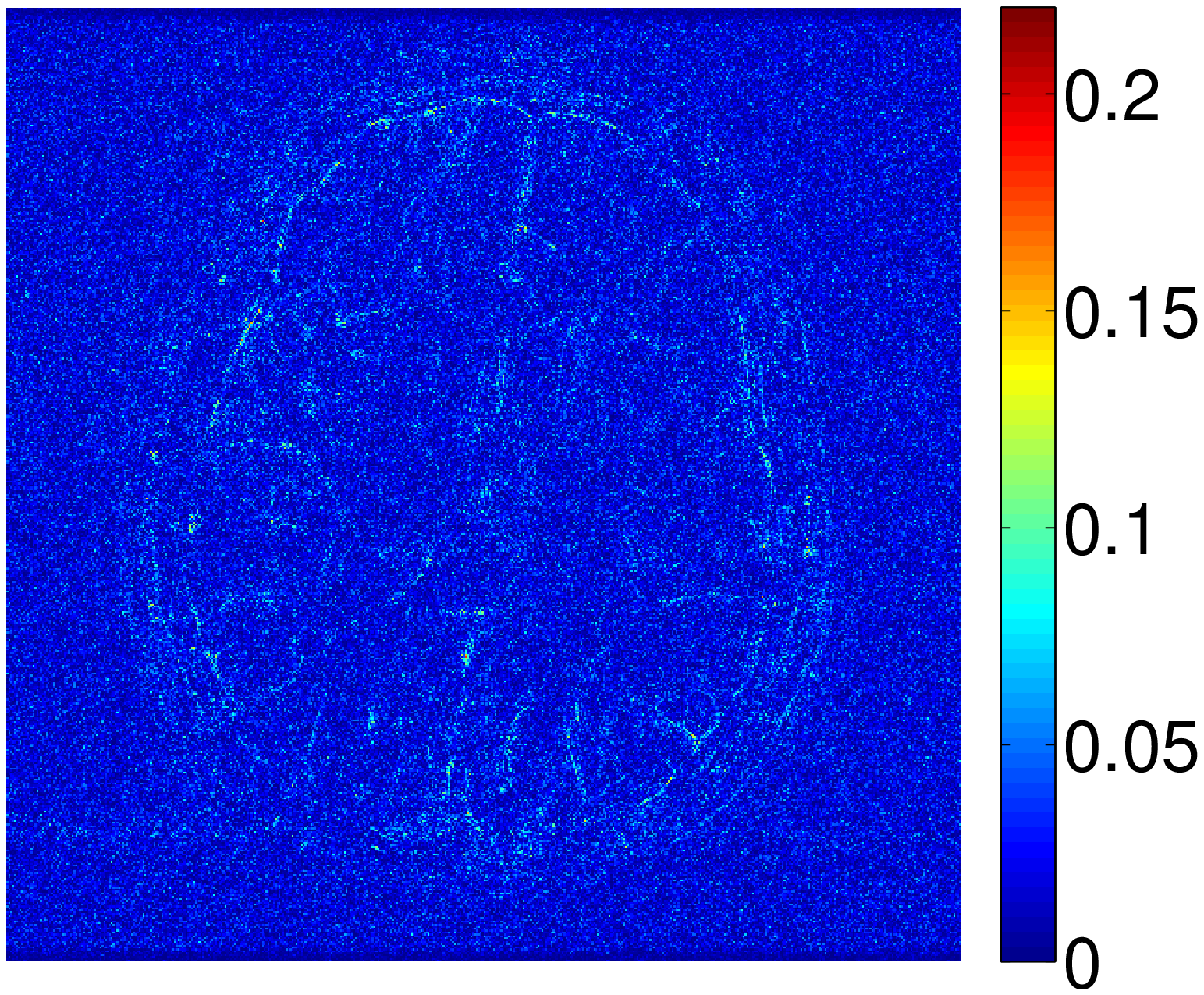}&
\includegraphics[height=1.2in]{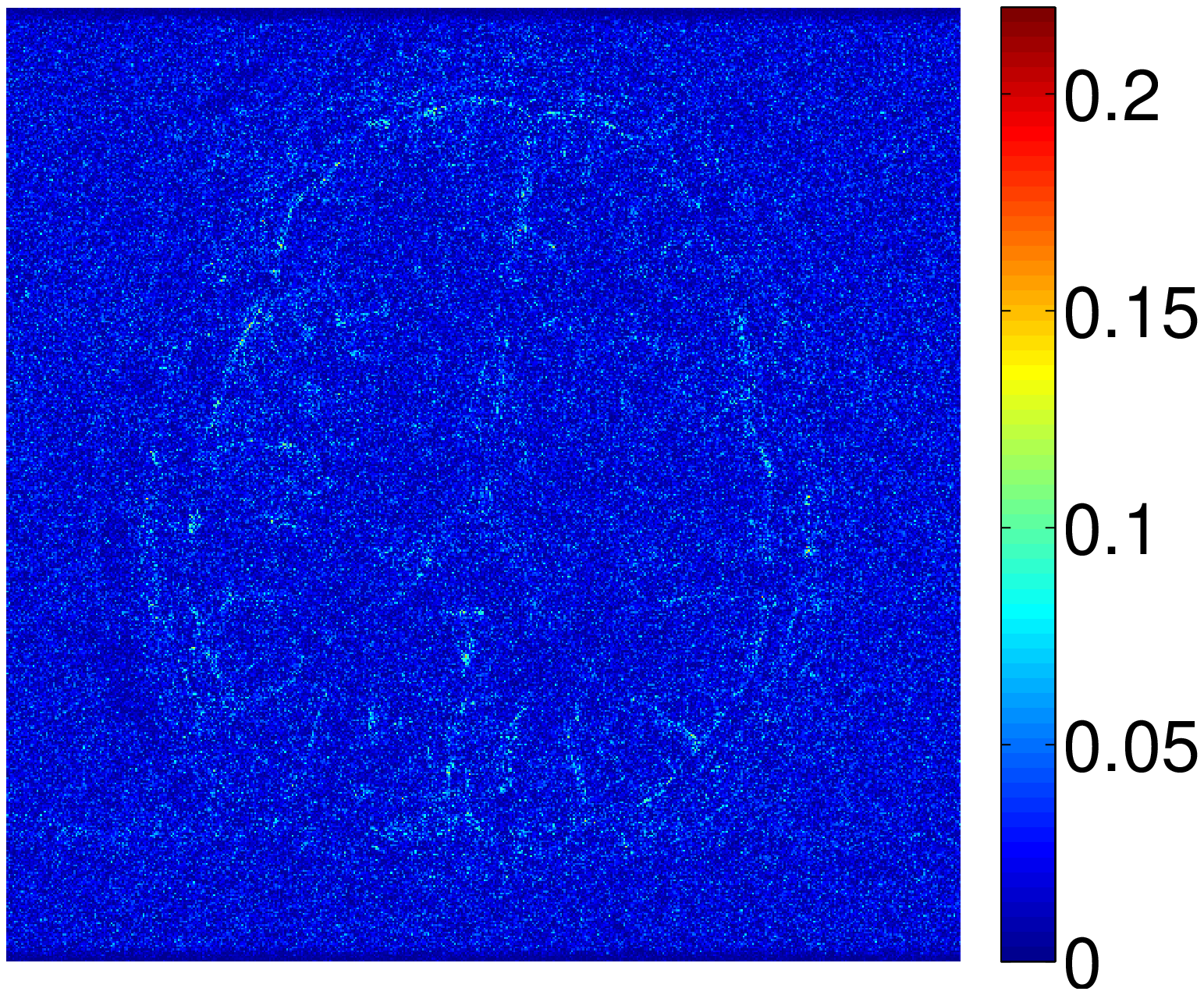} \\
(e) & (f) & (g) & (h)\\
\end{tabular}
\caption{Results for 2D random sampling and 7x undersampling: (a) Reference image; (b) k-space of reference; (c) sampling mask in k-space; (d) TLMRI reconstruction (31.94 dB); (e) magnitude of PBDWS \cite{Qu12} reconstruction error; (f) magnitude of PANO \cite{Qu2014843} reconstruction error; (g) magnitude of DLMRI \cite{bresai} reconstruction error; and (h) magnitude of TLMRI reconstruction error.}
\label{im2bcs}
\end{center}
\end{figure}


The average run times of the Sparse MRI, PBDWS, PANO, DLMRI, and TLMRI algorithms in Table \ref{tab1bcs} are 251 seconds, 797 seconds, 400 seconds, 3273 seconds, and 243 seconds, respectively. 
The PBDWS run time includes the time taken for computing the initial SIDWT based reconstruction or guide image \cite{Qu12} in the PBDWS software package \cite{Quweb}. 
The TLMRI algorithm is thus the fastest one in Table \ref{tab1bcs}, and provides a speedup of about 13.5x over the synthesis dictionary-based DLMRI, and a speedup of about 3.3x and 1.6x over the PBDWS and PANO
\footnote{Another faster version of the PANO method is also publicly available \cite{PANOwebFast}. However, we found that although this version has an average run time of only 40 seconds in Table \ref{tab1bcs}, it also provides worse reconstruction PSNRs than \cite{PANOweb} in Table \ref{tab1bcs}.} methods, respectively. Note that the speedups for TLMRI over PBDWS or PANO were obtained by comparing our unoptimized Matlab implementation of TLMRI against the MEX (or C) implementations of PBDWS and PANO.




\begin{figure}[!t]
\begin{center}
\begin{tabular}{ccc}
\includegraphics[height=1.4in]{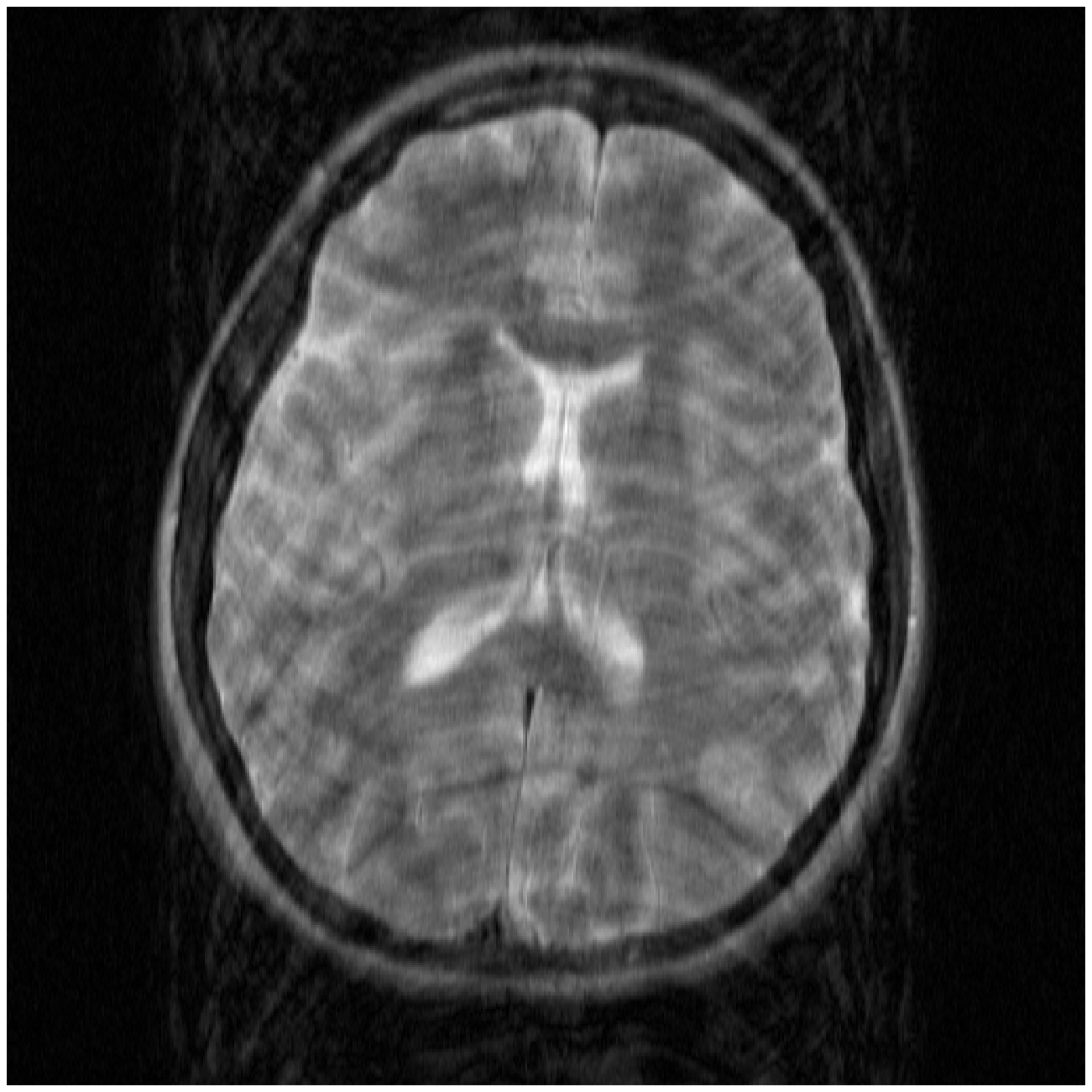}&
\includegraphics[height=1.4in]{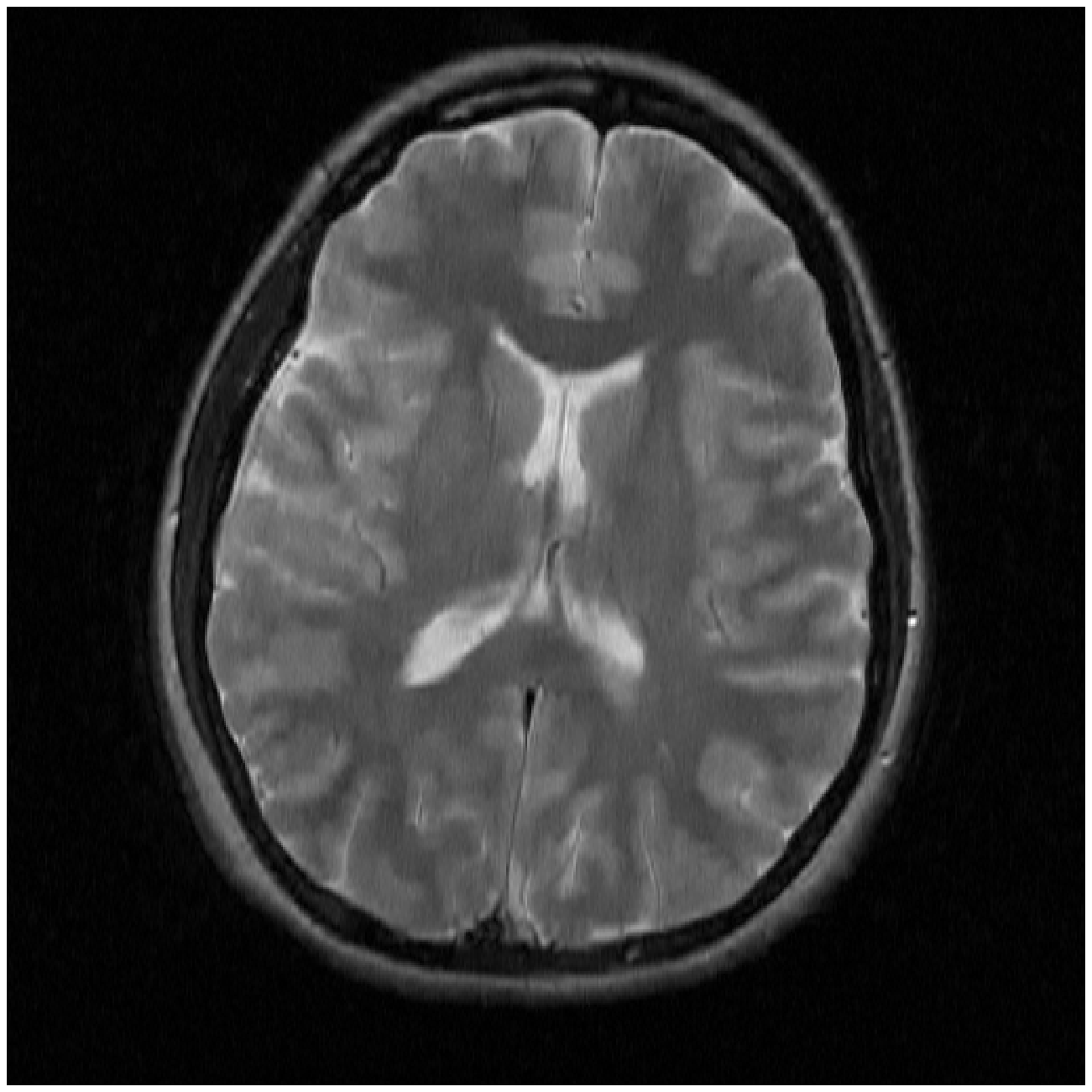}&
\includegraphics[height=1.4in]{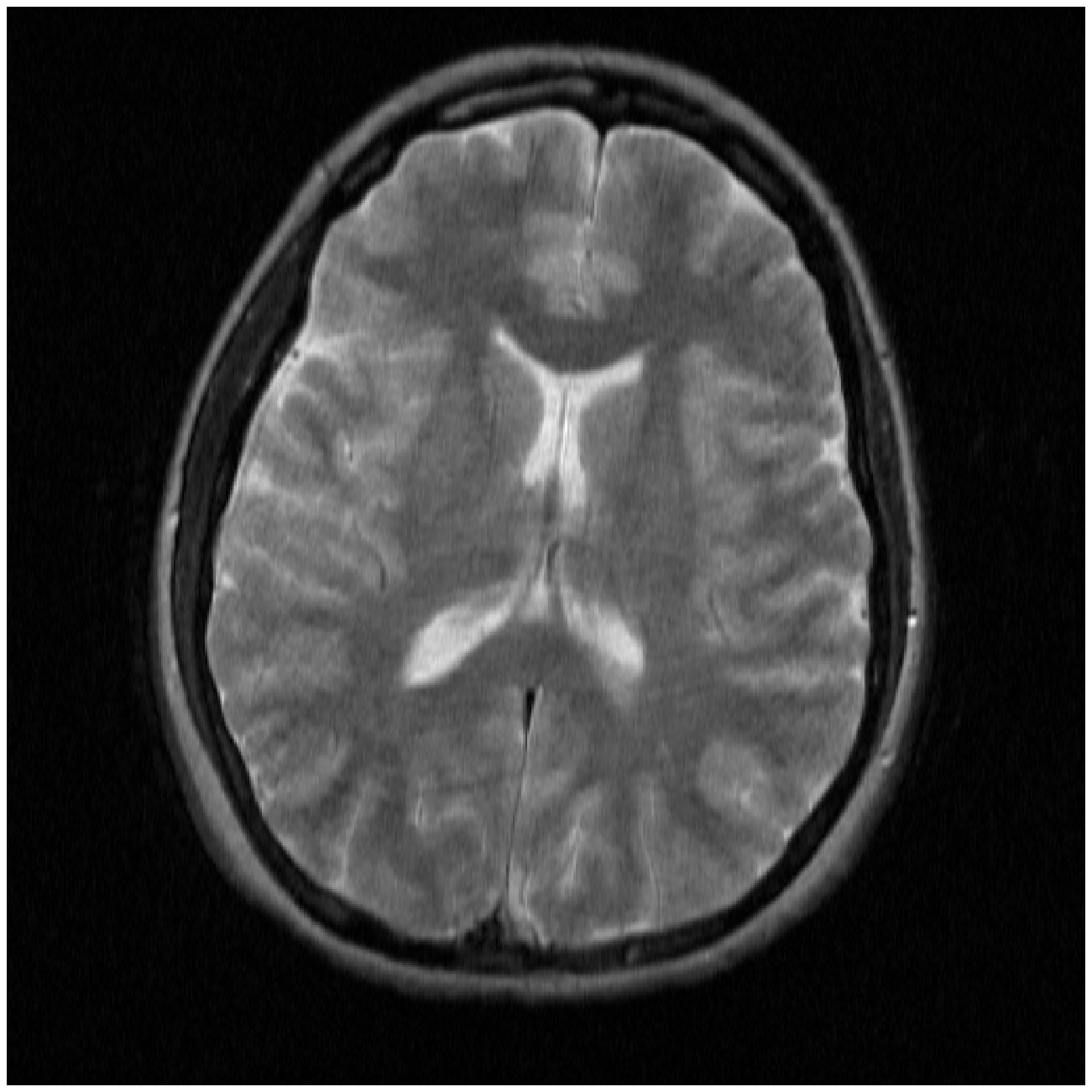}\\
(a) & (c) & (e)\\
\includegraphics[height=1.4in]{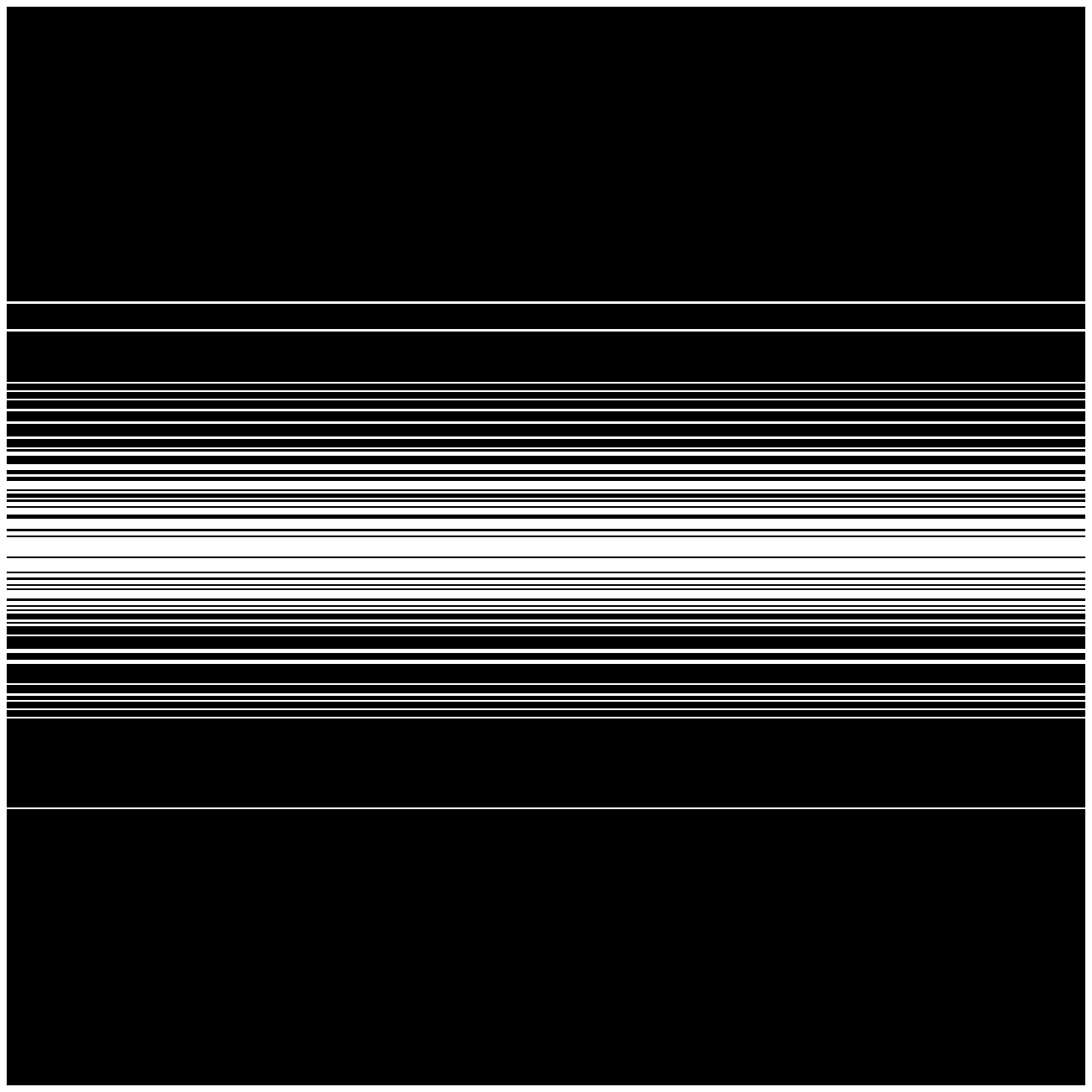}& 
\includegraphics[height=1.4in]{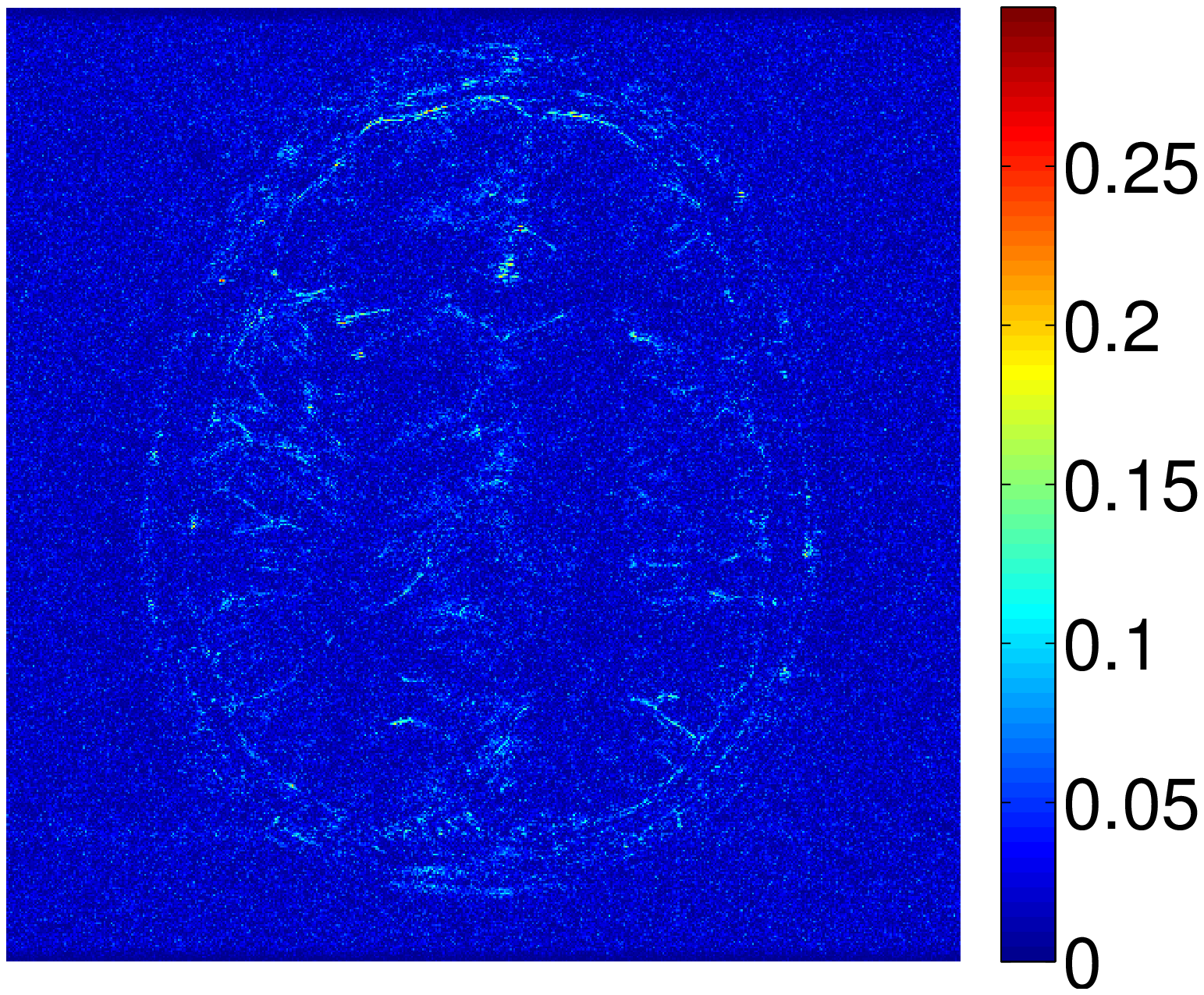}&
\includegraphics[height=1.4in]{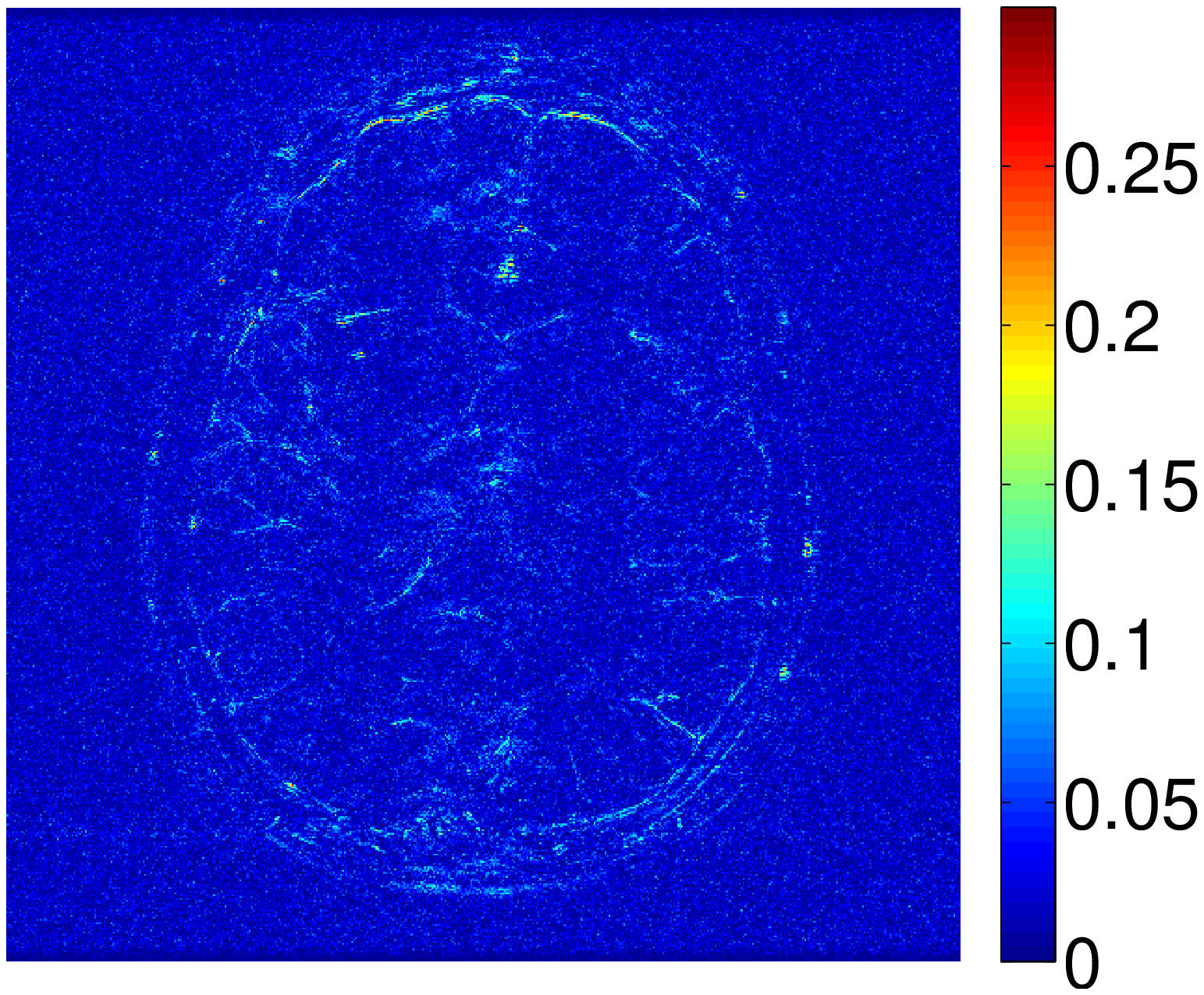}\\
(b) & (d) & (f)\\
\end{tabular}
\caption{Cartesian sampling with 7 fold undersampling: (a) Initial zero-filling reconstruction (27.9 dB); (b) sampling mask in k-space; (c) TLMRI reconstruction (31.24 dB); (d) magnitude of TLMRI reconstruction error; (e) DLMRI reconstruction (30.91 dB);  and (f) magnitude of DLMRI reconstruction error.}
\label{im3bcs}
\end{center}
\end{figure}

While our results show some (preliminary) potential for the proposed sparsifying transform-based blind compressed sensing framework (for MRI), a much more detailed investigation will be presented elsewhere. 
Combining the proposed scheme with the patch-based directional Wavelets ideas \cite{Qu11, Qu12}, or non-local patch similarity ideas \cite{Qu2014843, Y7027820}, or extending our framework to learning overcomplete sparsifying transforms (c.f., \cite{saiwen}) could potentially boost transform-based BCS performance further.



\section{Conclusions} \label{sec6}
In this work, we presented a novel sparsifying transform-based framework for blind compressed sensing. Our formulations exploit the (adaptive) transform domain sparsity of overlapping image patches in 2D, or voxels in 3D. The proposed formulations are however highly nonconvex. 
Our block coordinate descent-type algorithms for solving the proposed problems involve efficient update steps.
Importantly, our algorithms are guaranteed to converge to the critical points of the objectives defining the proposed formulations. These critical points are also guaranteed to be at least partial global and partial local minimizers. Our numerical examples showed the usefulness of the proposed scheme for magnetic resonance image reconstruction from highly undersampled k-space data. Our approach while being highly efficient also provides promising MR image reconstruction quality. The usefulness of the proposed blind compressed sensing methods in other inverse problems and imaging applications merits further study. A detailed investigation of the theoretical rate of convergence in our algorithms is also of potential interest.



\Appendix



\section{Convergence Proofs}  \label{apbcs1}

\subsection{Proof of Theorem \ref{theorem1bc}}  \label{ap1a}

In this proof, we let $\tilde{H_{s}}(Z)$ denote the \emph{set} of all optimal projections of $Z \in \mathbb{C}^{n \times N}$ onto the $s$-$\ell_{0}$ ball $\left \{ B \in \mathbb{C}^{n \times N} : \left \| B \right \|_{0} \leq s \right \}$, i.e., $\tilde{H_{s}}(Z)$ is the set of minimizers for the following problem.
\begin{equation} \label{dfdc22}
\tilde{H_{s}}(Z) = \underset{B \,:\, \left \| B \right \|_{0} \leq s}{\arg\min}\: \left \| Z-B \right \|_{F}^{2}
\end{equation}

Let $\left \{ W^{t}, B^{t}, x^{t} \right \}$ denote the iterate sequence generated by Algorithm A1 with measurements $y \in \mathbb{C}^{m}$ and initial $(W^{0}, B^{0}, x^{0})$. Assume that the initial $(W^{0}, B^{0}, x^{0})$ is such that $g\left ( W^{0}, B^{0}, x^{0}  \right )$ is finite.
Throughout this proof, we let $X^{t}$ be the matrix with $P_{j} x^{t}$ ($1\leq j \leq N$) as its columns.
The various results in Theorem \ref{theorem1bc} are now proved in the following order.
\begin{enumerate}[(i)]
\item The sequence of values of the objective in Algorithm A1 converges.
\item The iterate sequence generated by Algorithm A1 has an accumulation point.
\item All the accumulation points of the iterate sequence are equivalent in terms of their objective value.
\item Every accumulation point of the iterates is a critical point of the objective $g\left (W, B, x \right )$ satisfying \eqref{cnbcs4}, \eqref{cnbcs5}, and \eqref{cnbcs6}.
\item  The difference between successive image iterates  $\left \| x^{t} - x^{t-1} \right \|_{2}$, converges to zero.
\item Every accumulation point of the iterates is a local minimizer of $g\left (W, B, x \right )$ with respect to  $(B, x)$ or $(W, B)$.
\end{enumerate}

The following two Lemmas establish the convergence of the objective, and the boundedness of the iterate sequence. 

\begin{lemma}\label{lemma2bc} 
Let $\left \{ W^{t}, B^{t}, x^{t} \right \}$ denote the iterate sequence generated by Algorithm A1 with input $y \in \mathbb{C}^{m}$ and initial $(W^{0}, B^{0}, x^{0})$. Then, the sequence of objective function values  $\left \{ g\left ( W^{t}, B^{t}, x^{t}  \right ) \right \}$ is monotone decreasing, and converges to a finite value $g^{*} = g^{*}(W^{0}, B^{0}, x^{0})$.
\end{lemma}

\begin{proof}: Algorithm A1 first alternates between the transform update and sparse coding steps (Step 2 in algorithm pseudocode), with fixed image $x$. In the transform update step, we obtain a global minimizer (i.e., \eqref{tru1bcs}) with respect to $W$ for Problem \eqref{bcs6}. In the sparse coding step too, we obtain an exact solution for $B$ in Problem \eqref{bcs2} as $\hat{B} = H_{s}(Z)$. Therefore, the objective function can only decrease when we alternate between the transform update and sparse coding steps (similar to the case in \cite{sbclsTS2}). Thus, we have $g\left ( W^{t+1}, B^{t+1}, x^{t}  \right ) \leq g\left ( W^{t}, B^{t}, x^{t}  \right )$. 

In the image update step of Algorithm A1 (Step 4 in algorithm pseudocode), we obtain an exact solution to the constrained least squares problem \eqref{bcs9}. Therefore, the objective in this step satisfies $g\left (W^{t+1}, B^{t+1}, x^{t+1}  \right ) \leq g\left ( W^{t+1}, B^{t+1}, x^{t}  \right )$. By the preceding arguments, we have $g\left ( W^{t+1}, B^{t+1}, x^{t+1}  \right ) \leq g\left ( W^{t}, B^{t}, x^{t}  \right )$, for every $t$. Now, every term, except $\lambda Q(W)$, in the objective \eqref{cnbcs1} is trivially non-negative. Furthermore, the $Q(W)$ regularizer is bounded as $Q(W) \geq  \frac{n}{2}$ (cf. \cite{sabres}). Therefore, the objective $g\left (W, B, x \right ) > 0$. Since the sequence $ \left \{ g\left ( W^{t}, B^{t}, x^{t}  \right ) \right \}$ is monotone decreasing and lower bounded, it converges.
\end{proof}

\begin{lemma}\label{lemma3bc} 
The iterate sequence  $\left \{ W^{t}, B^{t}, x^{t} \right \}$ generated by Algorithm A1 is bounded, and it has at least one accumulation point.
\end{lemma}

\begin{proof}: The existence of a convergent subsequence (and hence, an accumulation point) for a bounded sequence is a standard result. Therefore, we only prove the boundedness of the iterate sequence.

Since $\left \| x^{t} \right \|_{2}\leq C$ $\forall$ $t$ trivially, we have that the sequence $ \left \{ x^{t}\right \}$ is bounded. We now show the boundedness of $ \left \{ W^{t}\right \}$.
Let us denote the objective $g(W^{t}, B^{t}, x^{t})$ as $g^{t}$. It is obvious that the squared $\ell_{2}$ norm terms and the barrier functions $\psi(B^{t})$ and $\chi(x^{t})$
in the objective $g^{t}$ \eqref{cnbcs1}, are non-negative. Therefore, we have
\begin{equation} \label{cstls1}
\lambda Q(W^{t}) \leq  g^{t} \leq g^{0}
\end{equation} 
where the second inequality follows from Lemma \ref{lemma2bc}. Now, the function $ Q(W^{t}) =\sum_{i=1}^{n}  (0.5\alpha_{i}^{2} - \log \, \alpha_{i}) $, where  $\alpha_{i}$ ($1 \leq i \leq n$) are the singular values of $W^{t}$, is a coercive  function of the (non-negative) singular values, and therefore, it has bounded lower level sets \footnote{The lower level sets of a function $f: A \subset \mathbb{R}^{n} \mapsto \mathbb{R}$ (where $A$ is unbounded) are bounded if $\lim_{t \to \infty} f(z^{t}) = + \infty$ whenever $\left \{ z^{t} \right \}\subset A $ and $\lim_{t \to \infty} \left \| z^{t} \right \| = \infty$.}. Combining this fact with \eqref{cstls1}, we can immediately conclude that $\exists$ $c_{0} \in \mathbb{R}$ depending on $g^{0}$ and $\lambda$, such that $\left \| W^{t} \right \|_{F}  \leq c_0$ $\forall$ $t$.

Finally, the boundedness of $ \left \{ B^{t}\right \}$ follows from the following arguments. First, for Algorithm A1 (see pseudocode), we have that $B^{t} = H_{s}\begin{pmatrix}
W^{t} X^{t-1}
\end{pmatrix}$. 
Therefore, by the definition of $H_{s} (\cdot)$, we have
\begin{equation} \label{cstls1b}
\begin{Vmatrix}
B^{t} 
\end{Vmatrix}_{F} = \begin{Vmatrix}
H_{s}\begin{pmatrix}
W^{t} X^{t-1}
\end{pmatrix}
\end{Vmatrix}_{F} \leq \begin{Vmatrix}
W^{t} X^{t-1}
\end{Vmatrix}_{F} \leq \begin{Vmatrix}
W^{t}
\end{Vmatrix}_{2}\begin{Vmatrix}
X^{t-1}
\end{Vmatrix}_{F}
\end{equation}
Since, by our previous arguments, $W^{t}$ and $X^{t-1}$ are both bounded by constants independent of $t$, we have that the sequence of sparse code matrices $ \left \{ B^{t}\right \}$, is also bounded.
\end{proof}

We now establish some key optimality properties of the accumulation points of the iterate sequence in Algorithm A1.

\begin{lemma}\label{lemma7bc} \vspace{0.02in}
All the accumulation points of the iterate sequence generated by Algorithm A1 with a given initial $(W^{0}, B^{0}, x^{0})$ correspond to a common objective value $g^{*}$. Thus, they are equivalent in that sense. 
\end{lemma}

\begin{proof}:
Consider the subsequence $\left \{ W^{q_t}, B^{q_t}, x^{q_t} \right \}$ (indexed by $q_t$) of the iterate sequence, that converges to the accumulation point $\left ( W^{*}, B^{*}, x^{*} \right )$. 
Before proving the lemma, we establish some simple properties of $\left ( W^{*}, B^{*}, x^{*} \right )$.
First, equation \eqref{cstls1} implies that $- \log \left | \det W^{q_t} \right | \leq (g^{0}/\lambda)$, for every $t$. This further implies that $\left | \det W^{q_t} \right | \geq e^{-g^{0}/\lambda} >0$ $\forall$ $t$.
Therefore, due to the continuity of the function $\left | \det W \right |$, the limit $W^{*}$ of the subsequence is also non-singular, with $\left | \det W^{*} \right | \geq e^{-g^{0}/\lambda}$. 
Second, $B^{q_t} =  H_{s}\begin{pmatrix} W^{q_t} X^{q_{t} - 1} \end{pmatrix} $, where $X^{q_{t}-1}$ is the matrix with $P_{j} x^{q_{t} - 1}$ ($1\leq j \leq N$) as its columns. Thus, $B^{q_t}$
trivially satisfies $\psi (B^{q_t}) = 0$ for every $t$. 
Now, $\left \{ B^{q_t} \right \}$ converges to $B^{*}$,  which makes $B^{*}$ the limit of a sequence of matrices, each of which has no more than $s$ non-zeros. Thus, the limit $B^{*}$ obviously cannot have more than $s$ non-zeros.
Therefore, $ \left \| B^{*} \right \|_{0} \leq s$, or equivalently $\psi (B^{*}) = 0$.
Finally, since $x^{q_{t}}$ satisfies the constraint $\left \| x^{q_{t}} \right \|_{2}^{2}\leq C^{2}$, we have $\chi(x^{q_t})=0$ $\forall$ $t$. Additionally, since $x^{q_t} \to x^{*}$ as $t \to \infty$, we also have $\left \| x^{*} \right \|_{2}^{2} = \lim_{t\to \infty} \left \| x^{q_{t}} \right \|_{2}^{2}\leq C^{2}$. Therefore, $\chi(x^{*})=0$.  Now, it is obvious from the above arguments that 
\begin{equation} \label{cstl2922}
\lim_{t \to \infty} \chi(x^{q_t}) = \chi(x^{*}), \;\; \lim_{t \to \infty} \psi (B^{q_t}) = \psi (B^{*})
\end{equation}
Moreover, due to the continuity of $Q(W)$ at non-singular matrices $W$, we have $\lim_{t \to \infty} Q(W^{q_t}) = Q(W^{*})$. We now use these limits along with some simple properties of convergent sequences (e.g., the limit of the sum/product of convergent sequences is equal to the sum/product of their limits), to arrive at the following result.
\begin{align}
\nonumber & \lim_{t \to \infty} g(W^{q_t}, B^{q_t}, x^{q_t}) = \\ 
 \nonumber &   \lim_{t \to \infty}\begin{Bmatrix}
\sum_{j=1}^{N} \left \| W^{q_t} P_{j}x^{q_t}- b_{j}^{q_t} \right \|_{2}^{2} + \nu \left \| Ax^{q_t}-y \right \|_{2}^{2} + \lambda \, Q(W^{q_t}) + \psi (B^{q_t})
\end{Bmatrix}\\
\nonumber & + \lim_{t \to \infty} \chi(x^{q_t}) = \sum_{j=1}^{N} \left \| W^{*} P_{j}x^{*}- b_{j}^{*} \right \|_{2}^{2} + \nu \left \| Ax^{*}-y \right \|_{2}^{2} + \lambda \, Q(W^{*}) + \psi (B^{*}) + \chi(x^{*}) \\
 & = g(W^{*}, B^{*}, x^{*})   \label{csbldf00}
\end{align}
The above result together with the fact (from Lemma \ref{lemma2bc}) that $\lim_{t \to \infty} g(W^{q_t}, B^{q_t}, x^{q_t}) = g^{*}$ implies that $g(W^{*}, B^{*}, x^{*}) = g^{*}$.
\end{proof}

The following lemma establishes partial global optimality with respect to the image, of every accumulation point.

\begin{lemma}\label{lemma6bc} \vspace{0.02in}
Any accumulation point $\left ( W^{*}, B^{*}, x^{*} \right )$ of the iterate sequence generated by Algorithm A1 satisfies
\begin{equation} \label{cstls6}
 x^{*} \in \underset{x}{\arg\min} \; \,  g\left ( W^{*}, B^{*}, x \right )
\end{equation}
\end{lemma}

\begin{proof}:
Consider the subsequence $\left \{ W^{q_t}, B^{q_t}, x^{q_t} \right \}$ (indexed by $q_t$) of the iterate sequence, that converges to the accumulation point $\left ( W^{*}, B^{*}, x^{*} \right )$. Then, due to the optimality of $x^{q_t}$ in the image update step of Algorithm A1, we have the following inequality for any $t$ and any $x \in \mathbb{C}^{p}$.
\begin{align}
\nonumber \sum_{j=1}^{N} \left \| W^{q_t} P_{j}x^{q_t}- b_{j}^{q_t} \right \|_{2}^{2} + \nu \left \| Ax^{q_t}-y \right \|_{2}^{2} & + \chi(x^{q_t}) \leq \sum_{j=1}^{N} \left \| W^{q_t} P_{j}x- b_{j}^{q_t} \right \|_{2}^{2} \\
& \;\;\; + \nu \left \| Ax-y \right \|_{2}^{2} + \chi(x)  \label{cstls6b}
\end{align}
Now, by \eqref{cstl2922}, we have that $\lim_{t \to \infty} \chi(x^{q_t}) = \chi(x^{*})$. Taking the limit $t \to \infty$ term by term on both sides of \eqref{cstls6b} for some fixed $x \in \mathbb{C}^{p}$ yields the following result
\begin{align}
\nonumber \sum_{j=1}^{N} \left \| W^{*} P_{j}x^{*}- b_{j}^{*} \right \|_{2}^{2} + \nu \left \| Ax^{*}-y \right \|_{2}^{2} & + \chi(x^{*}) \leq \sum_{j=1}^{N} \left \| W^{*} P_{j}x- b_{j}^{*} \right \|_{2}^{2} \\
& \;\;\;  + \nu \left \| Ax-y \right \|_{2}^{2} + \chi(x)  \label{cstls6d}
\end{align}



Since the choice of $x$ in \eqref{cstls6d} is arbitrary, \eqref{cstls6d} holds for any $x \in \mathbb{C}^{p}$.
Recall that $\psi (B^{*})=0$ and $Q(W^{*})$ is finite based on the arguments in the proof of Lemma \ref{lemma7bc}.
Therefore, \eqref{cstls6d} implies that $g(W^{*}, B^{*}, x^{*}) \leq g(W^{*}, B^{*}, x)$ $\forall$ $x \in \mathbb{C}^{p}$. This establishes the result \eqref{cstls6} of the Lemma.
\end{proof}

The following lemma will be used to establish that the change between successive image iterates $\left \| x^{t} - x^{t-1} \right \|_{2}$, converges to 0.

\begin{lemma}\label{lemma6bcwe} \vspace{0.02in}
Consider the subsequence $\left \{ W^{q_t}, B^{q_t}, x^{q_t} \right \}$ (indexed by $q_t$) of the iterate sequence, that converges to the accumulation point $\left ( W^{*}, B^{*}, x^{*} \right )$. Then, the subsequence  $\left \{ x^{q_{t} - 1} \right \}$ also converges to $x^{*}$.
\end{lemma}

\begin{proof}:
First, by Lemma \ref{lemma7bc}, we have that 
\begin{equation} \label{csbldf0}
g(W^{*}, B^{*}, x^{*}) =g^{*}
\end{equation}
Next, consider a convergent subsequence $\begin{Bmatrix} x^{q_{n_t} - 1} \end{Bmatrix}$ of (the bounded sequence) $\left \{ x^{q_{t} - 1} \right \}$ that converges to say $x^{**}$. Now, applying the same arguments as in Equation \eqref{csbldf00} (in the proof of Lemma \ref{lemma7bc}), but with respect to the (convergent) subsequence $\begin{Bmatrix}  x^{q_{n_t} - 1},  W^{q_{n_t}}, B^{q_{n_t}}
\end{Bmatrix}$, we have that 
\begin{equation} \label{csbldf2}
\lim_{t \to \infty} g(W^{q_{n_t}}, B^{q_{n_t}}, x^{q_{n_t}-1}) = g(W^{*}, B^{*}, x^{**}) 
\end{equation}
Now, the monotonic decrease of the objective in Algorithm A1 implies
\begin{equation} \label{csbldf1}
 g(W^{q_{n_t}}, B^{q_{n_t}}, x^{q_{n_t}})  \leq g(W^{q_{n_t}}, B^{q_{n_t}}, x^{q_{n_t}-1})  \leq g(W^{q_{n_t}-1}, B^{q_{n_t}-1}, x^{q_{n_t}-1}) 
\end{equation}
Taking the limit $t \to \infty$ all through \eqref{csbldf1} and using Lemma \ref{lemma2bc} (for the extreme left/right limits), and Equation \ref{csbldf2} (for the middle limit), immediately yields that $g(W^{*}, B^{*}, x^{**})  = g^{*}$. This result together with \eqref{csbldf0} implies $g(W^{*}, B^{*}, x^{**}) =g(W^{*}, B^{*}, x^{*})$.

 Now, we know by Lemma \ref{lemma6bc} that
\begin{equation} \label{csbldf11}
x^{*} \in \underset{x}{\arg\min} \; \,  g\left ( W^{*}, B^{*}, x \right )
\end{equation}
Furthermore, it was shown in Section \ref{imupstepg} that if the set of patches in our formulation (P1) cover all pixels in the image (always true for the case of periodically positioned overlapping image patches), then the minimization of $g\left ( W^{*}, B^{*}, x \right )$ with respect to $x$ has a unique solution. Therefore, $x^{*}$ is the unique minimizer in \eqref{csbldf11}. Combining this with the fact that $g(W^{*}, B^{*}, x^{**}) =g(W^{*}, B^{*}, x^{*})$ yields that $ x^{**} =  x^{*}$.
Since we worked with an arbitrary convergent subsequence  $\begin{Bmatrix} x^{q_{n_t} - 1} \end{Bmatrix}$ (of $\left \{ x^{q_{t} - 1} \right \}$) in the above proof,  we have that $x^{*}$ is the limit of any convergent subsequence of $\left \{ x^{q_{t} - 1} \right \}$.
Finally, since every convergent subsequence of the bounded sequence $\left \{ x^{q_{t} - 1} \right \}$ converges to $x^{*}$, it means that the (compact) set of accumulation points of $\left \{ x^{q_{t} - 1} \right \}$ coincides with $x^{*}$.
Since $x^{*}$ is the unique accumulation point of a bounded sequence, we therefore have \cite{malk122} that the sequence $\left \{ x^{q_{t} - 1} \right \}$ itself converges to $x^{*}$, which completes the proof.
\end{proof}

\begin{lemma}\label{lemma6bcsjdsd} \vspace{0.02in}
The iterate sequence $ \left \{ W^{t}, B^{t}, x^{t} \right \}$ in Algorithm A1 satisfies
\begin{equation}\label{csbldf900}
\lim_{t \to \infty} \left \| x^{t} - x^{t-1} \right \|_{2} =  0
\end{equation}
\end{lemma}

\begin{proof}: 
Consider the sequence $ \left \{ a^{t} \right \}$ with  $a^{t} \triangleq  \left \| x^{t} - x^{t-1} \right \|_{2} $. We will show below that every convergent subsequence of the bounded sequence $ \left \{ a^{t} \right \}$ converges to 0, thereby implying that 0 is both the limit inferior and limit superior of $ \left \{ a^{t} \right \}$, which means that the sequence $ \left \{ a^{t} \right \}$ itself converges to 0.


Now, consider a convergent subsequence $ \left \{ a^{q_{t}} \right \}$ of $ \left \{ a^{t} \right \}$. Since the sequence $\left \{ W^{q_t}, B^{q_t}, x^{q_t} \right \}$  is bounded, there exists a convergent subsequence $\left \{ W^{q_{n_t}}, B^{q_{n_t}}, x^{q_{n_t}} \right \}$ converging to say $\left ( W^{*}, B^{*}, x^{*} \right )$.
By Lemma \ref{lemma6bcwe}, we then have that $\begin{Bmatrix} x^{q_{n_t} - 1} \end{Bmatrix}$ also converges to $x^{*}$.
Thus, the subsequence  $ \left \{ a^{q_{n_t}} \right \}$ with $a^{q_{n_t}} \triangleq \left \|  x^{q_{n_{t}}} -  x^{q_{n_{t}} - 1} \right \|_{2}$ converges to $0$. Since, $ \left \{ a^{q_{n_t}} \right \}$ itself is a subsequence of a convergent sequence, we must have that $ \left \{ a^{q_{t}} \right \}$ converges to the same limit (i.e., 0). We have thus shown that zero is the limit of any convergent subsequence of  $ \left \{ a^{t} \right \}$.
\end{proof}


The next property is the partial global optimality of every accumulation point with respect to the sparse code. In order to establish this property, we need the following result.

\begin{lemma}\label{lemma4bc} 
Consider a  bounded matrix sequence $\left \{ Z^{k} \right \} $ with $Z^{k} \in \mathbb{C}^{n \times N}$, that converges to $Z^{*}$. 
Then, every accumulation point of $\left \{ H_{s}(Z^{k})  \right \}$ belongs to the set $\tilde{H_{s}}(Z^{*})$. 
\end{lemma}

\begin{proof}: The proof is very similar to that for Lemma 10 in \cite{sbclsTS2}.
\end{proof}

\begin{lemma}\label{lemma5bc} \vspace{0.02in}
Any accumulation point $\left ( W^{*}, B^{*}, x^{*} \right )$ of the iterate sequence generated by Algorithm A1 satisfies
\begin{equation} \label{cstls3}
 B^{*} \in \underset{B}{\arg\min} \; \,  g\left ( W^{*}, B, x^{*} \right )
\end{equation}
Moreover, denoting by $X^{*} \in \mathbb{C}^{n \times N}$ the matrix whose $\mathrm{j^{th}}$ column is $P_{j} x^{*}$, for $1\leq j \leq N$, the above condition can be equivalently stated as
\begin{equation} \label{cstls4}
B^{*} \in \tilde{H_{s}}(W^{*}X^{*})
\end{equation}
\end{lemma}
\begin{proof}:
Consider the subsequence $\left \{ W^{q_t}, B^{q_t}, x^{q_t} \right \}$ (indexed by $q_t$) of the iterate sequence, that converges to the accumulation point $\left ( W^{*}, B^{*}, x^{*} \right )$. 
Then, by Lemma \ref{lemma6bcwe}, $\left \{ x^{q_{t} - 1} \right \}$ converges to $x^{*}$, and the following inequalities hold. 
\begin{equation} \label{cstls5}
B^{*} = \lim_{t\to \infty} B^{q_t} =  \lim_{t\to \infty} H_{s}\begin{pmatrix}
W^{q_t} X^{q_{t}  - 1}
\end{pmatrix} \in \tilde{H_{s}}(W^{*}X^{*})
\end{equation}
Since $W^{q_t} X^{q_{t} - 1} \to W^{*}X^{*}$ as $t \to \infty$, the last containment relationship above follows by Lemma \ref{lemma4bc}. While we have proved \eqref{cstls4}, the result in \eqref{cstls3} now immediately follows by applying the definition of $\tilde{H_{s}}(\cdot)$ from \eqref{dfdc22}.
\end{proof}

The next result pertains to the partial global optimality with respect to $W$, of every accumulation point in Algorithm A1. We use the following lemma (simple extension of Lemma 1 in \cite{sbclsTS2} to the complex field) to establish the result.

\begin{lemma}\label{lemma2s2bbc}
Consider a sequence $\left \{ M_{k} \right \}$ with $M_{k} \in \mathbb{C}^{n \times n}$, that converges to $M$. For each $k$, let $V_{k} \Sigma_{k} R_{k}^{H}$ denote a full SVD of $M_{k}$. Then, every accumulation point  $\left ( V, \Sigma, R \right )$ of the sequence $\left \{ V_{k}, \Sigma_{k}, R_{k} \right \}$ is such that $V \Sigma R^{H}$ is a full SVD of $M$. In particular, $\left \{ \Sigma_{k} \right \}$ converges to $\Sigma$, the $n \times n$ singular value matrix of $M$.
\end{lemma}

\begin{lemma}\label{lemma8bc} \vspace{0.02in}
Any accumulation point $\left ( W^{*}, B^{*}, x^{*} \right )$ of the iterate sequence generated by Algorithm A1 satisfies
\begin{equation} \label{cstls7}
 W^{*} \in \underset{W}{\arg\min} \; \,  g\left ( W, B^{*}, x^{*} \right )
\end{equation}
\end{lemma}

\begin{proof}:
The proposed Algorithm A1 involves $\hat{M}$ alternations between the transform update and sparse coding steps in every outer iteration. At a certain (outer) iteration $t$, let us denote the intermediate transform update outputs as $\tilde{W}^{\left ( l, t \right )}  $, $1 \leq l \leq \hat{M}$. Then, $W^{t} = \tilde{W}^{\left ( \hat{M}, t \right )}  $. We will only use the sequence $\begin{Bmatrix}\tilde{W}^{\left ( 1, t \right )}  
\end{Bmatrix}$ (i.e., the intermediate outputs for $l=1$) in the following proof. 


Consider the subsequence $\left \{ W^{q_t}, B^{q_t}, x^{q_t} \right \}$ (indexed by $q_t$) of the iterate sequence in Algorithm A1, that converges to the accumulation point $\left ( W^{*}, B^{*}, x^{*} \right )$. 
Then, $\tilde{W}^{\left ( 1, q_{t}+1 \right )} $ is computed as follows, using the full SVD $ V^{q_t}  \Sigma^{q_t}  \left ( R^{q_t}  \right )^{H}$, of $\left ( L^{q_t} \right )^{-1}X^{q_t}\left ( B^{q_t} \right )^{H}$, where $ L^{q_t} = \begin{pmatrix}
X^{q_t}\left ( X^{q_t} \right )^{H} + 0.5 \lambda I
\end{pmatrix}^{1/2}$. 
\begin{equation} \label{cstls7a0}
\tilde{W}^{\left ( 1, q_{t}+1 \right )} =0.5 R^{q_{t}} \left(\Sigma^{q_{t}} + \left ( \left ( \Sigma^{q_{t}} \right )^{2}+2\lambda I \right )^{\frac{1}{2}}\right)\left ( V^{q_{t}} \right )^{H}\left ( L^{q_t} \right )^{-1} 
\end{equation}
To prove the lemma, we will consider the limit $t \to \infty$ in \eqref{cstls7a0}. In order to take this limit, we need the following results.
First, due to the continuity of the matrix square root and matrix inverse functions at positive definite matrices, we have that the following limit holds, where $X^{*} \in \mathbb{C}^{n \times N}$ has $P_{j} x^{*}$ ($1 \leq j \leq N$) as its columns.
\begin{equation*}  
\lim_{t \to \infty}\left ( L^{q_t} \right )^{-1} = \lim_{t \to \infty}\begin{pmatrix}
X^{q_t}\left ( X^{q_t} \right )^{H} + 0.5 \lambda I
\end{pmatrix}^{-1/2} = \begin{pmatrix}
X^{*}\left ( X^{*} \right )^{H} + 0.5 \lambda I
\end{pmatrix}^{-1/2} 
\end{equation*}
Next, defining $L^{*} \triangleq \begin{pmatrix} X^{*}\left ( X^{*} \right )^{H} + 0.5 \lambda I
\end{pmatrix}^{1/2}$, we also have that
\begin{equation} \label{cstls7b}
 \lim_{t \to \infty } \left ( L^{q_t} \right )^{-1}  X^{q_t} \left ( B^{q_t}  \right )^{H}  = \left ( L^{*} \right )^{-1} X^{*}  \left ( B^{*}  \right )^{H}
\end{equation}
Applying Lemma \ref{lemma2s2bbc} to \eqref{cstls7b}, we have that every accumulation point $\left ( V^{*}, \Sigma^{*}, R^{*} \right )$ of the sequence $\left \{ V^{q_t}, \Sigma^{q_t}, R^{q_t} \right \}$ is such that $V^{*} \Sigma^{*}\left ( R^{*} \right )^{H}$ is a full SVD of $\left ( L^{*} \right )^{-1} X^{*}  \left ( B^{*}  \right )^{H}$. Now, consider a convergent subsequence $\left \{ V^{q_{n_t}}, \Sigma^{q_{n_t}}, R^{q_{n_t}} \right \}$ of $\left \{ V^{q_t}, \Sigma^{q_t}, R^{q_t} \right \}$, with limit $\left ( V^{*}, \Sigma^{*}, R^{*} \right )$. Then, taking the limit $t \to \infty$ in \eqref{cstls7a0} along this subsequence, we have
\begin{equation*}
W^{**}  \triangleq \lim_{t \to \infty } \tilde{W}^{\left ( 1, q_{n_{t}}+1 \right )}   = 0.5 R^{*} \left(\Sigma^{*} + \left ( \left ( \Sigma^{*} \right )^{2}+2\lambda I \right )^{\frac{1}{2}}\right)\left ( V^{*} \right )^{H}\left ( L^{*} \right )^{-1} 
\end{equation*}
Combining this result with the aforementioned definitions of the square root $L^{*}$ and the full SVD $V^{*} \Sigma^{*}\left ( R^{*} \right )^{H}$, and applying Proposition \ref{propel1bcs}, we get
\begin{equation} \label{cstls7c}
W^{**} \in \underset{W}{\arg\min} \; \,  g\left ( W, B^{*}, x^{*} \right )
\end{equation}
Finally, applying the same arguments as in the proof of Lemma \ref{lemma7bc} to the subsequence $\begin{Bmatrix} B^{q_{n_t}}, x^{q_{n_t}}, \tilde{W}^{\left ( 1, q_{n_{t}}+1 \right )}
\end{Bmatrix}$, we easily get that $g^{*}= g(W^{**}, B^{*}, x^{*})$. Since, by Lemma \ref{lemma7bc}, we also have that $g^{*} = g(W^{*}, B^{*}, x^{*})$, we get $g(W^{*}, B^{*}, x^{*}) = g(W^{**}, B^{*}, x^{*})$, which together with \eqref{cstls7c} immediately establishes the required result \eqref{cstls7}.
\end{proof}

The following lemma establishes that every accumulation point of the iterate sequence in Algorithm A1 is a critical point of the objective $g(W, B, x)$. All derivatives or sub-differentials are computed with respect to the real and imaginary parts of the corresponding variables/vectors/matrices below.

\begin{lemma}\label{lemma9bc} \vspace{0.02in}
Every accumulation point $\left ( W^{*}, B^{*}, x^{*} \right )$ of the iterate sequence generated by Algorithm A1 is a critical point of the objective $g(W, B, x)$ satisfiying
\begin{equation} \label{cstls8}
0 \in \partial  g\left ( W^{*}, B^{*}, x^{*} \right )
\end{equation}
\end{lemma}

\begin{proof}:
Consider the subsequence $\left \{ W^{q_t}, B^{q_t}, x^{q_t} \right \}$ (indexed by $q_t$) of the iterate sequence, that converges to the accumulation point $\left ( W^{*}, B^{*}, x^{*} \right )$. By Lemmas \ref{lemma8bc}, \ref{lemma6bc}, and \ref{lemma5bc}, we have that
\begin{align}
 W^{*} & \in \underset{W}{\arg\min} \; \,  g\left ( W, B^{*}, x^{*} \right ) \label{cstls8a}\\
 x^{*} & \in \underset{x}{\arg\min} \; \,  g\left ( W^{*}, B^{*}, x \right ) \label{cstls8b}\\
 B^{*} & \in \underset{B}{\arg\min} \; \,  g\left ( W^{*}, B, x^{*} \right ) \label{cstls8c}
\end{align}
The function $g(W, B, x)$ is continuously differentiable with respect to $W$ at non-singular points $W$. Since $W^{*}$ is a (non-singular) partial global minimizer in \eqref{cstls8a}, we have as a necessary condition that $\nabla_{W} g\left ( W^{*}, B^{*}, x^{*} \right ) = 0$. Next, recall the statement in Section \ref{prelim}, that a necessary condition for $z \in \mathbb{R}^{q}$ to be a minimizer of some function $\phi: \mathbb{R}^{q} \mapsto (-\infty, + \infty]$ is that $z$ satisfies $0 \in \partial \phi(z)$.
Now, since $x^{*}$ and $B^{*}$ are partial global minimizers in \eqref{cstls8b} and \eqref{cstls8c}, respectively, we therefore have that $0 \in \partial  g_{x}\left ( W^{*}, B^{*}, x^{*} \right )$ and $0 \in \partial  g_{B}\left ( W^{*}, B^{*}, x^{*} \right )$, respectively. It is also easy to derive these conditions directly using the definition (Definition \ref{def1}) of the sub-differential.

Finally, (see Proposition 3 in \cite{Attouchaa}) the subdifferential $\partial  g $ at $\left ( W^{*}, B^{*}, x^{*} \right )$ satisfies
\begin{equation} \label{cstls8d}
 \partial  g\left ( W^{*}, B^{*}, x^{*} \right ) = 
\nabla_{W} g\left ( W^{*}, B^{*}, x^{*} \right )\times 
\partial  g_{B}\left ( W^{*}, B^{*}, x^{*} \right )\times
\partial  g_{x}\left ( W^{*}, B^{*}, x^{*} \right )
\end{equation}
Now, using the preceding results, we easily have that $0 \in \partial  g\left ( W^{*}, B^{*}, x^{*} \right )$ above. Thus, every accumulation point in Algorithm A1 is a critical point of the objective.
\end{proof}

The following two lemmas establish pairwise partial local optimality of the accumulation points in Algorithm A1. Here, $X^{*} \in \mathbb{C}^{n \times N}$ is the matrix with $P_{j}x^{*}$ as its columns.

\begin{lemma}\label{lemma10bc} \vspace{0.02in}
Every accumulation point $\left ( W^{*}, B^{*}, x^{*} \right )$ of the iterate sequence generated by Algorithm A1 is a partial minimizer of the objective $g(W, B, x)$ with respect to $(W, B)$, in the sense of \eqref{cnbcs4b}, for sufficiently small $dW \in \mathbb{C}^{n \times n}$, and all $\Delta B \in \mathbb{C}^{n \times N}$ in the union of the regions R1 and R2 in Theorem \ref{theorem1bc}. Furthermore, if $ \left \|  W^{*}X^{*} \right \|_{0} \leq s$, then the $\Delta B$ in \eqref{cnbcs4b} can be arbitrary.
\end{lemma}

\begin{proof}:
Consider the subsequence $\left \{ W^{q_t}, B^{q_t}, x^{q_t} \right \}$ (indexed by $q_t$) of the iterate sequence, that converges to the accumulation point $\left ( W^{*}, B^{*}, x^{*} \right )$. 
First, by Lemmas \ref{lemma5bc} and \ref{lemma8bc}, we have
\begin{align}
B^{*} & \in \tilde{H_{s}}(W^{*}X^{*}) \label{cstls10a} \\
2 W^{*} X^{*} \left ( X^{*} \right )^{H} & - 2 B^{*}\left ( X^{*} \right )^{H} + \lambda W^{*} - \lambda \left ( W^{*} \right )^{-H} = 0 \label{cstls10b}
\end{align}
where \eqref{cstls10b} follows from the first order conditions for partial global optimality of $W^{*}$ in \eqref{cstls7}. The accumulation point $\left ( W^{*}, B^{*}, x^{*} \right )$ also satisfies $\psi (B^{*})=0$ and $\chi (x^{*})=0$.

Now, considering perturbations $dW \in \mathbb{C}^{n \times n}$, and $\Delta B \in \mathbb{C}^{n \times N}$, we have that
\begin{align} 
\nonumber & g(W^{*} + dW,  B^{*} +  \Delta B, x^{*})  =  \left \| W^{*}X^{*} - B^{*}  + (dW)X^{*} - \Delta B \right \|_{F}^{2}   \\ 
& + \nu \left \| Ax^{*}-y \right \|_{2}^{2} + 0.5 \lambda \,  \left \| W^{*} + dW \right \|_{F}^{2} - \lambda \,  \log \,\left | \mathrm{det \,} (W^{*} + dW) \right | + \psi (B^{*}+ \Delta B) \label{cstls10c}
\end{align}
In order to prove the condition \eqref{cnbcs4b} in Theorem \ref{theorem1bc}, it suffices to consider \emph{sparsity preserving} perturbations $\Delta B$, that is $\Delta B \in \mathbb{C}^{n \times N}$ such that $B^{*} + \Delta B$ has sparsity $\leq s$. Otherwise $ g(W^{*} + dW,  B^{*} +  \Delta B, x^{*}) = + \infty > g(W^{*}, B^{*}, x^{*})$ trivially. 
Therefore, we only consider sparsity preserving $\Delta B$ in the following, for which $\psi (B^{*}+ \Delta B) = 0$ in \eqref{cstls10c}. 

Since the image $x^{*}$ is fixed here, we can utilize \eqref{cstls10a} and \eqref{cstls10b}, and apply similar arguments as in the proof of Lemma 9 (equations (43)-(46)) in \cite{sbclsTS2}, to simplify the right hand side in \eqref{cstls10c}. The only difference is that the matrix transpose $(\cdot)^{T}$ operations in \cite{sbclsTS2} are replaced with Hermitian $(\cdot)^{H}$ operations here, and the operation $\left \langle Q, R \right \rangle$ involving two matrices (or, vectors) $Q$ and $R$ in \cite{sbclsTS2} is redefined \footnote{We include the $Re(\cdot)$ operation in the definition here, which allows for simpler notations in the rest of the proof. However, the $\left \langle Q, R \right \rangle$ defined in \eqref{trcincomp} is no longer the conventional inner product.} here as 
\begin{equation} \label{trcincomp}
\left \langle Q, R \right \rangle \triangleq Re\left \{ \mathrm{tr} (Q R^{H}) \right \}
\end{equation}
Upon such simplifications, we can conclude \cite{sbclsTS2} that $\exists \epsilon' > 0$ depending on $W^{*}$ such that whenever $\left \| dW \right \|_{F} < \epsilon'$, we have
\begin{equation} \label{cstls10d}
g(W^{*} + dW, B^{*} + \Delta B, x^{*} ) \geq g(W^{*}, B^{*}, x^{*})- 2 \left \langle W^{*}X^{*} - B^{*}, \Delta B \right \rangle
\end{equation}
Consider first the case of $\Delta B$ in region R1.
Then, the term $-\left \langle W^{*}X^{*} - B^{*}, \Delta B \right \rangle$ above is trivially non-negative for such $\Delta B$ in Theorem \ref{theorem1bc}, and therefore, $g(W^{*} + dW, B^{*} + \Delta B,  x^{*}) \geq g(W^{*}, B^{*},  x^{*})$ for $\Delta B \in \mathrm{R1}$.

Next, consider the case of $\Delta B$ in region R2.
Then, when $ \left \|  W^{*}X^{*} \right \|_{0} > s$, it is easy to see that any such sparsity preserving $\Delta B$ in region R2 in Theorem \ref{theorem1bc} will have its support (i.e., non-zero locations) contained in the support of $B^{*}  \in \tilde{H_{s}}(W^{*}X^{*})$. 
Therefore, $\left \langle W^{*}X^{*} - B^{*}, \Delta B \right \rangle = 0$ in this case.
On the other hand,  if $ \left \|  W^{*}X^{*} \right \|_{0} \leq s$, then by \eqref{cstls10a}, $W^{*}X^{*} - B^{*} = 0$.
Therefore, by these arguments, $\left \langle W^{*}X^{*} - B^{*}, \Delta B \right \rangle = 0$ for any sparsity preserving $\Delta B \in \mathrm{R2}$.
This result together with \eqref{cstls10d} implies $g(W^{*} + dW, B^{*} + \Delta B, x^{*}) \geq g(W^{*}, B^{*}, x^{*})$ for any $\Delta B \in \mathrm{R2}$.

Finally, if $ \left \|  W^{*}X^{*} \right \|_{0} \leq s$, then since $W^{*}X^{*} - B^{*} = 0$ in \eqref{cstls10d}, therefore  in this case, $\Delta B$ in \eqref{cnbcs4b} can be arbitrary. 
\end{proof}


\begin{lemma}\label{lemma11bc} \vspace{0.02in}
Every accumulation point $\left ( W^{*}, B^{*}, x^{*} \right )$ of the iterate sequence generated by Algorithm A1 is a partial minimizer of the objective $g(W, B, x)$ with respect to $(B, x)$, in the sense of \eqref{cnbcs5b}, for all $\tilde{\Delta } x \in \mathbb{C}^{p}$, and all $\Delta B \in \mathbb{C}^{n \times N}$ in the union of the regions R1 and R2 in Theorem \ref{theorem1bc}. Furthermore, if $ \left \|  W^{*}X^{*} \right \|_{0} \leq s$, then the $\Delta B$ in \eqref{cnbcs5b} can be arbitrary.
\end{lemma}

\begin{proof}:
Consider the subsequence $\left \{ W^{q_t}, B^{q_t}, x^{q_t} \right \}$ (indexed by $q_t$) of the iterate sequence, that converges to the accumulation point $\left ( W^{*}, B^{*}, x^{*} \right )$. It follows from Lemmas \ref{lemma5bc} and \ref{lemma6bc} that  $\psi (B^{*}) = \chi (x^{*})=0$. Now, considering perturbations $\Delta B \in \mathbb{C}^{n \times N}$ whose columns are denoted as $\Delta b_{j}$ ($1 \leq j \leq N$), and $\tilde{\Delta } x \in \mathbb{C}^{p}$, we have that
\begin{align} 
\nonumber g(W^{*},&  B^{*} + \Delta B,  x^{*} + \tilde{\Delta } x )  = \sum_{j=1}^{N} \begin{Vmatrix}
W^{*} P_{j}x^{*}- b_{j}^{*} + W^{*} P_{j}\tilde{\Delta } x - \Delta b_{j}
\end{Vmatrix}_{2}^{2}  \\
& + \lambda \, Q(W^{*}) + \nu \begin{Vmatrix}
Ax^{*}-y + A\tilde{\Delta } x
\end{Vmatrix}_{2}^{2}  + \psi (B^{*} + \Delta B) + \chi(x^{*} + \tilde{\Delta } x ) \label{cstls11a}
\end{align}
In order to prove the condition \eqref{cnbcs5b} in Theorem \ref{theorem1bc}, it suffices to consider \emph{sparsity preserving} perturbations $\Delta B$, that is $\Delta B \in \mathbb{C}^{n \times N}$ such that $B^{*} + \Delta B$ has sparsity $\leq s$. It also suffices to consider \emph{energy preserving} perturbations $\tilde{\Delta } x $, which are such that $\begin{Vmatrix} x^{*} + \tilde{\Delta } x 
\end{Vmatrix}_{2} \leq C$. For any other $\Delta B$ or $\tilde{\Delta } x $,  $g(W^{*}, B^{*} + \Delta B, x^{*} + \tilde{\Delta } x ) = + \infty > g(W^{*}, B^{*}, x^{*})$ trivially. Therefore, we only consider the energy/sparsity preserving perturbations in the following, for which $\psi (B^{*} + \Delta B) =0$ and $\chi(x^{*} + \tilde{\Delta } x ) = 0$ in \eqref{cstls11a}. Now, upon expanding the squared $\ell_{2}$ terms in \eqref{cstls11a}, and dropping non-negative perturbation terms, we get
\begin{align} 
\nonumber g(W^{*}, B^{*} + & \Delta B,  x^{*} + \tilde{\Delta } x ) \geq g(W^{*}, B^{*},  x^{*}) 
- 2 \sum_{j=1}^{N}\left \langle W^{*} P_{j}x^{*}- b_{j}^{*}, \Delta b_{j} \right \rangle \\
& + 2 \sum_{j=1}^{N}\left \langle W^{*} P_{j}x^{*}- b_{j}^{*}, W^{*} P_{j}\tilde{\Delta } x\right \rangle + 2 \nu \left \langle Ax^{*}-y, A\tilde{\Delta } x \right \rangle
 \label{cstls11b}
\end{align}
where the $\left \langle \cdot, \cdot \right \rangle$ notation is as defined in \eqref{trcincomp}. Now, using arguments identical to those used with \eqref{cstls10d}, it is clear that the term $- 2 \sum_{j=1}^{N}\left \langle W^{*} P_{j}x^{*}- b_{j}^{*}, \Delta b_{j} \right \rangle \geq 0$ in \eqref{cstls11b} for all sparsity preserving $\Delta B \in \mathrm{R1 \cup R2}$.
Furthermore, since by Lemma \ref{lemma6bc}, $x^{*}$ is a global minimizer of $g(W^{*}, B^{*}, x)$, it must satisfy the Normal equation \eqref{bcs10} for some (unique) $\mu \geq 0$. Using these arguments, \eqref{cstls11b} simplifies to
\begin{align} 
 g(W^{*}, B^{*} + & \Delta B,  x^{*} + \tilde{\Delta } x ) \geq g(W^{*}, B^{*},  x^{*}) 
- 2 \mu \left \langle x^{*}, \tilde{\Delta } x \right \rangle  \label{cstls11c}
\end{align}
Now, if (the optimal) $\mu = 0$ above, then $g(W^{*}, B^{*} + \Delta B,  x^{*} + \tilde{\Delta } x ) \geq g(W^{*}, B^{*},  x^{*}) $ for arbitrary $\tilde{\Delta } x \in \mathbb{C}^{p}$, and $\Delta B \in \mathrm{R1 \cup R2}$.
The alternative scenario is the case when (the optimal) $\mu>0$, which can occur only if $\begin{Vmatrix}x^{*}\end{Vmatrix}_{2} = C$ holds. 
Since we are considering energy preserving perturbations $\tilde{\Delta} x$,
we have $\begin{Vmatrix} x^{*} + \tilde{\Delta } x 
\end{Vmatrix}_{2}^{2} \leq C^{2} = \begin{Vmatrix} x^{*}
\end{Vmatrix}_{2}^{2}$, which implies $- 2 \left \langle x^{*}, \tilde{\Delta } x   \right \rangle \geq \begin{Vmatrix}
\tilde{\Delta } x 
\end{Vmatrix}_{2}^{2} \geq 0$. Combining this result with \eqref{cstls11c}, we again have (now for the $\mu>0$ case) that $g(W^{*}, B^{*} + \Delta B,  x^{*} + \tilde{\Delta } x ) \geq g(W^{*}, B^{*},  x^{*}) $ for arbitrary $\tilde{\Delta } x \in \mathbb{C}^{p}$, and $\Delta B \in \mathrm{R1 \cup R2}$. 
\end{proof}



\subsection{Proofs of Theorems \ref{theorem2bc} and \ref{theorem3bc}}  \label{ap2}

The proofs of Theorems \ref{theorem2bc} and \ref{theorem3bc} are very similar to that for Theorem \ref{theorem1bc}. We only discuss some of the minor differences, as follows. 

First, in the case of Theorem \ref{theorem2bc}, the main difference in the proof is that the non-negative barrier function $\psi (B)$ and the operator $H_{s}(\cdot)$  (in the proof of Theorem \ref{theorem1bc}) are replaced by the non-negative penalty $\eta^{2}  \sum _{j=1}^{N} \left \| b_{j} \right \|_{0}$ and the operator $\hat{H}_{\eta}^{1} ( \cdot )$, respectively. Moreover, the mapping $\tilde{H_{s}}(\cdot)$ defined in \eqref{dfdc22} for the proof of Theorem \ref{theorem1bc}, is replaced by the matrix-to-set mapping $\hat{H}_{\eta}(\cdot) $ (for the proof of Theorem \ref{theorem2bc}) defined as 
\begin{equation} \label{equ889cc}
\left ( \hat{H}_{\eta}(Z) \right )_{ij}=\left\{\begin{matrix}
 0&, \;\;\left | Z_{ij} \right | < \eta \\
\left \{ Z_{ij}, \, 0 \right \}  & ,\;\;\left | Z_{ij} \right | = \eta \\ 
Z_{ij}  & ,\;\;\left | Z_{ij} \right | > \eta
\end{matrix}\right.
\end{equation}
By thus replacing the relevant functions and operators, the various steps in the proof of Theorem \ref{theorem1bc} can be easily extended to the case of Theorem \ref{theorem2bc}. As such, Theorem \ref{theorem2bc} mainly differs from Theorem \ref{theorem1bc} in terms of the definition of the set of allowed (local) perturbations $\Delta B$.
In particular, the proofs of partial local optimality of accumulation points in Theorem \ref{theorem2bc} are easily extended from the aforementioned proofs for Lemmas \ref{lemma10bc} and \ref{lemma11bc}, by using the techniques and inequalities presented in Appendix F of \cite{sbclsTS2}. 

Finally, the main difference between the proofs of Theorems \ref{theorem1bc} and \ref{theorem3bc} is that the non-negative $\lambda \, Q(W)$ penalty in the former is replaced by the barrier function $\varphi(W)$ (that enforces the unitary property, and keeps $W^{t}$ always bounded) in the latter. Otherwise, the proof techniques are very similar for the two cases.

\bibliographystyle{siam}
\bibliography{BCS_v10}

\begin{thebibliography}{10}

\bibitem{elad5}
{\sc M.~Aharon and M.~Elad}, {\em Sparse and redundant modeling of image
  content using an image-signature-dictionary}, SIAM Journal on Imaging
  Sciences, 1 (2008), pp.~228--247.

\bibitem{elad}
{\sc M.~Aharon, M.~Elad, and A.~Bruckstein}, {\em {K-SVD}: An algorithm for
  designing overcomplete dictionaries for sparse representation}, IEEE
  Transactions on signal processing, 54 (2006), pp.~4311--4322.

\bibitem{Attouchaa}
{\sc H.~Attouch, J.~Bolte, P.~Redont, and A.~Soubeyran}, {\em Proximal
  alternating minimization and projection methods for nonconvex problems: An
  approach based on the kurdyka-{\l}ojasiewicz inequality}, Math. Oper. Res.,
  35 (2010), pp.~438--457.

\bibitem{Bre-C2008a}
{\sc Y.~Bresler}, {\em Spectrum-blind sampling and compressive sensing for
  continuous-index signals}, in 2008 Information Theory and Applications
  Workshop Conference, 2008, pp.~547--554.

\bibitem{BreFen-C96c}
{\sc Y.~Bresler and P.~Feng}, {\em Spectrum-blind minimum-rate sampling and
  reconstruction of {2-D} multiband signals}, in Proc. 3rd IEEE Int. Conf. on
  Image Processing, ICIP'96, sep 1996, pp.~701--704.

\bibitem{BreGasVen-C99}
{\sc Y.~Bresler, M.~Gastpar, and R.~Venkataramani}, {\em Image compression
  on-the-fly by universal sampling in {F}ourier imaging systems}, in Proc. 1999
  IEEE Information Theory Workshop on Detection, Estimation, Classification,
  and Imaging, feb 1999, p.~48.

\bibitem{ambruck}
{\sc A.~M. Bruckstein, D.~L. Donoho, and M.~Elad}, {\em From sparse solutions
  of systems of equations to sparse modeling of signals and images}, SIAM
  Review, 51 (2009), pp.~34--81.

\bibitem{tao1}
{\sc E.~Cand\`{e}s, J.~Romberg, and T.~Tao}, {\em Robust uncertainty
  principles: exact signal reconstruction from highly incomplete frequency
  information}, IEEE Trans. Information Theory, 52 (2006), pp.~489--509.

\bibitem{cand}
{\sc E.~Cand\`{e}s and T.~Tao}, {\em Decoding by linear programming}, IEEE
  Trans. on Information Theory, 51 (2005), pp.~4203--4215.

\bibitem{Char}
{\sc R.~Chartrand}, {\em Fast algorithms for nonconvex compressive sensing:
  {MRI} reconstruction from very few data}, in Proc. {IEEE} International
  Symposium on Biomedical Imaging (ISBI), 2009, pp.~262--265.

\bibitem{chen}
{\sc G.~H. Chen, J.~Tang, and S.~Leng}, {\em Prior image constrained compressed
  sensing (piccs): A method to accurately reconstruct dynamic {CT} images from
  highly undersampled projection data sets}, Med. Phys., 35 (2008),
  pp.~660--663.

\bibitem{chen2}
{\sc S.~S. Chen, D.~L. Donoho, and M.~A. Saunders}, {\em Atomic decomposition
  by basis pursuit}, SIAM J. Sci. Comput., 20 (1998), pp.~33--61.

\bibitem{choi11}
{\sc K.~Choi, J.~Wang, L.~Zhu, T.-S. Suh, S.~Boyd, and L.~Xing}, {\em
  Compressed sensing based cone-beam computed tomography reconstruction with a
  first-order method}, Med. Phys., 37 (2010), pp.~5113--5125.

\bibitem{emlie12}
{\sc E.~Chouzenoux, J.-C. Pesquet, and A.~Repetti}, {\em A block coordinate
  variable metric forward-backward algorithm},  (2014).
\newblock Preprint:
  \url{http://www.optimization-online.org/DB_HTML/2013/12/4178.html}.

\bibitem{dbov}
{\sc K.~Dabov, A.~Foi, V.~Katkovnik, and K.~Egiazarian}, {\em Image denoising
  by sparse {3D} transform-domain collaborative filtering}, IEEE Trans. on
  Image Processing, 16 (2007), pp.~2080--2095.

\bibitem{wei}
{\sc W.~Dai and O.~Milenkovic}, {\em Subspace pursuit for compressive sensing
  signal reconstruction}, IEEE Trans. Information Theory, 55 (2009),
  pp.~2230--2249.

\bibitem{npa}
{\sc G.~Davis, S.~Mallat, and M.~Avellaneda}, {\em Adaptive greedy
  approximations}, Journal of Constructive Approximation, 13 (1997),
  pp.~57--98.

\bibitem{do}
{\sc M.~N. Do and M.~Vetterli}, {\em The contourlet transform: an efficient
  directional multiresolution image representation}, IEEE Trans. Image
  Process., 14 (2005), pp.~2091--2106.

\bibitem{don}
{\sc D.~Donoho}, {\em Compressed sensing}, IEEE Trans. Information Theory, 52
  (2006), pp.~1289--1306.

\bibitem{Dono}
{\sc D.~L. Donoho}, {\em For most large underdetermined systems of linear
  equations the minimal l1-norm solution is also the sparsest solution}, Comm.
  Pure Appl. Math, 59 (2004), pp.~797--829.

\bibitem{Dono2}
{\sc D.~L. Donoho, M.~Elad, and V.~N. Temlyakov}, {\em Stable recovery of
  sparse overcomplete representations in the presence of noise}, IEEE Trans.
  Inform. Theory, 52 (2006), pp.~6--18.

\bibitem{befro}
{\sc B.~Efron, T.~Hastie, I.~Johnstone, and R.~Tibshirani}, {\em Least angle
  regression}, Annals of Statistics, 32 (2004), pp.~407--499.

\bibitem{elad2}
{\sc M.~Elad and M.~Aharon}, {\em Image denoising via sparse and redundant
  representations over learned dictionaries}, IEEE Trans. Image Process., 15
  (2006), pp.~3736--3745.

\bibitem{elmiru}
{\sc M.~Elad, P.~Milanfar, and R.~Rubinstein}, {\em Analysis versus synthesis
  in signal priors}, Inverse Problems, 23 (2007), pp.~947--968.

\bibitem{eldena}
{\sc L.~Eld\`{e}n}, {\em Solving quadratically constrained least squares
  problems using a differential-geometric approach}, BIT Numerical Mathematics,
  42 (2002), pp.~323--335.

\bibitem{eng}
{\sc K.~Engan, S.O. Aase, and J.H. Hakon-Husoy}, {\em Method of optimal
  directions for frame design}, in Proc. {IEEE} International Conference on
  Acoustics, Speech, and Signal Processing, 1999, pp.~2443--2446.

\bibitem{Fen-PT97}
{\sc P.~Feng}, {\em Universal Spectrum Blind Minimum Rate Sampling and
  Reconstruction of Multiband Signals}, PhD thesis, University of Illinois at
  Urbana-Champaign, mar 1997.
\newblock {Y}oram Bresler, adviser.

\bibitem{feng96a}
{\sc P.~Feng and Y.~Bresler}, {\em Spectrum-blind minimum-rate sampling and
  reconstruction of multiband signals}, in ICASSP, vol.~3, may 1996,
  pp.~1689--1692.

\bibitem{GasBre-C00a}
{\sc M.~Gastpar and Y.~Bresler}, {\em On the necessary density for
  spectrum-blind nonuniform sampling subject to quantization}, in ICASSP,
  vol.~1, jun 2000, pp.~348--351.

\bibitem{Glei12}
{\sc S.~Gleichman and Yonina~C. Eldar}, {\em Blind compressed sensing}, IEEE
  Transactions on Information Theory, 57 (2011), pp.~6958--6975.

\bibitem{golkah1}
{\sc G.~H. Golub and C.~F.~Van Loan}, {\em Matrix Computations}, Johns Hopkins
  University Press, Baltimore, Maryland, 1996.

\bibitem{RaoFocus}
{\sc I.~F. Gorodnitsky, J.~George, and B.~D. Rao}, {\em Neuromagnetic source
  imaging with {FOCUSS}: A recursive weighted minimum norm algorithm},
  Electrocephalography and Clinical Neurophysiology, 95 (1995), pp.~231--251.

\bibitem{kar}
{\sc R.~Gribonval and K.~Schnass}, {\em Dictionary identification--sparse
  matrix-factorization via $\textit{l}_{1}$ -minimization}, IEEE Trans. Inform.
  Theory, 56 (2010), pp.~3523--3539.

\bibitem{Yoo}
{\sc Y.~Kim, M.~S. Nadar, and A.~Bilgin}, {\em Wavelet-based compressed sensing
  using gaussian scale mixtures}, in Proc. ISMRM, 2010, p.~4856.

\bibitem{CT11}
{\sc X.~Li and S.~Luo}, {\em A compressed sensing-based iterative algorithm for
  ct reconstruction and its possible application to phase contrast imaging},
  BioMedical Engineering OnLine, 10 (2011), p.~73.

\bibitem{lingal1}
{\sc S.~G. Lingala and M.~Jacob}, {\em Blind compressive sensing dynamic mri},
  IEEE Transactions on Medical Imaging, 32 (2013), pp.~1132--1145.

\bibitem{lus33}
{\sc M.~{L}ustig}, {\em Michael {L}ustig home page}.
\newblock \url{http://www.eecs.berkeley.edu/~mlustig/Software.html}, 2014.
\newblock [Online; accessed October, 2014].

\bibitem{lustig}
{\sc M.~Lustig, D.L. Donoho, and J.M. Pauly}, {\em Sparse {MRI}: The
  application of compressed sensing for rapid {MR} imaging}, Magnetic Resonance
  in Medicine, 58 (2007), pp.~1182--1195.

\bibitem{lustig2}
{\sc M.~Lustig, J.~M. Santos, D.~L. Donoho, and J.~M. Pauly}, {\em k-t
  {SPARSE}: High frame rate dynamic {MRI} exploiting spatio-temporal sparsity},
  in Proc. ISMRM, 2006, p.~2420.

\bibitem{Mai}
{\sc J.~Mairal, F.~Bach, J.~Ponce, and G.~Sapiro}, {\em Online learning for
  matrix factorization and sparse coding}, J. Mach. Learn. Res., 11 (2010),
  pp.~19--60.

\bibitem{elad3}
{\sc J.~Mairal, M.~Elad, and G.~Sapiro}, {\em Sparse representation for color
  image restoration}, IEEE Trans. on Image Processing, 17 (2008), pp.~53--69.

\bibitem{elad6}
{\sc J.~Mairal, G.~Sapiro, and M.~Elad}, {\em Learning multiscale sparse
  representations for image and video restoration}, SIAM Multiscale Modeling
  and Simulation, 7 (2008), pp.~214--241.

\bibitem{malc2}
{\sc K.~Malczewski}, {\em Pet image reconstruction using compressed sensing},
  in Signal Processing: Algorithms, Architectures, Arrangements, and
  Applications (SPA), 2013, Sept 2013, pp.~176--181.

\bibitem{malk122}
{\sc S.~C. Malik}, {\em Principles of Real Analysis}, New Age International,
  New Delhi, India, 1982.

\bibitem{wav}
{\sc S.~Mallat}, {\em A Wavelet Tour of Signal Processing}, Academic Press, San
  Diego, CA, 1999.

\bibitem{mp2}
{\sc S.~G. Mallat and Z.~Zhang}, {\em Matching pursuits with time-frequency
  dictionaries}, IEEE Transactions on Signal Processing, 41 (1993),
  pp.~3397--3415.

\bibitem{jpg2}
{\sc M.~W. Marcellin, M.~J. Gormish, A.~Bilgin, and M.~P. Boliek}, {\em An
  overview of {JPEG-2000}}, in Proc. Data Compression Conf., 2000,
  pp.~523--541.

\bibitem{Y7027820}
{\sc Y.~Mohsin, G.~Ongie, and M.~Jacob}, {\em Iterative shrinkage algorithm for
  patch-smoothness regularized medical image recovery}, IEEE Transactions on
  Medical Imaging,  (2015).

\bibitem{vari2}
{\sc B.~S. Mordukhovich}, {\em Variational Analysis and Generalized
  Differentiation. Vol. {I}: Basic theory}, Springer-Verlag, Heidelberg,
  Germany, 2006.

\bibitem{npb}
{\sc B.~K. Natarajan}, {\em Sparse approximate solutions to linear systems},
  SIAM J. Comput., 24 (1995), pp.~227--234.

\bibitem{Qu12}
{\sc B.~Ning, X.~Qu, D.~Guo, C.~Hu, and Z.~Chen}, {\em Magnetic resonance image
  reconstruction using trained geometric directions in 2d redundant wavelets
  domain and non-convex optimization}, Magnetic Resonance Imaging, 31 (2013),
  pp.~1611--1622.

\bibitem{ols}
{\sc B.~A. Olshausen and D.~J. Field}, {\em Emergence of simple-cell receptive
  field properties by learning a sparse code for natural images}, Nature, 381
  (1996), pp.~607--609.

\bibitem{pati}
{\sc Y.~Pati, R.~Rezaiifar, and P.~Krishnaprasad}, {\em Orthogonal matching
  pursuit : recursive function approximation with applications to wavelet
  decomposition}, in Asilomar Conf. on Signals, Systems and Comput., 1993,
  pp.~40--44 vol.1.

\bibitem{luke1}
{\sc L.~Pfister}, {\em Tomographic reconstruction with adaptive sparsifying
  transforms}, master's thesis, University of Illinois at Urbana-Champaign,
  Aug. 2013.

\bibitem{luke2}
{\sc L.~Pfister and Y.~Bresler}, {\em Model-based iterative tomographic
  reconstruction with adaptive sparsifying transforms}, in SPIE International
  Symposium on Electronic Imaging: Computational Imaging XII, vol.~9020, 2014,
  pp.~90200H--1--90200H--11.

\bibitem{tfcode}
{\sc W.~K. Pratt, J.~Kane, and H.~C. Andrews}, {\em Hadamard transform image
  coding}, Proc. {IEEE}, 57 (1969), pp.~58--68.

\bibitem{pMRI-Survey}
{\sc K.~P. Pruessmann}, {\em Encoding and reconstruction in parallel {MRI}},
  NMR in Biomedicine, 19 (2006), pp.~288--299.

\bibitem{kal}
{\sc C.~Qiu, W.~Lu, and N.~Vaswani}, {\em Real-time dynamic {MR} image
  reconstruction using kalman filtered compressed sensing}, in Proc. {IEEE}
  International Conference on Acoustics, Speech and Signal Processing, 2009,
  pp.~393--396.

\bibitem{PANOweb}
{\sc X.~{Q}u}, {\em {PANO} {C}ode}.
\newblock
  \url{http://www.quxiaobo.org/project/CS_MRI_PANO/Demo_PANO_SparseMRI.zip},
  2014.
\newblock [Online; accessed May, 2015].

\bibitem{PANOwebFast}
\leavevmode\vrule height 2pt depth -1.6pt width 23pt, {\em {PANO} {C}ode with
  multi-core cpu parallel computing}.
\newblock
  \url{http://www.quxiaobo.org/project/CS_MRI_PANO/Demo_Parallel_PANO_SparseMRI.zip},
  2014.
\newblock [Online; accessed April, 2015].

\bibitem{Quweb}
\leavevmode\vrule height 2pt depth -1.6pt width 23pt, {\em {PBDWS} {C}ode}.
\newblock
  \url{http://www.quxiaobo.org/project/CS_MRI_PBDWS/Demo_PBDWS_SparseMRI.zip},
  2014.
\newblock [Online; accessed September, 2014].

\bibitem{Qu11}
{\sc X.~Qu, D.~Guo, B.~Ning, Y.~Hou, Y.~Lin, S.~Cai, and Z.~Chen}, {\em
  Undersampled {MRI} reconstruction with patch-based directional wavelets},
  Magnetic Resonance Imaging, 30 (2012), pp.~964--977.

\bibitem{Qu2014843}
{\sc X.~Qu, Y.~Hou, F.~Lam, D.~Guo, J.~Zhong, and Z.~Chen}, {\em Magnetic
  resonance image reconstruction from undersampled measurements using a
  patch-based nonlocal operator}, Medical Image Analysis, 18 (2014),
  pp.~843--856.

\bibitem{bresai}
{\sc S.~Ravishankar and Y.~Bresler}, {\em {MR} image reconstruction from highly
  undersampled k-space data by dictionary learning}, {IEEE} {T}rans. {M}ed.
  {I}mag., 30 (2011), pp.~1028--1041.

\bibitem{symul}
\leavevmode\vrule height 2pt depth -1.6pt width 23pt, {\em {Multiscale}
  dictionary learning for {MRI}}, in Proc. ISMRM, 2011, p.~2830.

\bibitem{yblsgg}
\leavevmode\vrule height 2pt depth -1.6pt width 23pt, {\em Learning doubly
  sparse transforms for image representation}, in IEEE Int. Conf. Image
  Process., 2012, pp.~685--688.

\bibitem{sabres3}
\leavevmode\vrule height 2pt depth -1.6pt width 23pt, {\em Closed-form
  solutions within sparsifying transform learning}, in IEEE International
  Conference on Acoustics, Speech and Signal Processing (ICASSP), 2013,
  pp.~5378--5382.

\bibitem{dlmri1}
\leavevmode\vrule height 2pt depth -1.6pt width 23pt, {\em {DLMRI} - {L}ab:
  Dictionary learning {MRI} software}.
\newblock \url{http://www.ifp.illinois.edu/~yoram/DLMRI-Lab/DLMRI.html}, 2013.
\newblock [Online; accessed October, 2014].

\bibitem{doubsp2l}
\leavevmode\vrule height 2pt depth -1.6pt width 23pt, {\em Learning doubly
  sparse transforms for images}, IEEE Trans. Image Process., 22 (2013),
  pp.~4598--4612.

\bibitem{sabres}
\leavevmode\vrule height 2pt depth -1.6pt width 23pt, {\em Learning sparsifying
  transforms}, IEEE Trans. Signal Process., 61 (2013), pp.~1072--1086.

\bibitem{syber}
\leavevmode\vrule height 2pt depth -1.6pt width 23pt, {\em Sparsifying
  transform learning for compressed sensing {MRI}}, in {P}roc. {IEEE} {I}nt.
  {S}ymp. {B}iomed. {I}mag., 2013, pp.~17--20.

\bibitem{samptabcs}
\leavevmode\vrule height 2pt depth -1.6pt width 23pt, {\em Blind compressed
  sensing using sparsifying transforms}, in International Conference on
  Sampling Theory and Applications (SampTA), May 2015, pp.~513--517.

\bibitem{sbclsTS2}
\leavevmode\vrule height 2pt depth -1.6pt width 23pt, {\em $\ell_0$ sparsifying
  transform learning with efficient optimal updates and convergence
  guarantees}, IEEE Trans. Signal Process., 63 (2015), pp.~2389--2404.

\bibitem{vari1}
{\sc R.~T. Rockafellar and Roger J.-B. Wets}, {\em Variational Analysis},
  Springer-Verlag, Heidelberg, Germany, 1998.

\bibitem{skret}
{\sc K.~Skretting and K.~Engan}, {\em Recursive least squares dictionary
  learning algorithm}, IEEE Transactions on Signal Processing, 58 (2010),
  pp.~2121--2130.

\bibitem{josh}
{\sc J.~Trzasko and A.~Manduca}, {\em Highly undersampled magnetic resonance
  image reconstruction via homotopic $l_{0}$-minimization}, IEEE Trans. Med.
  Imaging, 28 (2009), pp.~106--121.

\bibitem{tseng6}
{\sc P.~Tseng}, {\em Convergence of a block coordinate descent method for
  nondifferentiable minimization}, J. Optim. Theory Appl., 109 (2001),
  pp.~475--494.

\bibitem{vali1}
{\sc S.~Valiollahzadeh, T.~Chang, J.~W. Clark, and O.~R. Mawlawi}, {\em Image
  recovery in pet scanners with partial detector rings using compressive
  sensing}, in IEEE Nuclear Science Symposium and Medical Imaging Conference
  (NSS/MIC), Oct 2012, pp.~3036--3039.

\bibitem{VenBre-C98b}
{\sc R.~Venkataramani and Y.~Bresler}, {\em Further results on spectrum blind
  sampling of {2D} signals}, in Proc. IEEE Int. Conf. Image Proc., ICIP,
  vol.~2, Oct. 1998, pp.~752--756.

\bibitem{wangying}
{\sc Y.~Wang, Y.~Zhou, and L.~Ying}, {\em Undersampled dynamic magnetic
  resonance imaging using patch-based spatiotemporal dictionaries}, in 2013
  IEEE 10th International Symposium on Biomedical Imaging (ISBI), April 2013,
  pp.~294--297.

\bibitem{saiwen}
{\sc B.~Wen, S.~Ravishankar, and Y.~Bresler}, {\em Structured overcomplete
  sparsifying transform learning with convergence guarantees and applications},
  International Journal of Computer Vision,  (2014), pp.~1--31.

\bibitem{xu222}
{\sc Y.~Xu and W.~Yin}, {\em A block coordinate descent method for regularized
  multiconvex optimization with applications to nonnegative tensor
  factorization and completion}, SIAM Journal on Imaging Sciences, 6 (2013),
  pp.~1758--1789.

\bibitem{Yagh}
{\sc M.~Yaghoobi, T.~Blumensath, and M.~Davies}, {\em Dictionary learning for
  sparse approximations with the majorization method}, IEEE Transaction on
  Signal Processing, 57 (2009), pp.~2178--2191.

\bibitem{YeBreMou-J02}
{\sc J.~C. Ye, Y.~Bresler, and P.~Moulin}, {\em A self-referencing level-set
  method for image reconstruction from sparse {F}ourier samples}, Int. J.
  Computer Vision, 50 (2002), pp.~253--270.

\bibitem{gyu}
{\sc G.~Yu, G.~Sapiro, and S.~Mallat}, {\em Image modeling and enhancement via
  structured sparse model selection}, in Proc. {IEEE} International Conference
  on Image Processing (ICIP), 2010, pp.~1641--1644.

\end{thebibliography}

\end{document}